\documentclass[sigconf,nonacm]{acmart}

\setcopyright{none}
\renewcommand\footnotetextcopyrightpermission[1]{}

\usepackage{booktabs}
\usepackage{float}
\usepackage{placeins}
\usepackage{capt-of}
\usepackage{algorithm}
\usepackage{algpseudocode}
\usepackage{tikz}
\usetikzlibrary{arrows.meta,positioning,fit,calc}
\usepackage{enumitem}
\setlist[itemize]{leftmargin=1.2em, topsep=2pt, itemsep=1.5pt,
                  parsep=0pt}

\AtBeginDocument{%
  \setlength{\abovedisplayskip}{4pt plus 1pt minus 2pt}%
  \setlength{\belowdisplayskip}{4pt plus 1pt minus 2pt}%
  \setlength{\abovedisplayshortskip}{2pt plus 1pt minus 1pt}%
  \setlength{\belowdisplayshortskip}{2pt plus 1pt minus 1pt}}

\newtheorem{lemma}{Lemma}
\newtheorem{proposition}{Proposition}
\newtheorem{fact}{Fact}

\newcommand{\panache}{\textsc{Panache}}
\newcommand{\scamp}{\textsc{Scamp}}
\newcommand{\stumpy}{\textsc{Stumpy}}
\newcommand{\attimo}{\textsc{Attimo}}
\newcommand{\dist}{\mathrm{d}}
\newcommand{\pbound}{\underline{\dist}}

\begin{document}

\title{\panache: One-Pass Motif Discovery at Every Window Length}

\author{Tej Sanibh Ranade}
\email{tsranade1@gmail.com}
\affiliation{%
  \institution{Independent Researcher}
  \city{}
  \country{}
}

\begin{abstract}
Motif discovery, the search for recurring patterns within a time
series, is a core primitive of exploratory data analysis. A pattern,
however, is defined by its duration, which analysts rarely know in
advance. To resolve this unknown duration, an interval of window
lengths is defined, and the accepted method is to try every length
in that interval. Existing pan matrix profile (PMP) methods compute
one z-normalized matrix profile per length, so $L$ lengths cost $L$
quadratic self-joins over the same series. We introduce \panache{},
to our knowledge the first one-pass streaming algorithm for
z-normalized PMP motif discovery. It replaces the repeated
self-joins with a single scan whose runtime is near-linear in the
series length. The key observation is that mean-centering a
subsequence changes only its DC Fourier coefficient, so the non-DC
spectrum of every z-normalized subsequence can be maintained online
by sliding-DFT recurrences and running statistics. This spectral
state is the key under which similar subsequences collide in an
occupancy-controlled hash directory and, through Parseval's
theorem, yields a lower bound that rejects most colliding pairs
before any exact computation. \panache{} computes every
data-dependent parameter itself, leaving only a resource budget to
tune. At the default budget, it recovers all top-20 pan-motifs
against exact fixed-exclusion ground truth on 17 UCR
configurations, and is faster than every CPU and GPU baseline
benchmarked in this paper. On Wafer at five million samples over
51 lengths, \panache{} completes one pass in 2.9 minutes and emits
the exact motifs in 6.0 minutes, against 7.95 hours for the
fastest exact CPU baseline and 38.3 minutes for \scamp{} on an
H100 GPU.
\end{abstract}



\maketitle

\section{Introduction}
\label{sec:intro}

Time-series motif discovery looks for subsequences that recur in the
same series~\cite{lin2002finding,patel2002mining}. The difficulty is
that a motif is not defined until a
duration is chosen. A window length $m$ represents a duration. It
determines which subsequences are compared and what it means for
two occurrences to have the same shape. Too short a window fragments the event; too long a window
mixes the event with its context. In exploratory data analysis, this
duration is often the unknown rather than an input the analyst can
set reliably.

The matrix profile~\cite{yeh2016matrix} replaced the varying
definitions of a motif with a single formulation. For a chosen
window length $m$, a fixed-length algorithm starts by initializing
a window and computing its z-normalized distance to all valid
non-overlapping windows of the same length in the entire series
(Figure~\ref{fig:pmpcon}). The window that achieved the smallest of
these distances is the nearest neighbor. When the same window
advances by one position, distances are recomputed across the
entire series again. This process is repeated from the start of the
series to the end, producing a quadratic self-join, and one
self-join yields one row, the matrix profile $P_m$, which stores
each window's nearest-neighbor distance. Each stored entry is a
cell, and motifs are the cells of $P_m$ with the smallest distances
rather than the outcome of a separate combinatorial search. The
speed of these algorithms comes from recurrences that reuse dot
products while two windows of the same length slide together. The
reuse makes each evaluation cheap, but it never makes evaluations
fewer. This formulation is why fixed-length motif discovery has
mature exact
algorithms~\cite{zhu2016matrix2,zhu2018scrimp,zimmerman2019scamp,law2019stumpy}.

\begin{figure}[t]
  \centering
  \includegraphics[width=\columnwidth]{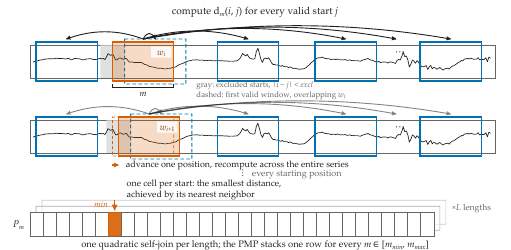}
  \caption{Existing PMP representation on a sample of ECG200.
  Window $w_i$ is evaluated against all windows in the series.
  Every slide causes the repetition. Every different window length
  requires the same process in the PMP.}
  \label{fig:pmpcon}
\end{figure}

The pan matrix profile (PMP)~\cite{madrid2019skimp} replaces the
duration choice with an interval of lengths. It computes a matrix
profile for every length and stacks each profile for every
$m\in[m_{\min},m_{\max}]$. Existing PMP algorithms achieve this
generality through repetition. The computations performed at one
length are never reused at another, because every length changes
the set of windows, the overlap rule that rejects trivial matches,
the mean and standard deviation that z-normalize every window, and
the length-$m$ Fourier basis used by the FFT steps. The recurrences
never cross lengths either. Every single window length forces its
own rescan of the entire series, and that just by taking one slide.
The cost increases with every length in the interval.

We present \panache{}, to our knowledge the first one-pass streaming
algorithm for z-normalized PMP motif discovery. \panache{} scans the
series once while maintaining a single multi-length search state.
Each sample is consumed on arrival and advances the whole
multi-length state at once. No length and no slide ever forces a
rescan of the series. From this state it selects which subsequence
pairs to evaluate. A profile cell is updated only with the exact
z-normalized distance of a valid pair (\S\ref{sec:background}).
Every written cell therefore holds an exact distance to a valid
neighbor.

The central representation is the normalized AC spectrum of a window.
If $X_k$ is a Fourier coefficient of a window, subtracting the window
mean changes only $X_0$. Every coefficient with $k>0$ is unchanged by
centering and is divided only by the window standard deviation. The
non-DC coefficients can therefore be maintained online by sliding-DFT
recurrences. Beyond serving as the streaming state, this vector is
the SimHash key~\cite{charikar2002estimating}, a locality-sensitive
hash (LSH) under which windows of similar normalized shape
collide. The same vector gives the Parseval lower bound that rejects
most candidate pairs before exact distance computation.

We introduce four mechanisms that make this representation scale
with a growing stream. \emph{Anchor lengths} restrict the online
state to a subset of the interval's lengths, relying on the
persistence of motifs across neighboring lengths. Similar windows
collide in the same hash bucket, and \emph{occupancy-controlled
prefix hashing} keeps mean bucket occupancy near a fixed target as
the series grows, so a constant per-query budget stays effective for
all but the densest classes.
\emph{Exclusion-zone batching} parallelizes the stream while leaving
the output bit-identical to serial processing at any thread count. A
final \emph{emit} stage materializes the profile with exact
comparisons inside oversized buckets and around the discovered
motifs.

Our contributions are as follows.
\begin{itemize}
\item A formulation of one-pass PMP motif discovery around a sound
  upper pan-profile, which guarantees exact distances and valid
  neighbors for every written cell (\S\ref{sec:background}).
\item A normalized AC-spectrum representation whose non-DC
  coefficients are maintained online and serve simultaneously as the
  streaming state, the SimHash key, and a Parseval rejection bound
  (\S\ref{sec:spectra}).
\item The mechanisms required at scale: occupancy-controlled
  global-depth hashing, anchor-length search, deterministic
  exclusion-zone batching, and emit-time completion
  (\S\ref{sec:anchors}--\ref{sec:emit}).
\item An evaluation on ten UCR datasets across 17
  exact-ground-truth configurations, with full top-20 recovery,
  faster runtime than every comparable public CPU and GPU baseline,
  and per-configuration statistics of the recovered motif values
  (\S\ref{sec:eval}). Every data-dependent parameter is computed by
  the algorithm itself, and every design constant stayed at its
  default in every experiment.
\end{itemize}

\section{Background and Problem}
\label{sec:background}

A time series is a vector $T=(t_0,\ldots,t_{n-1})$. For a window
length $m$, write $T_{i,m}=(t_i,\ldots,t_{i+m-1})$, with
$0\le i<n-m+1$, for the subsequence starting at position $i$. Its
z-normalized form is
$Z_{i,m}=(T_{i,m}-\mu_{i,m}\mathbf{1})/\sigma_{i,m}$,
where constant windows, for which $\sigma_{i,m}=0$, are ignored. The
distance between two subsequences of length $m$ is the Euclidean
distance after z-normalization,
$\dist_m(i,j)=\lVert Z_{i,m}-Z_{j,m}\rVert_2$,
equivalently $\dist_m(i,j)^2 = 2m(1-\rho_m(i,j))$ for their Pearson
correlation $\rho_m(i,j)$. Exact matrix-profile algorithms exploit this
identity. The correlation is the raw dot product corrected by the
window moments,
$m\,\rho_m(i,j)=(T_{i,m}\cdot T_{j,m}-m\,\mu_{i,m}\mu_{j,m})
/(\sigma_{i,m}\sigma_{j,m})$. One dot product and the moments
$\mu,\sigma$ of both windows therefore determine the distance in
closed form, without re-reading
the samples~\cite{yeh2016matrix,zimmerman2019scamp}. In this paper, an
exact distance is the value $\dist_m(i,j)$ itself, never an estimate
or a bound.

In a self-join, windows whose starts are very close share most of
their samples and are not considered motif pairs. A pair is
\emph{valid} when its start positions satisfy
$|i-j|\ge \mathit{excl}$, for a fixed exclusion parameter
$\mathit{excl}$~\cite{lin2002finding,chiu2003probabilistic}. For one
length $m$, the
matrix profile is the row pair $(P_m,I_m)$~\cite{yeh2016matrix},
\[
  P_m[i] = \min_{\substack{j\\ |i-j|\ge \mathit{excl}}}
           \dist_m(i,j), \qquad
  I_m[i] \in \arg\min_{\substack{j\\ |i-j|\ge \mathit{excl}}}
           \dist_m(i,j).
\]

The pan matrix profile (PMP)~\cite{madrid2019skimp} is the stack of
these rows over an interval $[m_{\min},m_{\max}]$. Let
$L=m_{\max}-m_{\min}+1$. A pan-motif is a low cell of this stack
together with its nearest-neighbor index. To avoid reporting the same
region repeatedly, the top-$k$ list is formed by sorting cells by
distance and skipping any cell that lies within $m/2$ of an already
accepted cell at the same length.

\paragraph{The streaming target.}
The exact PMP contains $\sum_{m}(n-m+1)$ cells, each requiring a
nearest-neighbor search among the windows of its own length.
\panache{} maintains a relaxation of this object. A \emph{sound upper
pan-profile} is a stack of rows $(\hat{P}_m,\hat{I}_m)$ satisfying,
for every cell, either $\hat{P}_m[i]=\infty$ or
\[
  \hat{P}_m[i]=\dist_m(i,\hat{I}_m[i])
  \quad\text{with}\quad
  |i-\hat{I}_m[i]|\ge \mathit{excl},
\]
and in all cases $\hat{P}_m[i]\ge P_m[i]$. If
$\hat{P}_m[i]<\infty$, then $\hat{I}_m[i]$ is a valid neighbor and
$\hat{P}_m[i]$ is its exact distance. Only the nearest-neighbor
quantifier is relaxed: if the algorithm evaluates the true
nearest-neighbor pair, $\hat{P}_m[i]$ attains the exact profile value
$P_m[i]$; otherwise it remains a valid upper bound, witnessed by an
exactly evaluated neighbor. A finite value below the exact profile
value cannot occur, because every written value is the exact distance
of some valid pair, and the exact value is the minimum over all of
them.

\paragraph{Problem statement.}
\looseness=-1
The problem tackled in this paper is to process the samples once, in
stream order, over the full length interval. The method must
maintain a sound upper pan-profile and emit top-$k$ pan-motifs whose
reported distances are exact distances to valid neighbors. It must
find, within its single scan, the valid pairs whose exact distances
lower the motif cells, without executing the $L$ self-joins that
define the exact PMP. Updates at stream time $\tau$ may use only
$t_0,\ldots,t_\tau$. Space linear in $n$ is inherent to the task,
since a materialized pan-profile alone has $\Theta(nL)$ cells. The
memory accounting is given in \S\ref{sec:cost}.

\section{The \panache{} Algorithm}
\label{sec:algorithm}

Every PMP construction, exact or approximate, is characterized by the
set of subsequence pairs whose distances it evaluates. Let
$\mathcal U=\{(m,i,j):\ m_{\min}\le m\le m_{\max},\
0\le i<j\le n-m,\ j-i\ge\mathit{excl}\}$
denote the admissible comparison universe. The exact PMP evaluates
all of it, which is what the $L$ self-joins compute. For an
evaluated subset $\mathcal S\subseteq\mathcal U$, writing
$\langle i,j\rangle=(\min\{i,j\},\max\{i,j\})$, the profile
\emph{induced} by $\mathcal S$ is
\[
P_m^{\mathcal S}[i]=\min_{j:\ (m,\langle i,j\rangle)\in\mathcal S}
\dist_m(i,j),
\]
with the empty minimum equal to $\infty$ and $I_m^{\mathcal S}[i]$ a
witness of the minimum.

\begin{proposition}[Soundness]
\label{prop:sound}
For every $\mathcal S\subseteq\mathcal U$, the induced profile is a
sound upper pan-profile: every finite cell equals the exact
z-normalized distance to the valid neighbor it stores, and
$P_m^{\mathcal S}[i]\ge P_m[i]$ for every cell.
\end{proposition}
\begin{proof}
A finite cell equals $\dist_m(i,j)$ for some evaluated triple
$(m,i,j)\in\mathcal S$. Every triple of $\mathcal U$ satisfies
$j-i\ge\mathit{excl}$, so the stored neighbor is valid.
$P_m^{\mathcal S}[i]$ is a minimum over a subset of the neighbors
defining $P_m[i]$, so it cannot be smaller than $P_m[i]$.
\end{proof}

\panache{} computes $(\hat P,\hat I)=(P^{\mathcal S},I^{\mathcal S})$
for an evaluated set assembled in one pass,
\[
\mathcal S=\mathcal S_{\mathrm{stream}}
\cup\mathcal S_{\mathrm{hot}}
\cup\mathcal S_{\mathrm{local}},
\]
where $\mathcal S_{\mathrm{stream}}$ is generated during the
streaming pass (\S\ref{sec:eligibility}--\ref{sec:batching}) and
$\mathcal S_{\mathrm{hot}}$ and $\mathcal S_{\mathrm{local}}$ are
added by the emit stage (\S\ref{sec:emit}).
Proposition~\ref{prop:sound} separates correctness from coverage.
Written values are correct for any choice of $\mathcal S$; the
remaining design question is which pairs to evaluate.

\looseness=-1
Every such choice must fit within a budget fixed before the stream
starts. Per arriving sample and indexed length (\S\ref{sec:anchors}),
that budget is $O(K)$ for the spectral state, $O(BK)$ for the
signature, and at most $R$ candidate inspections ($K$: retained
spectral width, \S\ref{sec:spectra}; $B$: signature bits,
\S\ref{sec:hashing}; $R$: verification budget, \S\ref{sec:hashing}). The sections below describe the mechanisms that keep this
budget effective as $n$ grows. Anchor lengths cover the length
interval (\S\ref{sec:anchors}), occupancy control keeps the hash
directory bounded (\S\ref{sec:hashing}), and exclusion-zone batching
parallelizes the stream while keeping the output identical to serial
processing (\S\ref{sec:batching}). Emit-time
completion evaluates the oversized hash classes and the motif
neighborhoods at non-anchor lengths, completing the evaluated set
$\mathcal S$ (\S\ref{sec:emit}).
Algorithms~\ref{alg:insert} and \ref{alg:emit} (appendix) list the
update and emit procedures. Figure~\ref{fig:hero} illustrates each
stage on an ECG5000 motif pair ($i=3925$, $j=7564$, $m=20$), and its
panels are referenced throughout the section.

\subsection{Windows on a Stream}
\label{sec:eligibility}

An offline PMP computation holds the full series, so a window is
identified by where it starts. The algorithm picks a start position
and reads the $m$ samples to its right. A stream allows no such
choice. Samples arrive one at a time, and a window can be used only
once its last sample has arrived, so the latest published sample
decides which windows exist. Windows on a stream are therefore
indexed by their right endpoints. Let $\tau$ be the index of the
most recently published sample. A length-$m$ window with start $q$
completes at $\tau=q+m-1$, the moment its last sample arrives.
Until then its mean, its variance, its non-DC spectrum, and its
distance to any other window do not exist. When sample $\tau$ is
published, exactly one new window completes at every length
$m\le\tau+1$, namely $T_{\tau-m+1,m}$.

For each \emph{anchor length} $a\in A\subseteq[m_{\min},m_{\max}]$,
the subset chosen in \S\ref{sec:anchors}, \panache{} maintains one
sliding window. The length-$a$ window completes for the first time
at $\tau=a-1$, and every later sample slides it to the start
$q=\tau-a+1$ (Figure~\ref{fig:hero}A). Recomputing the window at
every slide would read all $a$ samples for a one-position change.
Instead, \panache{} maintains an online running state
(\S\ref{sec:spectra}). At each slide the oldest sample leaves on
the left and the newest enters on the right, the state is corrected
for exactly those two samples at cost $O(K)$, and the same state is
carried through the entire series. The raw samples are retained, so
the exact distance of any pair of completed windows can be computed
at any later time.

A sample does not belong to one length. Take lengths 2 and 3 on a
stream $t_0,t_1,t_2,\ldots$. When $t_1$ arrives, the first length-2
window $(t_0,t_1)$ completes. When $t_2$ arrives, that window
slides to $(t_1,t_2)$, and the first length-3 window
$(t_0,t_1,t_2)$ completes. When $t_3$ arrives, both windows slide
at once, to $(t_2,t_3)$ and $(t_1,t_2,t_3)$. The arriving sample
$t_3$ enters both windows, while each window drops a different
sample. Every later arrival behaves the same way. \panache{} holds
one running state per anchor length, each published sample advances
all of them together, at $O(K)$ per length and $O(|A|K)$ in total.

\subsection{Normalized Spectral Features}
\label{sec:spectra}

\looseness=-1
Each sliding window's state has three jobs. It must
be updated at every slide, provide the signature under which similar
windows collide, and yield the Parseval bound that rejects colliding
pairs. The normalized non-DC Fourier coefficients do all three.

For a window $T_{i,m}$, define the DFT coefficients
\[
  X_{i,m,k}=\sum_{r=0}^{m-1}t_{i+r}e^{-2\pi\mathrm{i} kr/m}.
\]
The coefficient $X_{i,m,0}$ is proportional to the window mean; the
coefficients with $k>0$ describe variation around that mean level. At
an indexed length, \panache{} retains the scaled prefix
$F_{i,m}=(X_{i,m,1},\ldots,X_{i,m,K'})/\sigma_{i,m}$, zero-padded to
the shared width $K-1$, where
\[
  K'=\min\!\bigl(K-1,\ \lceil m/2\rceil-1\bigr)
\]
is the number of independent non-DC coefficients at length $m$: by
conjugate symmetry bin $k$ and bin $m-k$ are mirror images, and for
$k\ge m$ the index aliases back onto the mean, so only bins
$1\le k\le K'$ carry new shape information (for even $m$ the Nyquist
bin $m/2$ is omitted). Real and imaginary parts are kept; the DC
coefficient is not stored.

\begin{lemma}[AC invariance]
\label{lem:ac}
Let $\hat X_{i,m,k}$ be the DFT coefficient of the z-normalized
window $Z_{i,m}$. For every $1\le k\le m-1$,
$\hat X_{i,m,k}=X_{i,m,k}/\sigma_{i,m}$.
\end{lemma}
\begin{proof}
Centering subtracts $\mu_{i,m}G_{k,m}$ from coefficient $k$, where
$G_{k,m}=\sum_{r=0}^{m-1}e^{-2\pi\mathrm{i} kr/m}$. This geometric sum is
zero for $k\not\equiv 0\pmod m$. Scaling by $\sigma_{i,m}^{-1}$
scales the remaining coefficients.
\end{proof}

The stored vector $F_{i,m}$ therefore equals the leading $K'$
independent bins of the non-DC spectrum of the z-normalized window
(Figure~\ref{fig:hero}A), the remaining entries zero. Mean subtraction
affects only the DC bin, and the division by the standard deviation
scales all bins uniformly.

\paragraph{Online update.}
A slide changes the window only at its ends. The departing sample
leaves, the arriving sample enters, and every retained sample keeps
its value while its position inside the window drops by one. This
is the fact the sliding DFT exploits. A DFT coefficient is a sum of
samples, each weighted by a phase factor $e^{-2\pi\mathrm{i}kr/m}$
that depends on its position $r$. When every position drops by one,
every phase factor is multiplied by the same fixed rotation
$e^{2\pi\mathrm{i}k/m}$. One multiplication of the coefficient
therefore re-phases all retained samples at once, and nothing is
recomputed for them. The only new value entering the sum is the
arriving sample.

Formally, let $S=\sum_r t_{i+r}$ and $Q=\sum_r t_{i+r}^2$ be the
running sum and sum of squares of the current window, with moments
$\mu_{i,m}=S/m$ and $\sigma_{i,m}^2=Q/m-\mu_{i,m}^2$. When the
window advances by one sample,
\[
\begin{aligned}
S' &= S-t_i+t_{i+m},\qquad
Q' = Q-t_i^2+t_{i+m}^2,\\
X_{i+1,m,k} &= e^{2\pi\mathrm{i} k/m}\bigl(X_{i,m,k}-t_i+t_{i+m}\bigr),\\
F_{i+1,m} &= (X_{i+1,m,1},\ldots,X_{i+1,m,K'},0,\ldots,0)/\sigma_{i+1,m},
\end{aligned}
\]
so advancing one indexed length by one sample costs $O(K)$. The
update runs in four steps. (i)~The running sums advance to $S'$ and
$Q'$, giving the new moments. (ii)~The remove--insert correction
$(X_{i,m,k}-t_i+t_{i+m})$ subtracts the departing sample and adds
the arriving one to each retained coefficient. (iii)~One rotation
by the fixed factor $e^{2\pi\mathrm{i}k/m}$ re-phases all retained
samples, as above. Steps ii--iii are the sliding
DFT~\cite{jacobsen2003sliding,jacobsen2004update}. (iv)~The rotated
coefficients are divided by the new $\sigma_{i+1,m}$ to form
$F_{i+1,m}$. Incremental updates accumulate floating-point rounding
error, so the implementation recomputes the moments and
coefficients from the raw samples every 4{,}096 slides per anchor.

\paragraph{Hash key.}
$F_{i,m}$ is the state that is hashed to a signature and committed
to a bucket of a hash directory after every slide. Windows with
similar normalized shapes subtend small angles in this truncated
spectral space (Figure~\ref{fig:hero}B), and SimHash converts small
angles into high collision probability, so similar windows tend to
share buckets (the mechanics are the subject of
\S\ref{sec:hashing}).

\paragraph{Rejection bound.}
A collision only proposes a candidate pair. The bound that filters
these pairs before any exact computation is built from the same
stored vector $F_{i,m}$. With $K'$ as defined above, the rejection
bound is
\[
  \pbound_{K,m}^2(i,j)
  =\frac{2}{m}\sum_{k=1}^{K'}|F_{i,m,k}-F_{j,m,k}|^2 .
\]

\begin{lemma}[Parseval lower bound]
\label{lem:parseval}
For equal-length windows $i,j$:\;
$\pbound_{K,m}(i,j)\le \dist_m(i,j)$.
\end{lemma}

By Parseval's theorem, the squared z-normalized distance equals
$\frac{1}{m}\sum_{k}|\hat X_{i,m,k}-\hat X_{j,m,k}|^2$ over all $m$
bins, and the DC term vanishes after z-normalization. Any subset of
the remaining terms is a lower bound on the sum. The factor $2$
accounts for conjugate-symmetric bin pairs of real signals, and the
cap on $K'$ prevents a bin and its alias from being counted twice
when $m<2K$ (proof in Appendix~\ref{app:proofs}). A candidate pair
$(i,j)$ is evaluated exactly only if $\pbound_{K,m}(i,j)$ is smaller
than $\hat P_m[i]$ or $\hat P_m[j]$, the two cells that the pair can
lower. Otherwise, the exact distance cannot lower either cell,
and the pair is skipped. On the ECG5000 example of
Figure~\ref{fig:hero}C, this test rejects 214 of the 215 candidates
in the probed buckets. Features are cached in single precision. The
bound is therefore multiplied by $1-\lambda$ before the comparison,
and Lemma~\ref{lem:f32} shows that $\lambda=10^{-4}$ preserves
validity at every magnitude where the bound can prune.

\paragraph{Spectral width selection.}
The spectral width $K$ sets how many leading Fourier coefficients
the state retains. The prefix $F_{i,m}$ above holds bins $1$
through $K'(m)$, zero-padded to the shared width $K-1$, and the
signature and the bound are both built from it, so their precision
depends on $K$. Signatures are compared bit
for bit, so every window in the directory must be hashed at one
shared width, and $K$ cannot change once insertion has begun. The
width is therefore fixed before insertion begins by a calibration.
The calibration step reads sampled windows at each of $m_{\min}$,
the midpoint length, and $m_{\max}$, in the implementation at most
256 windows for each length, and selects the smallest
$K$ satisfying
\[
  \mathbb E\!\left[
    \frac{2}{m^2}\sum_{k=1}^{K'}|\hat X_{i,m,k}|^2
  \right]\ge \eta
  \quad\text{and}\quad
  n\,\mathbb E\!\left[(1-\theta_K/\pi)^B\right]\le 2\gamma\alpha ,
\]
where $\theta_K$ is the sampled angle between $K$-dimensional
features and $R=\gamma\alpha$ is the verification budget of
\S\ref{sec:hashing}. The first condition requires the retained
coefficients to capture a fraction $\eta$ of the normalized spectral
energy, so that features separate windows. A z-normalized window
carries total spectral energy $\sum_k|\hat X_{i,m,k}|^2=m^2$, so the
criterion leaves, in expectation over the sampled windows, at most
$(1-\eta)m^2$ outside the retained bins. Since
$|a-b|^2\le 2|a|^2+2|b|^2$, the expected squared-distance mass
invisible to $\pbound$ is at most $4(1-\eta)m$, a correlation slack
of $2(1-\eta)$. Larger $\eta$ therefore tightens pruning at the cost
of a larger $K$. Retaining most of the spectral energy in a small
prefix of coefficients is the established practice of the
Fourier-indexing
tradition~\cite{agrawal1993efficient,faloutsos1994fast,rafiei1998efficient},
so we set $\eta=0.85$.
The second condition bounds the expected number of accidental
signature collisions at the target stream length $n$ by twice the
verification budget. Near-duplicate pairs are excluded from the
estimate, since collisions between near-duplicates are the intended
behavior. Classes that exceed the budget are handled by the emit
stage (\S\ref{sec:emit}). It is important to notice that the
calibration keeps no state and returns only the integer $K$, so any
sample of windows serves, including the earliest windows of a live
stream. The pass itself is unchanged. Every sample is read once, in
stream order. No choice of $K$ can invalidate a reported motif. The
width decides the cost of the search, never the validity of its
output, and recovery is measured unchanged at every admissible
width (Appendix~\ref{app:params}).

\begin{proposition}[Soundness is width-independent]
\label{prop:widthfree}
For every width $K\ge 2$: (i) the bound stays valid,
$\pbound_{K,m}(i,j)\le\dist_m(i,j)$ for every equal-length pair,
and is non-decreasing in $K$; (ii) the profile computed at width
$K$ is a sound upper pan-profile (the proof is in
Appendix~\ref{app:proofs}).
\end{proposition}

\begin{figure*}[t]
  \centering
  \makebox[\textwidth][c]{%
    \includegraphics[width=1.06\textwidth]{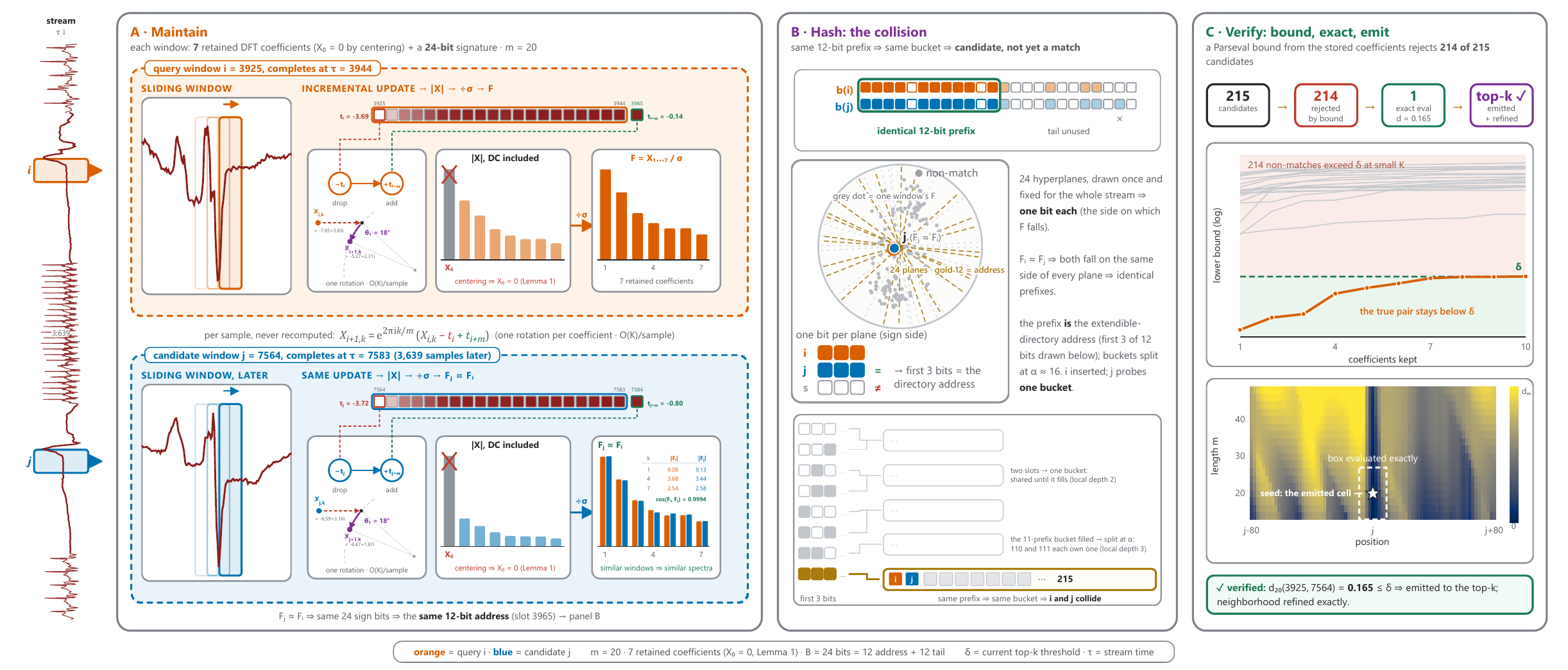}}
  \caption{\panache{} end to end on an ECG5000 motif pair ($i=3925$,
  $j=7564$, $m=20$). All values are computed from the data, and the
  left spine is the stream itself. (A)~Each sample updates the
  retained coefficients in $O(K)$: drop $t_i$, add $t_{i+m}$, one
  fixed rotation (worked example in each lane); a completed window is
  normalized to $F$ and hashed into a 12-bit directory address
  (\S\ref{sec:eligibility}--\ref{sec:hashing}). (B)~$F_j\approx
  F_i$: same prefix, same bucket. (C)~The Parseval bound rejects 214
  of 215 candidates (Lemma~\ref{lem:parseval}); the one exact
  evaluation ($\dist=0.165$) is emitted and its neighborhood refined
  (\S\ref{sec:emit}).}
  \label{fig:hero}
\end{figure*}

\subsection{Anchor Lengths}
\label{sec:anchors}

Maintaining an online index at every length in $[m_{\min},m_{\max}]$
would still be one-pass and linear in the stream length, but each
arriving sample would pay for $L$ spectral states, $L$ directories,
and $L$ probes. \panache{} optimizes further. It maintains output
rows for every length but builds online indexes only at the anchor
lengths
$A=\{m_{\min}+rs:\ m_{\min}+rs\le m_{\max}\}\cup\{m_{\max}\}$,
for a stride $s$ and integers $r\ge 0$, so the online search cost
scales with $|A|=\lceil L/s\rceil$ rather than $L$. Every experiment
in this paper uses $s=2$ (\S\ref{sec:setup}).

\looseness=-1
The assumption behind anchors is not that motifs occur at anchor
durations, but that low PMP cells occupy contiguous regions of the
length--position plane. When two subsequences match at length $m$,
their z-normalized distance after re-cutting both one sample longer
or shorter is bounded above and below by explicit functions of the
distance at length $m$~\cite{mueen2013moen}. Nearby start positions
usually remain in the same low-distance region. A region that
spans at least $s$ consecutive lengths intersects at least one
anchor. Each anchor draws its own projection vectors, so the anchors
that a region intersects provide multiple, nearly independent
opportunities for its discovery (Figure~\ref{fig:persist}).

\begin{figure}[t]
  \centering
  \includegraphics[width=\columnwidth]{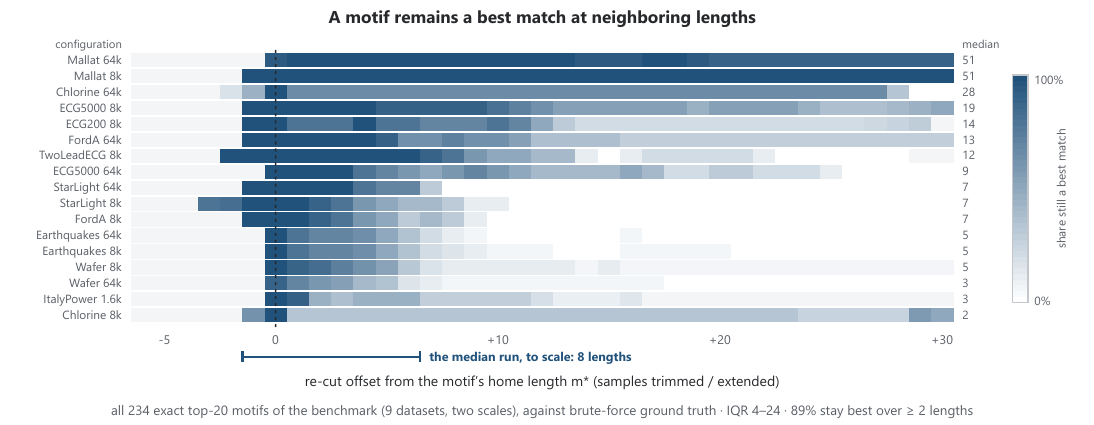}
  \caption{Empirical validation of the persistence assumption: a true
  motif remained a best match for a median of 8 consecutive lengths.}
  \label{fig:persist}
\end{figure}

\begin{proposition}[Cross-length redundancy]
\label{prop:recall}
Call a pair matchable at anchor $a$ if it is valid and would lower at
least one profile cell at length $a$, and let $M_A$ be the set of
anchors at which it is matchable. Let $p_a$ be the probability that
signatures and probes collide at anchor $a$, and $c_a$ the
probability that the budgeted scan retains the pair after collision.
Since the per-anchor projections are drawn independently
(\S\ref{sec:hashing}), the probability that no anchor discovers the
pair is at most $\prod_{a\in M_A}(1-p_ac_a)$ (proof in
Appendix~\ref{app:proofs}).
\end{proposition}

The proposition quantifies the redundancy that anchors provide.
The probabilities $p_a$ and $c_a$ are data dependent, so recall is
evaluated empirically (\S\ref{sec:eval}), while soundness
(Prop.~\ref{prop:sound}) and determinism (Prop.~\ref{prop:det})
hold unconditionally. Anchoring skips lengths on purpose, and the
skipped lengths carry no online state. They are still searched.
Every motif emitted at an anchor length starts a local completion
(\S\ref{sec:emit}), which evaluates exact distances at the
neighboring lengths and positions around the emitted pair, and the
contiguity of low regions is exactly what makes these neighborhoods
the right place to look. A low region that spans fewer than $s$
consecutive lengths can still escape every anchor. A smaller stride
or a wider completion neighborhood moves the method toward dense
coverage at higher cost.

\subsection{Streaming Candidate Generation}
\label{sec:hashing}

For every anchor length, an individual hash directory is
maintained, so a window is only ever compared against earlier
windows of the same length.
Earlier start positions are stored under a SimHash signature of
their feature vectors. Bit $\ell$ of the signature of a start $i$ at
anchor $a$ is
$b_a(i)_\ell=\mathbf 1\{\langle g_{a,\ell},F_{i,a}\rangle\ge 0\}$
for $\ell=1,\ldots,B$, where $g_{a,\ell}$
are seeded Gaussian projections drawn independently per anchor.
The per-anchor draw is what makes the anchors' discovery attempts
nearly independent (Prop.~\ref{prop:recall}).

\begin{fact}[\cite{charikar2002estimating}]
\label{fact:simhash}
For a standard Gaussian projection $g$ and vectors $u,v$ at angle
$\theta$,
$\Pr[\mathrm{sign}\langle g,u\rangle\ne
\mathrm{sign}\langle g,v\rangle]=\theta/\pi$.
\end{fact}

Windows whose normalized shapes are close therefore agree on each
signature bit with high probability (Figure~\ref{fig:hero}B). A hash
table with a fixed number of buckets, however, degrades as the stream
grows. Bucket populations increase linearly with $n$ while the
per-query budget is constant, so either the budget is exceeded or
queries compare only against recent arrivals.
\panache{} controls bucket size with a \emph{global-depth} prefix
directory, a variant of extendible hashing~\cite{fagin1979extendible}
in which every bucket shares one directory-wide depth $d$ rather than
carrying a per-bucket local depth. The bucket address of a window is
the $d$-bit prefix of its signature. When the number of stored
windows exceeds $\alpha 2^{d}$ for an occupancy target $\alpha$, and
unused signature bits remain ($d<B$), the \emph{entire} directory
doubles: the depth increases to $d+1$ and every entry is rehomed by
its next signature bit, available because full signatures are stored.
No individual bucket splits on its own and no local depths are kept;
a single depth $d$ governs the whole anchor's directory.

\begin{proposition}[Occupancy control]
\label{prop:occupancy}
Splitting whenever the population exceeds $\alpha 2^d$ keeps mean
bucket occupancy at most $\alpha$, and within $(\alpha/2,\alpha]$
between splits after the initial fill; the total rehoming work
amortizes to $O(1)$ per insertion until the signature or
implementation depth limit is reached (proof in
Appendix~\ref{app:proofs}).
\end{proposition}

\looseness=-1
Deeper prefixes reduce bucket size but also reduce the probability
that two similar windows share a full prefix. The query side
compensates by multi-probing~\cite{lv2007multiprobe}. A query
inspects its exact prefix bucket and then buckets at Hamming
distance one, in increasing order of the projection margin
$|\langle g_{a,\ell},F_{q,a}\rangle|$. A small margin indicates a bit
that was close to the decision boundary, so the corresponding
one-bit variants are the most probable alternative buckets. The
number of Hamming-1 probes grows with depth as $1+\max(0,\,d-d_0)$,
for a reference depth $d_0$, capped by the depth and an
implementation limit. Candidates from the probed buckets are
deduplicated, filtered by the exclusion rule, and inspected up to the
budget $R=\gamma\alpha$, a fixed multiple $\gamma$ of the occupancy
target $\alpha$. The target $\alpha$ is the only parameter exposed
to the analyst. When a probed bucket contains more entries than the
remaining budget, half of that budget inspects the most recent
entries and the other half inspects entries at even intervals
across the whole bucket. This avoids bias toward recent arrivals.

Let $Q_a(q)$ denote the candidates retained for query start $q$,
with $|Q_a(q)|\le R$. The streaming pass evaluates exactly the pairs
\[
\mathcal S_{\mathrm{stream}}
=\{(a,\langle i,q\rangle): i\in Q_a(q),
\pbound_{K,a}(i,q)<\max(\hat P_a[i],\hat P_a[q])\}.
\]
Each computed distance updates both cells, $\hat P_a[i]$ and
$\hat P_a[q]$. A cell and its neighbor index are overwritten
whenever the new distance is smaller.

\subsection{Batched Online Updates}
\label{sec:batching}

\looseness=-1
Batching lets \panache{} run on any number of threads without
changing its output. A batch is a group of $\beta\le\mathit{excl}$
consecutive arriving samples processed together. At a fixed length,
the windows completed by one batch have consecutive start
positions, all closer than $\mathit{excl}$, so no two of them can
form a valid pair. Windows of the same batch therefore never search
each other. The batch is published in full before any of its
searches begins. The only windows a search sees early are its own
batch-mates, every one of them invalid by the exclusion rule, and
skipping an invalid candidate costs none of the verification
budget.

\looseness=-1
A batch runs in three phases. Phase~A advances the running state of
\S\ref{sec:spectra} for all anchors in parallel and inserts the
completed windows into the directories. Each anchor owns its
moments, coefficients, and directory, so the threads share nothing.
Phase~B searches in parallel, one work unit per completed window.
Each unit probes its directory (\S\ref{sec:hashing}), rejects
candidates with the bound, computes exact distances for the
survivors, and stores the cell improvements in a private buffer.
Phase~C merges the buffers into the profile in a fixed order of
anchor and buffer index, so equal distances resolve to the same
neighbor on every run. No batch pair is valid, no state is shared,
and the merge order is fixed. Nothing in the schedule can therefore
change a result, and a window is searched at most $\beta-1$ samples
after it completes.

\begin{proposition}[Determinism]
\label{prop:det}
For a fixed input stream, configuration, and seed, the profiles and
motifs emitted by \panache{} are identical for every thread count and
schedule (proof in Appendix~\ref{app:proofs}).
\end{proposition}

\begin{table*}[!t]
\caption{\textbf{Complete benchmark table.} \panache{} columns:
time (s) and memory measure the streaming pass; top-20
(pos$+$val) is the recovered motifs split into position and value
matches; t.\ mat.\ and mem.\ mat.\ are time and memory charged
with full materialization (protocol in \S\ref{sec:setup}). Unmarked
recovery uses exact \stumpy{} fixed-exclusion ground truth,
$^\dagger$ the exact \scamp{}-GPU reference. Baselines: wall time
for the identical multi-length task. Dashes: unscored or
cost-bounded, not failed.}
\label{tab:master}
\centering
\scriptsize
\setlength{\tabcolsep}{4.8pt}
\renewcommand{\arraystretch}{0.85}
\begin{tabular}{@{}l r | r c r rr | rrrrrrr@{}}
\toprule
 & & \multicolumn{5}{c|}{\textbf{\panache}} & \multicolumn{7}{c}{baselines: wall time (s); (recovered/20) where scored} \\
\cmidrule(lr){3-7}\cmidrule(l){8-14}
Dataset & \multicolumn{1}{c}{$n$} & \multicolumn{1}{c}{time (s)} & top-20 (pos+val) & \multicolumn{1}{c}{memory} & \multicolumn{1}{c}{t.\ mat.} & \multicolumn{1}{c|}{mem.\ mat.} & \multicolumn{1}{c}{\scamp} & \multicolumn{1}{c}{\scamp{}-GPU} & \multicolumn{1}{c}{\stumpy} & \multicolumn{1}{c}{SCRUMP} & \multicolumn{1}{c}{STIMP} & \multicolumn{1}{c}{L-\textsc{Stumpi}} & \multicolumn{1}{c@{}}{\attimo} \\
\midrule
\multicolumn{14}{@{}l}{\emph{Recovery suites --- UCR (per-dataset length ranges)}} \\
ECG200 & 8k & 0.06 & 20/20 (16+4) & 37\,MB & 0.14 & 32\,MB & 0.17 & 0.85 & 4.80 (20) & 8.32 (20) & 8.29 & 50.1 & 20.8 (18) \\
ItalyPowerDemand & 1,600 & 0.01 & 20/20 (20+0) & 37\,MB & 0.02 & 37\,MB & 0.04 & 0.21 & 0.52 (20) & 0.91 (20) & 0.90 & 4.83 & 8.48 (20) \\
Wafer & 8k & 0.08 & 20/20 (19+1) & 37\,MB & 0.25 & 51\,MB & 0.29 & 1.20 & 7.45 (20) & 12.8 (20) & 13.0 & 81.6 & 117 (20) \\
Earthquakes & 8k & 0.07 & 20/20 (20+0) & 37\,MB & 0.25 & 64\,MB & 0.30 & 1.16 & 7.13 (20) & 12.8 (20) & 13.0 & 81.9 & 1,052 (20) \\
StarLightCurves & 8k & 0.08 & 20/20 (20+0) & 100\,MB & 0.23 & 133\,MB & 0.29 & 1.22 & 7.76 (20) & 12.7 (20) & 12.9 & 83.6 & 20.5 (20) \\
ChlorineConcentration & 8k & 0.08 & 20/20 (20+0) & 37\,MB & 0.22 & 65\,MB & 0.25 & 1.18 & 7.18 (20) & 12.8 (20) & 13.0 & 84.8 & 10.4 (20) \\
TwoLeadECG & 8k & 0.07 & 20/20 (20+0) & 37\,MB & \textbf{0.15} & 42\,MB & 0.14 & 0.64 & 4.62 (20) & 7.73 (20) & 7.87 & 49.8 & 7.38 (20) \\
ECG5000 & 8k & 0.06 & 20/20 (20+0) & 37\,MB & 0.14 & 38\,MB & 0.15 & 0.62 & 4.61 (20) & 7.93 (20) & 8.02 & 49.1 & 9.29 (20) \\
FordA & 8k & 0.06 & 20/20 (20+0) & 37\,MB & 0.20 & 55\,MB & 0.25 & 1.16 & 7.54 (20) & 13.2 (20) & 13.4 & 81.6 & 166 (20) \\
Mallat & 8k & 0.07 & 20/20 (15+5) & 44\,MB & 0.21 & 70\,MB & 0.26 & 1.29 & 7.79 (20) & 13.0 (20) & 13.0 & 85.3 & 295 (20) \\
\midrule
\multicolumn{14}{@{}l}{\emph{Recovery suite --- 64k, lengths 30--80}} \\
ChlorineConcentration & 64k & 1.05 & 20/20 (20+0) & 305\,MB & 2.27 & 451\,MB & 6.04 & 6.04 & -- & -- & -- & -- & -- \\
Earthquakes & 64k & 1.01 & 20/20 (20+0) & 305\,MB & 2.68 & 458\,MB & 6.07 & 5.29 & -- & -- & -- & -- & -- \\
ECG5000 & 64k & 0.80 & 20/20 (12+8) & 202\,MB & 2.08 & 364\,MB & 6.19 & 5.93 & -- & -- & -- & -- & -- \\
FordA & 64k & 0.70 & 20/20 (15+5) & 203\,MB & 1.97 & 363\,MB & 6.13 & 5.74 & -- & -- & -- & -- & -- \\
Mallat & 64k & 0.96 & 20/20 (17+3) & 317\,MB & 4.03 & 490\,MB & 6.29 & 6.18 & -- & -- & -- & -- & -- \\
StarLightCurves & 64k & 0.92 & 20/20 (20+0) & 365\,MB & 2.20 & 530\,MB & 6.30 & 5.99 & -- & -- & -- & -- & -- \\
Wafer & 64k & 0.78 & 20/20 (16+4) & 198\,MB & 1.99 & 359\,MB & 6.34 & 5.36 & -- & -- & -- & -- & -- \\
\midrule
\multicolumn{14}{@{}l}{\emph{Full natural lengths (capped at 1M)}} \\
ItalyPowerDemand & 1,608 & 0.01 & 20/20$^\dagger$ (20+0) & 37\,MB & 0.02 & 37\,MB & 0.04 & 0.15 & -- & -- & -- & -- & -- \\
ECG200 & 9,600 & 0.07 & 20/20$^\dagger$ (16+4) & 37\,MB & 0.17 & 38\,MB & 0.25 & 0.81 & -- & -- & -- & -- & -- \\
TwoLeadECG & 93,398 & 1.20 & 20/20$^\dagger$ (12+8) & 273\,MB & 2.38 & 407\,MB & 7.34 & 4.55 & -- & -- & -- & -- & -- \\
Wafer & 152k & 2.88 & 20/20$^\dagger$ (0+20) & 634\,MB & 5.98 & 1.0\,GB & 31.7 & 11.7 & -- & -- & -- & -- & -- \\
Earthquakes & 164,864 & 3.65 & 20/20$^\dagger$ (19+1) & 793\,MB & 7.06 & 1.2\,GB & 35.0 & 11.4 & -- & -- & -- & -- & -- \\
ECG5000 & 630k & 10.7 & 20/20$^\dagger$ (8+12) & 1.8\,GB & 19.6 & 2.8\,GB & 299 & 39.5 & -- & -- & -- & -- & -- \\
ChlorineConcentration & 637,440 & 17.0 & 20/20$^\dagger$ (3+17) & 2.7\,GB & 31.8 & 4.2\,GB & 473 & 67.5 & -- & -- & -- & -- & -- \\
FordA & 660k & 18.1 & 20/20$^\dagger$ (13+7) & 3.1\,GB & 33.6 & 4.7\,GB & 501 & 79.2 & -- & -- & -- & -- & -- \\
StarLightCurves & 1M & 20.7 & 20/20$^\dagger$ (1+19) & 4.6\,GB & 53.1 & 7.2\,GB & 1,213 & 144 & -- & -- & -- & -- & -- \\
Mallat & 1M & 23.7 & 20/20$^\dagger$ (0+20) & 4.6\,GB & 35.9 & 7.1\,GB & 1,213 & 143 & -- & -- & -- & -- & -- \\
\midrule
\multicolumn{14}{@{}l}{\emph{Length scaling --- Wafer (tiled beyond 152k), lengths 30--80}} \\
Wafer & 8k & 0.07 & 20/20$^\dagger$ (17+3) & 65\,MB & 0.24 & 94\,MB & 0.31 & 1.11 & -- & -- & -- & -- & -- \\
Wafer & 16k & 0.14 & 20/20$^\dagger$ (17+3) & 91\,MB & 0.41 & 139\,MB & 0.73 & 1.97 & 16.3 & 27.1 & 27.9 & 167 & 30.9 \\
Wafer & 32k & 0.32 & 20/20$^\dagger$ (19+1) & 140\,MB & 0.89 & 227\,MB & 2.04 & 3.23 & 34.0 & 57.5 & 60.5 & 335 & 106 \\
Wafer & 64k & 0.82 & 20/20$^\dagger$ (10+10) & 240\,MB & 2.03 & 400\,MB & 6.44 & 5.77 & 75.3 & 122 & 134 & 667 & 186 \\
Wafer & 128k & 2.17 & 20/20$^\dagger$ (0+20) & 567\,MB & 4.75 & 889\,MB & 22.3 & 10.2 & -- & -- & -- & -- & -- \\
Wafer & 256k & 5.80 & 20/20$^\dagger$ (0+20) & 1.2\,GB & 15.0 & 1.8\,GB & 83.9 & 24.1 & -- & -- & -- & -- & -- \\
Wafer & 512k & 13.4 & 20/20$^\dagger$ (0+20) & 2.4\,GB & 30.4 & 3.6\,GB & 316 & 59.7 & -- & -- & -- & -- & -- \\
Wafer & 1M & 28.5 & 20/20$^\dagger$ (0+20) & 4.6\,GB & 62.7 & 7.1\,GB & 1,204 & 148 & -- & -- & -- & -- & -- \\
Wafer & 2M & 62.2 & 20/20$^\dagger$ (0+20) & 9.3\,GB & 133 & 14.1\,GB & 4,744 & 451 & -- & -- & -- & -- & -- \\
Wafer & 5M & 172 & 20/20$^\dagger$ (0+20) & 23.3\,GB & 363 & 35.5\,GB & 28,638 & 2,299 & -- & -- & -- & -- & -- \\
\bottomrule
\end{tabular}

\end{table*}

\subsection{Emit-Time Completion}
\label{sec:emit}

\looseness=-1
After the streaming pass and before the motifs are reported, the
emit stage runs once to close the two gaps the budgeted pass left.
A bucket holding more entries than the budget $R$ was never scanned
in full. The anchor stride left
the lengths between anchors without any online comparison. Emit
closes the first gap with a hot-class sieve and the second with
local completion, both through exact comparisons only.

\paragraph{Hot classes ($\mathcal S_{\mathrm{hot}}$).}
\looseness=-1
In motif discovery, common patterns repeat, such as repeated
heartbeats. \panache{} piles similar windows of the same length
into the same bucket, so the bucket spills past the verification
budget, and because these windows share the same signature on every
bit, no prefix split can separate them. Somewhere in that bucket
may sit a motif that went unexamined, since every arriving window
scanned at most $R$ of the bucket's entries during the stream. A
hot class is a bucket that remains larger than the verification
budget after every usable prefix split. To find the unexamined
motifs, the $r$ spilled members of the hot class are considered.
Each member's normalized spectral state is projected onto one
seeded Gaussian direction $g$, the inner product $g\cdot F_x$
collapsing every member $x$ to one scalar, and the hot class is
sorted by it. Every pair of members that sits within $w$ positions
of each other in this sorted line-up then receives a full exact
z-normalized distance evaluation:
\[
\mathcal S_{\mathrm{hot}}(C)
=\{(a,\langle\pi(u),\pi(v)\rangle): 0<|u-v|\le w,
|\pi(u)-\pi(v)|\ge\mathit{excl}\},
\]
where $\pi$ is the sorted order of the hot class $C$. The cost per
class is $O(r\log r+rwm)$ rather than $O(r^2m)$, and profile cells
are lowered only through exact distances.

\begin{lemma}[Sieve adjacency]
\label{lem:sieve}
Let $(u,v)$ be the closest pair of a class $C$ at feature distance
$\delta$, and let
$r_x=\min(\lVert x-u\rVert,\lVert x-v\rVert)$ for other members $x$.
Under a Gaussian projection, the expected number of members landing
between $u$ and $v$ in sorted order is at most
$\frac{\delta}{2}\sum_x 1/r_x$, and the probability that the pair is
separated by more than $w$ positions is at most this quantity
divided by $w$ (proof in Appendix~\ref{app:proofs}).
\end{lemma}

The lemma bounds how many members can land between the closest
pair in the sorted order, so a small width $w$ suffices, and the
pair that went unexamined surfaces at its exact distance in the
final top-$k$ selection.

\paragraph{Local completion ($\mathcal S_{\mathrm{local}}$).}
To evaluate the non-anchor lengths, local completion uses seeds.
Each round takes as seeds the $M$ current best motif candidates of
the induced profile stack. A seed has a length $m_0$, a start $i_0$,
and a stored neighbor $j_0$.
Re-cutting this pair at nearby lengths
and starts usually keeps it close (\S\ref{sec:anchors}). In one
round,
\panache{} computes $\dist_m(i,j)$ exactly for every length $m$
within the radius $\Delta_m$ of $m_0$ and all starts $i$ and $j$
within $\Delta_p$ of $i_0$ and $j_0$. These are the admissible triples
\[
\{(m,i,j): |m-m_0|\le\Delta_m,\ |i-i_0|\le\Delta_p,\
|j-j_0|\le\Delta_p,\ |i-j|\ge\mathit{excl}\},
\]
\looseness=-1
clipped to valid ranges. Every cell $(m,i)$ or $(m,j)$ that
improves is overwritten with the new exact value and stores its
improving partner as a \emph{hint}. Because a non-anchor length
carries no index neighbor of its own, a cell draws the neighbor $j$
it re-cuts from the hash index when its length is anchored and from
this stored hint otherwise, the ``index or hint'' of the emit
refinement loop. The seeds are re-selected after every round, so a
discovery spreads one radius at a time. No round rescans the series or any length row. A seed costs at
most $(2\Delta_m{+}1)(2\Delta_p{+}1)^2$ exact distance computations
per round. At most $P_{\max}$ rounds run, stopping earlier once no
cell improves. Refinement therefore performs a fixed constant of
work per seed, independent of both $n$ and $L$
(Figure~\ref{fig:hero}C). Distances below $\varepsilon_0$ are
reported as exactly zero. Otherwise, values identical up to rounding
noise would order arbitrarily in the final top-$k$ sort. No distance
is interpolated. Every written cell comes from the raw windows at
its own length, so Proposition~\ref{prop:sound} applies to the final
output.

\looseness=-1
This completes the algorithm. The evaluated set
$\mathcal S=\mathcal S_{\mathrm{stream}}\cup\mathcal S_{\mathrm{hot}}
\cup\mathcal S_{\mathrm{local}}$ is now assembled, and \panache{}
reports the top-$k$ cells of the induced profile
$(P^{\mathcal S},I^{\mathcal S})$ under the $m/2$ overlap rule of
\S\ref{sec:background} as the discovered motifs.

\subsection{Complexity and Memory}
\label{sec:cost}

Per sample, each anchor spends $O(K)$ on the sliding state, $O(BK)$
on the signature, and at most $R$ candidate inspections. If a
fraction $\varphi$ of inspected candidates reaches exact evaluation,
the online work is
$O\!\left(n|A|\bigl(BK+R(K+\varphi m)\bigr)\right)$
for $|A|=\lceil L/s\rceil$.
The emit stage adds
\[
  \sum_{C\in\mathcal H}O(|C|\log|C|+|C|wm)
  +O\!\left(P_{\max}M(2\Delta_m+1)(2\Delta_p+1)^2m\right)
\]
over the hot classes $\mathcal H$ and the $M$ seeds refined per
round.
\looseness=-1
For a fixed length interval and configuration, the online work is
linear in $n$ and the emit stage is bounded by $O(n(\log n+wm))$, so
the total runtime is near-linear in the series length. Memory is
$O(n|A|(K+B)+nL+n)$,
covering the cached features, signatures and hash metadata, profile
rows, and retained raw samples. A materialized pan-profile has
$\Theta(nL)$ cells, so space linear in $n$ applies to any method that
produces this object; the feature cache dominates the constants. At
the largest reported run ($n=5\times10^6$, $L=51$, $|A|=26$, $K=16$),
the feature cache is approximately 15.6\,GB and peak memory is
23.3\,GB in the streaming pass and 35.5\,GB at full
materialization (Table~\ref{tab:master}, \S\ref{sec:conclusion}).

\section{Evaluation}
\label{sec:eval}
\label{sec:setup}

\looseness=-1
All CPU runs use one GCP \texttt{c3-standard-44} instance (44
vCPU, 176\,GB RAM). All methods use
their public implementations and all available threads. The data are
ten UCR datasets~\cite{dau2019ucr}, converted to float64. Wafer
scaling uses lengths 30--80 from 8k to 5M samples. Points beyond the
natural 152k length are tiled.
Recovery is scored against exact references. For ten natural-or-8k
and seven 64k configurations, exact fixed-exclusion ground truth is
computed with \stumpy{}. The full-length and Wafer scaling rows
(marked $\dagger$ in Table~\ref{tab:master}) are scored against an
exact reference computed with the official \scamp{}-GPU
implementation (NVIDIA H100 80GB, \texttt{pyscamp} 4.0.3) under the
same fixed-exclusion rule, after reconciling \scamp{}'s native
exclusion. \panache{} timings are always c3 CPU timings.

\looseness=-1
\panache{} uses the occupancy target $\alpha=16$. The verification
budget $R=\gamma\alpha=128$ ($\gamma=8$) is derived from it. The
target was left at its default in every reported experiment. Every
data-dependent quantity is computed by the algorithm itself. The
spectral width $K\in[8,16]$ comes from the calibration of
\S\ref{sec:spectra}, and the directory depth and probe count come
from the occupancy rules of \S\ref{sec:hashing}. The remaining
constants are fixed by design and were never changed in the study.
They are anchor stride $s=2$, batch size $\beta=\mathit{excl}=8$,
$B=24$ SimHash bits, probe reference depth $d_0=12$, energy threshold
$\eta=0.85$, sieve width $w=8$, refinement radii
$(\Delta_m,\Delta_p)=(6,4)$, zero-snap threshold
$\varepsilon_0=10^{-6}$, and at most eight refinement rounds, though
the sweeps show the reported \panache{} could run $7.9\times$ faster
at full recovery (Appendix~\ref{app:params}).

\looseness=-1
Reported timings are deliberately baseline-favorable. Each public
API is charged for its least expensive multi-length workflow.
Fixed-length methods are run once per length and their
times summed. Wrapper logic, JIT warmups, motif extraction and
materialization, and scoring are excluded from the measured time
of every baseline, \textbf{but not from \panache{}'s}. Its emit
stage, which materializes the profile rows, is charged in the t.\
mat.\ and mem.\ mat.\ columns of Table~\ref{tab:master}, and for a
like-for-like comparison with the baselines we also report the
time and memory columns without it, the streaming pass. When a
package lacks
documented fixed-exclusion control or comparable raw top-$k$
output, timing alone is reported. The exact fixed-length
baselines are \scamp{}~\cite{zimmerman2019scamp} and
\stumpy{}~\cite{law2019stumpy}. SCRUMP is the anytime
baseline implementing SCRIMP++~\cite{zhu2018scrimp}. STIMP is the
native pan-profile baseline~\cite{madrid2019skimp}. L-\textsc{Stumpi}
streams one length at a time.
\attimo{}~\cite{ceccarello2022attimo} runs once per length with
$k=20$. Recovery compares the exact and emitted top-20 cells at the same
length. A recovered motif is a \emph{position match} if the emitted
start lies within $m/2$ of the exact start at the same length. A
recovered motif is a \emph{value match} if it is not a position match
but the best emitted distance at that length is within $5\%$ relative
or $0.05$ absolute of the exact motif's distance.
Table~\ref{tab:recovery-stats}
in Appendix~\ref{app:recovery} reports the measured deviations on
every reference-scored row.

\subsection{Main Results}
\label{sec:mainresults}

\looseness=-1
Table~\ref{tab:master} gives the complete measured result. With one
configuration, \panache{} recovers every scored top-20 motif,
340/340 against exact \stumpy{} ground truth (310 at the
exact position) and 740/740 across the full benchmark matrix. The
largest near-duplicate
hash class, StarLightCurves at 64k, is
recovered by the hot-class sieve with 20/20 positions.
In the streaming pass, \panache{} is the fastest on every benchmark
row against every baseline. Against the exact baselines its median
speedup is $6.6\times$ over \scamp{} CPU (range $2.0$--$166.5\times$)
and $7.3\times$ over \scamp{}-GPU ($3.1$--$21.0\times$), and it runs
$90$--$1360\times$ faster than the approximate and single-length
streaming baselines (per-baseline in Table~\ref{tab:master}).
Charging its emit stage too, the exact-baseline medians become
$2.9\times$ over both. The one exception is TwoLeadECG 8k, where
\panache{}'s pass is faster but full materialization takes 13\,ms
more than \scamp{} CPU's reported $0.14$\,s, which covers the scan
but not materialization (bold in Table~\ref{tab:master}). On the
seven full-length rows over 150k samples the median speedup over the
fastest CPU baseline is $27.8\times$ streaming and $14.9\times$
materialized, so the gap widens with sample size: \panache{} reports
the motifs of Wafer 5M in 6.0 minutes (pass in 2.9), against 7.95
hours for \scamp{} CPU and 38.3 minutes for \scamp{}-GPU.

\section{Related Work and Discussion}
\label{sec:discussion}

\looseness=-1
The matrix profile line~\cite{yeh2016matrix,yeh2018unifying,
zhu2016matrix2,zhu2018scrimp,zimmerman2019scamp,law2019stumpy} made
fixed-length motif discovery exact and fast: its recurrences reuse
dot products as equal-length windows slide, but restart when the
length changes. \panache{} keeps their guarantee, an exact distance
in every written cell, and replaces per-length enumeration with one
streaming selection over all lengths.
Streaming and multi-length coverage have not coexisted. A streaming
method fixes one length up front, because keeping a cell exact means
comparing each arriving window against all earlier ones at a cost
that grows with the stream; each concedes something: an exact online
tracker~\cite{mueen2010online} keeps only a recent window; an
unbounded-stream tracker~\cite{begum2014rare} keeps the full past but
probabilistically, so a pair can be missed; LAMP
replaces the exact profile with an offline-trained
model~\cite{zimmerman2019lamp}; DAMP streams discords, not
motifs~\cite{lu2022damp}; and the matrix-profile
index~\cite{shahcheraghi2021mpindex} serves one length at a time.
Covering an interval with any of them takes $L$ runs. Multi-length
methods make the opposite tradeoff, covering all lengths but only
offline: SKIMP and STIMP schedule one profile per
length~\cite{madrid2019skimp}, and VALMOD and MOEN prune that
schedule with cross-length
bounds~\cite{linardi2018valmod,linardi2020valmodsuite,mueen2013moen}.
Others change the problem itself: MADRID targets discords, not
motifs~\cite{lu2023madrid}; ULISSE indexes variable-length queries,
not self-joins~\cite{linardi2018ulisse}; a batch method interpolates
the pan-profile under unnormalized distance~\cite{zhang2022linkump};
and grammar-, motiflet-, and warping-based methods redefine the motif
model~\cite{senin2018grammarviz,gao2019hime,
schafer2022motiflets,vanwesenbeeck2024locomotif}. None produces a
one-pass z-normalized pan-profile with exact-distance cells.
Streaming spectral summaries with bucketed search predate the matrix
profile: StatStream maintains DFT coefficients incrementally and
reports correlated stream pairs from a
grid~\cite{zhu2002statstream,cole2005uncoop}, and hashed spectral
fingerprints locate repeating segments in
audio~\cite{wang2003shazam} and, via LSH, in seismic
records~\cite{yoon2015earthquake}. All monitor correlations or
near-duplicates at one window configuration, never z-normalized
nearest-neighbor profiles over a length interval. Among motif
methods, \attimo{}~\cite{ceccarello2022attimo} is closest: LSH-based
top-$k$ discovery at one length with probabilistic guarantees,
extended to multidimensional series by
MOMENTI~\cite{ceccarello2025momenti}, but a PMP needs one run per
length. \panache{} instead processes the whole interval in one pass,
its online AC spectrum serving at once as hash key and pruning
certificate, every written cell exact under a bit-deterministic
schedule. Its remaining ingredients are
classical: SimHash and multi-probe
search~\cite{charikar2002estimating,lv2007multiprobe}, extendible
hashing~\cite{fagin1979extendible}, random
projections~\cite{buhler2002finding,chiu2003probabilistic}, and
truncated Fourier features~\cite{agrawal1993efficient,faloutsos1994fast,
rafiei1998efficient}, maintained online per window across all
lengths.

\section{Conclusion}
\label{sec:conclusion}

\looseness=-1
\panache{} replaces the $L$ fixed-length self-joins of PMP motif
discovery with one streaming pass. Every written cell is an exact distance
to a valid neighbor, and the benchmark shows full top-20 recovery
at the lowest wall time. The limitation is memory. Features,
signatures, hash metadata, samples, and profile rows grow linearly
with the stream (35.5\,GB at 5M samples). Future work will address
memory as a binding resource and a GPU implementation.

\bibliographystyle{ACM-Reference-Format}
\bibliography{refs}

\clearpage
\appendix
\section{Update and Emit Pseudocode}
\label{app:pseudocode}

Algorithm~\ref{alg:insert} lists one batch of the streaming update
(\S\ref{sec:batching}). Phase~A advances every anchor's sliding
state and inserts the new signatures, splitting any directory whose
population exceeds its occupancy bound
(Prop.~\ref{prop:occupancy}). Phase~B forms candidates from the
exact-prefix bucket and its margin-ordered Hamming-1 neighbors,
rejects them with the Parseval bound (Lemma~\ref{lem:parseval}),
computes exact distances for the survivors, and buffers the cell
updates privately. Phase~C merges the buffers in a fixed order,
which is what makes the output independent of the thread schedule
(Prop.~\ref{prop:det}). Algorithm~\ref{alg:emit} lists the emit
stage (\S\ref{sec:emit}). It runs the hot-class sieve, then
refinement to a fixpoint, then the final zero-snap and top-$k$
selection. As
throughout, $R=\gamma\alpha$ is the per-query verification budget.

\begin{algorithm}[h]
\caption{\textsc{Update}: one batch at the anchor lengths}
\label{alg:insert}
\begin{algorithmic}[1]
\footnotesize
\Require batch endpoints $\tau_0,\ldots,\tau_1$,
         $\tau_1-\tau_0+1=\beta\le\mathit{excl}$; anchors $A$
\State publish the $\beta$ samples \Comment{batch pairs invalid (\S\ref{sec:batching})}
\ForAll{$a\in A$ \textbf{in parallel}} \Comment{phase A: state}
  \For{$\tau=\tau_0,\ldots,\tau_1$}
    \If{$\tau+1<a$} \textbf{continue} \EndIf
    \State $q\gets \tau-a+1$
    \State update moments and $X_{1..K-1}$; form $F_{q,a}$
      \Comment{Lemma~\ref{lem:ac}}
    \State compute $b_a(q)$ and margins; insert $q$ by prefix
    \If{population $>\alpha2^{d_a}$} split to depth $d_a+1$
      \Comment{Prop.~\ref{prop:occupancy}}
    \EndIf
  \EndFor
\EndFor
\ForAll{$(a,\tau)\in A\times\{\tau_0,\ldots,\tau_1\}$ \textbf{in parallel}}
  \Comment{phase B: candidates}
  \If{$\tau+1<a$} \textbf{continue} \EndIf
  \State $q\gets \tau-a+1$
  \State form $Q_a(q)$ from exact and margin-ordered Hamming-1 prefixes
  \ForAll{$c\in Q_a(q)$}
    \If{$\pbound_{K,a}(q,c)\ge \max(\mathit{best}_q,\hat P_a[c])$}
      \textbf{continue} \Comment{Lemma~\ref{lem:parseval}}
    \EndIf
    \State $\delta\gets\dist_a(q,c)$ exactly
    \State update $\mathit{best}_q$; buffer candidate-endpoint update
  \EndFor
\EndFor
\ForAll{$a\in A$, fixed order} \Comment{phase C: merge}
  \State apply the query and candidate endpoint minima
\EndFor
\end{algorithmic}
\end{algorithm}

\begin{algorithm}[h]
\caption{\textsc{Emit}: sieve, refine to fixpoint, report}
\label{alg:emit}
\begin{algorithmic}[1]
\footnotesize
\ForAll{anchor $a$, bucket $C$ with $|C| > R$}
  \Comment{sieve}
  \State sort $C$ by $\langle g, F_\cdot\rangle$; verify each member
    against $w$ sort neighbors; min-update $(\hat{P}_a, \hat{I}_a)$
\EndFor
\Repeat \Comment{refinement, $\le P_{\max}$ rounds}
  \State $\mathcal{A} \gets$ top cells of the current profiles
  \ForAll{$(m, i) \in \mathcal{A}$ with neighbor $j$ (index or hint)}
    \State recheck $\{m \pm \Delta_m\} \times \{i \pm \Delta_p\}
      \times \{j \pm \Delta_p\}$ exactly; on improvement write cell
      and hint
  \EndFor
\Until{no improvement}
\State snap $\dist < \varepsilon_0$ to 0; \Return top-$k$ under the
  $m/2$ overlap rule
\end{algorithmic}
\end{algorithm}

\section{Notation}
\label{app:notation}

Table~\ref{tab:notation} collects the symbols used throughout the
paper. Calligraphic letters denote sets of comparison triples, and
hatted symbols denote maintained quantities, as opposed to their
exact counterparts.

\begin{table}[h]
\caption{Symbols used throughout the paper.}
\label{tab:notation}
\footnotesize
\begin{tabular}{@{}ll@{}}
\toprule
$T$, $n$ & time series and its length \\
$m\in[m_{\min},m_{\max}]$, $L$ & window length; interval size \\
$T_{i,m}$, $Z_{i,m}$ & window at start $i$; its z-normalized form \\
$\dist_m(i,j)$ & z-normalized Euclidean distance \\
$\mathit{excl}$ & exclusion radius for valid pairs \\
$P_m$, $I_m$; $\hat P_m$, $\hat I_m$ & exact and maintained profile rows \\
$\mathcal U$; $\mathcal S$; $P^{\mathcal S}$ & comparison universe; evaluated set; induced profile \\
$\tau$; $q$ & stream time (latest sample); window start \\
$A$, $a$, $s$ & anchor set; an anchor length; anchor stride \\
$K$, $K'$ & retained spectral width; aliasing-safe cap \\
$F_{i,m}$ & normalized non-DC spectral feature vector \\
$\pbound_{K,m}(i,j)$ & Parseval lower bound on $\dist_m(i,j)$ \\
$B$; $b_a(i)$; $g_{a,\ell}$ & signature width; a signature; a projection \\
$d$, $d_0$ & directory depth; probe reference depth \\
$\alpha$; $\gamma$; $R=\gamma\alpha$ & occupancy target; multiplier; verification budget \\
$\beta$ & micro-batch size ($\beta\le\mathit{excl}$) \\
$w$ & sieve neighbor width \\
$\Delta_m$, $\Delta_p$, $P_{\max}$ & refinement radii and round cap \\
$\eta$; $\theta_K$ & energy-capture threshold; sampled feature angle \\
$\lambda$; $\varepsilon_0$ & bound shrink factor; zero-snap threshold \\
$\varphi$ & fraction of candidates reaching exact evaluation \\
\bottomrule
\end{tabular}
\end{table}

\section{Deferred Proofs}
\label{app:proofs}

This appendix completes every claim the paper states without
proof. Six statements of \S\ref{sec:algorithm} defer their proofs
here, one lemma appears here for the first time, and one step that
\S\ref{sec:spectra} uses without derivation is carried out in
full. Lemma~\ref{lem:parseval} states that the truncated spectral sum of
\S\ref{sec:spectra} never exceeds the exact z-normalized distance,
and that the cap $K'$ prevents a frequency bin and its mirror image
from being counted twice when the window is shorter than twice the
retained width. Proposition~\ref{prop:widthfree} states that no
spectral width can invalidate an output, because the bound stays a
lower bound at every width and a width change alters only the set
of evaluated pairs. Proposition~\ref{prop:occupancy} states that
splitting a directory whenever its population crosses the occupancy
threshold keeps the mean bucket size near the target $\alpha$, at
rehoming cost that amortizes to a constant per insertion.
Proposition~\ref{prop:recall} states that the probability that
every anchor misses a matchable pair is the product of the
per-anchor miss probabilities, because each anchor draws its own
projections. Lemma~\ref{lem:sieve} states that the closest pair of
a hot class lands within a few positions of each other when the
class is sorted by one random projection. Proposition~\ref{prop:det}
states that the emitted motifs are identical for every thread count
and schedule. The new statement is Lemma~\ref{lem:f32}. The
feature cache is the largest memory consumer
(\S\ref{sec:cost}), so \panache{} stores features in single rather
than double precision (\S\ref{sec:spectra}). Rounding perturbs
every stored coefficient, and a bound computed from perturbed
features can come out larger than the true bound. Larger is the
dangerous direction, since an overestimated lower bound could
discard a pair whose exact distance would have survived the test.
The implementation therefore shrinks every computed bound by
$1-\lambda$ before the comparison, and the lemma proves the shrink
covers the rounding error whenever the bound is large enough to
prune at all. Single precision halves the cache, and the bound
stays a true lower bound. The proofs below derive every
inequality they use. The uncovered step, the vanishing geometric
sum in the proof of Lemma~\ref{lem:ac}, is completed first.

\begin{proof}[Geometric sum in Lemma~\ref{lem:ac}]
The proof in \S\ref{sec:spectra} uses that $G_{k,m}$, the sum of
$e^{-2\pi\mathrm{i}kr/m}$ over $r=0,\ldots,m-1$, vanishes when
$k\not\equiv 0 \pmod m$. The sum is a geometric series with ratio
$q = e^{-2\pi\mathrm{i}k/m}$. For $k \not\equiv 0 \pmod m$ the
ratio satisfies $q \ne 1$, so the series sums to
$G_{k,m} = (1-q^m)/(1-q)$, and $q^m = e^{-2\pi\mathrm{i}k} = 1$
makes the numerator zero. Hence $G_{k,m} = 0$.
\end{proof}

\begin{lemma}[Quantized bound safety]
\label{lem:f32}
With features stored at relative precision $\varepsilon$, shrinking
the computed bound by $(1 - \lambda)$ preserves
Lemma~\ref{lem:parseval} whenever the bound exceeds
$\sqrt{2m} \cdot 2\sqrt{2}\,\varepsilon / \lambda$.
\end{lemma}

\begin{proof}[Proof of Lemma~\ref{lem:parseval}]
Let $\hat{A}_k$ and $\hat{B}_k$ be the DFT coefficients of the
z-normalized windows $Z_{i,m}$ and $Z_{j,m}$, and write
$t_k = |\hat{A}_k - \hat{B}_k|^2$. Parseval's theorem for the DFT
states
\[
\dist_m(i,j)^2
 = \sum_{r=0}^{m-1} \bigl(Z_{i,m}[r]-Z_{j,m}[r]\bigr)^2
 = \frac{1}{m}\sum_{k=0}^{m-1} t_k .
\]
Both windows have mean zero, so $\hat{A}_0 = \hat{B}_0 = 0$ and
$t_0 = 0$. Both windows are real, so
$\hat{A}_{m-k} = \overline{\hat{A}_k}$ and
$\hat{B}_{m-k} = \overline{\hat{B}_k}$, and taking magnitudes gives
$t_{m-k} = t_k$ for $1 \le k \le m-1$. Hence
\[
2\sum_{k=1}^{K'} t_k = \sum_{k=1}^{K'} \bigl(t_k + t_{m-k}\bigr).
\]
The cap $K' \le \lceil m/2 \rceil - 1$ is equivalent to $2K' < m$.
Under it, the index sets $\{1,\ldots,K'\}$ and
$\{m{-}K',\ldots,m{-}1\}$ are disjoint subsets of
$\{1,\ldots,m{-}1\}$, so the right side above is a sum of $2K'$
distinct terms of the full Parseval sum. Every $t_k$ is
nonnegative, so
\[
2\sum_{k=1}^{K'} t_k
 \;\le\; \sum_{k=0}^{m-1} t_k
 \;=\; m\,\dist_m(i,j)^2 .
\]
By Lemma~\ref{lem:ac}, the stored features equal the coefficients
$\hat{A}_k$ and $\hat{B}_k$, so the left side equals
$m\,\pbound_{K,m}^2(i,j)$ by the definition of the bound in
\S\ref{sec:spectra}. Dividing by $m$ and taking square roots gives
$\pbound_{K,m}(i,j) \le \dist_m(i,j)$.
\end{proof}

\begin{proof}[Proof of Proposition~\ref{prop:widthfree}]
(i) In the proof of Lemma~\ref{lem:parseval} above,
$m\,\pbound_{K,m}^2(i,j)=2\sum_{k=1}^{K'}t_k$ is a sum of $2K'$
distinct nonnegative terms of the full Parseval sum
$\sum_{k=0}^{m-1}t_k=m\,\dist_m(i,j)^2$, and no step of that
argument constrains $K$. This proves the inequality for every
$K\ge 2$. For $K_1\le K_2$ the caps satisfy $K_1'\le K_2'$, and
the sub-sum at $K_2$ contains the sub-sum at $K_1$ plus
nonnegative terms, which proves the monotonicity.
(ii) An evaluated pair's exact distance is computed from raw
samples and window moments in the closed form of
\S\ref{sec:background}. No write reads the features, and the
features are read only to propose, filter, and order candidate
pairs (\S\ref{sec:hashing}, \S\ref{sec:emit}). A different width
therefore changes only the evaluated set $\mathcal S$, and
Proposition~\ref{prop:sound} gives soundness for every
$\mathcal S\subseteq\mathcal U$.
\end{proof}

\begin{proof}[Proof of Proposition~\ref{prop:occupancy}]
Let $P$ be the number of windows in one anchor's directory and let
$d$ be its prefix depth. A further split is possible while
$d$ is smaller than the signature width $B$, since deepening the
prefix uses one signature bit beyond the current depth. While a
split is possible, the split rule fires as soon as an insertion
pushes $P$ above $\alpha 2^d$, so $P \le \alpha 2^d$ holds after
every insertion, and the mean bucket occupancy $P/2^d$ is at most
$\alpha$. Before the first split the directory may be far below
this bound. After the first split, the split that brought the
directory to its current depth $d$ fired because $P$ exceeded
$\alpha 2^{d-1}$, and $P$ never decreases, so
$\alpha 2^{d-1} < P \le \alpha 2^d$ and the mean occupancy lies in
$(\alpha/2,\alpha]$. A doubling at depth $d$ rehomes at most
$\alpha 2^d + 1$ entries. The split that reached the final depth
$D$ fired because the population then exceeded $\alpha 2^{D-1}$,
which is smaller than $P_{\mathrm{final}}$, so the total rehoming
work over all doublings is at most
\[
\sum_{j=1}^{D} \bigl(\alpha 2^j + 1\bigr)
 \;\le\; 2\alpha 2^{D} + D
 \;<\; 4\,P_{\mathrm{final}} + D .
\]
This is $O(P_{\mathrm{final}})$ moves over $P_{\mathrm{final}}$
insertions, which is $O(1)$ amortized per insertion. Once $d$ reaches the signature width or the
implementation cap, no bit remains to split on, and no further
occupancy guarantee is possible without raising that limit.
\end{proof}

\begin{proof}[Proof of Proposition~\ref{prop:recall}]
Consider a matchable pair. For each anchor $a \in M_A$, let $E_a$
be the event that anchor $a$ discovers the pair, meaning that the
two windows collide in anchor $a$'s directory and that the budgeted
scan then retains the pair. By the definitions of $p_a$ and $c_a$
in the statement, $\Pr[E_a] = p_a c_a$. Treat the series as given,
so that the only remaining randomness is in the projections
$g_{a,\ell}$. Every quantity that determines $E_a$ is a function of
the data and of anchor $a$'s own projections
$g_{a,1},\ldots,g_{a,B}$ and of nothing else: the signatures of the
two windows, the bucket each occupies, the other occupants of that
bucket, and the scan's deterministic selection within it all depend
only on these (\S\ref{sec:hashing}). The projection families of
different anchors are drawn independently, so the events
$\{E_a\}_{a\in M_A}$ are functions of mutually independent random
variables and are therefore mutually independent, and so are their
complements. Hence
\[
\Pr\Bigl[\,\bigcap_{a\in M_A} E_a^{\mathrm{c}}\Bigr]
 \;=\; \prod_{a\in M_A}\Pr\bigl[E_a^{\mathrm{c}}\bigr]
 \;=\; \prod_{a\in M_A}\bigl(1-p_a c_a\bigr),
\]
which establishes the claimed bound. \qedhere
\end{proof}

\begin{proof}[Proof of Lemma~\ref{lem:sieve}]
Write $p(y) = \langle g, y\rangle$ for the projection of a feature
vector $y$, so the sieve sorts the class by $p$. A member $x$ lands
strictly between $u$ and $v$ in the sorted order exactly when
$p(x)$ lies strictly between $p(u)$ and $p(v)$, which happens
exactly when $p(x)-p(u) = \langle g, x-u\rangle$ and
$p(x)-p(v) = \langle g, x-v\rangle$ differ in sign. By
Fact~\ref{fact:simhash}, this event has probability
$\angle(x{-}u,\,x{-}v)/\pi$ over the Gaussian draw of $g$. That
angle is the angle at which $x$ sees the segment from $u$ to $v$.
The segment has length $\delta$, and $x$ lies at distance at least
$r_x$ from both of its endpoints, so the angle is largest when both
distances equal $r_x$ exactly, the isosceles configuration with
$\sin(\theta/2) = \delta/2r_x$. The function $\arcsin$ is convex on
$[0,1]$ with $\arcsin(0)=0$ and $\arcsin(1)=\pi/2$, so
$\arcsin(z) \le (\pi/2)\,z$ there. The probability that $x$ lands
between the pair is therefore at most
\[
\frac{\theta}{\pi}
 \;\le\; \frac{2\arcsin(\delta/2r_x)}{\pi}
 \;\le\; \frac{2}{\pi}\cdot\frac{\pi}{2}\cdot\frac{\delta}{2r_x}
 \;=\; \frac{\delta}{2r_x} .
\]
Let $N$ be the number of members landing strictly between $u$ and
$v$. $N$ is a sum of indicator variables, one per other member, so
linearity of expectation gives the first claim,
\[
\mathbb{E}[N] \;\le\; \frac{\delta}{2}\sum_x \frac{1}{r_x} .
\]
If $u$ and $v$ are separated by more than $w$ positions in the
sorted order, then at least $w$ members lie strictly between them,
so that event implies $N \ge w$. Outcomes with $N \ge w$ contribute
at least $w$ to the expectation, and all other outcomes contribute
at least zero, so $\mathbb{E}[N] \ge w\Pr[N \ge w]$. Dividing by
$w$ gives the second claim, $\Pr[N \ge w] \le \mathbb{E}[N]/w$.
\end{proof}

\begin{proof}[Proof of Lemma~\ref{lem:f32}]
Let $a_k$ and $b_k$ be the feature components of the two windows,
let $D^2 = \sum_k (a_k - b_k)^2$ be the true squared feature gap,
and let $\tilde{D}^2$ be the value computed from components stored
with relative error at most $\varepsilon$. The computed bound of
\S\ref{sec:spectra} is $\frac{2}{m}\tilde{D}^2$. A stored component
equals its true value times $(1+\theta)$ with
$|\theta| \le \varepsilon$, so
\[
\bigl|(\tilde{a}_k - \tilde{b}_k) - (a_k - b_k)\bigr|
 \;\le\; \varepsilon\bigl(|a_k| + |b_k|\bigr).
\]
Squaring, each term of $\tilde{D}^2$ differs from its true value by
at most $2\varepsilon (|a_k| + |b_k|) |a_k - b_k| +
O(\varepsilon^2)$. Summing over $k$ and applying the
Cauchy--Schwarz inequality with $S^2 = \sum_k (|a_k| + |b_k|)^2$
gives
\[
\bigl|\tilde{D}^2 - D^2\bigr|
 \;\le\; 2\varepsilon \sum_k (|a_k|+|b_k|)\,|a_k - b_k|
   + O(\varepsilon^2)
 \;\le\; 2\varepsilon S D + O(\varepsilon^2),
\]
so the relative error of $\tilde{D}^2$ is at most
$2\varepsilon S / D$, and the shrink $(1-\lambda)$ absorbs it
whenever $D \ge 2\varepsilon S / \lambda$. The features are stored
in single precision, so $\varepsilon = 2^{-24}$, and the shrink in
use is $\lambda = 10^{-4}$ (\S\ref{sec:spectra}), which makes
the condition $D \ge S/839$ up to rounding. A z-normalized window
carries total spectral energy
\[
\sum_{k=0}^{m-1} |\hat{X}_k|^2 \;=\; m \sum_r Z[r]^2 \;=\; m^2,
\]
and the stored features cover at most one bin of each conjugate
pair, so over the stored bins $\sum_k |\hat{X}_k|^2 \le m^2/2$.
Since $(|a_k|+|b_k|)^2 \le 2(a_k^2 + b_k^2)$, this gives
$S^2 \le 2\,(m^2/2 + m^2/2) = 2m^2$, so $S \le \sqrt{2}\,m$ and the
condition becomes
\[
\frac{2}{m} D^2 \;\ge\; \bigl(\sqrt{2m}/600\bigr)^2,
\]
the threshold of the lemma. Below the threshold the shrink no
longer covers the relative error, and the guarantee degrades
gracefully instead of failing. Dividing
$|\tilde{D}^2 - D^2| \le 2\varepsilon S D$ by $\tilde{D}+D \ge D$
and scaling by $\sqrt{2/m}$ bounds the absolute inflation of the
computed bound by $2\varepsilon S\sqrt{2/m} \le 4\sqrt{m}\,
\varepsilon$, which is below $2.2\cdot10^{-6}$ at every benchmark
length. A comparison misjudged below the threshold therefore
leaves the affected cell within $2.2\cdot10^{-6}$ of its exact
value, an error of the same order as the zero-snap threshold
$\varepsilon_0=10^{-6}$ (\S\ref{sec:emit}) and below every
precision at which distances are reported in this paper.
\end{proof}

\begin{proof}[Proof of Proposition~\ref{prop:det}]
Fix the input stream, the configuration, and the seed. It suffices
to show that the memory state after each batch is a function of the
state before it and the batch contents alone, with no dependence on
thread count or schedule, and that the emit stage is deterministic
given that state.

Within one batch (Algorithm~\ref{alg:insert}), the phase-A workers
advance the online state of \S\ref{sec:spectra}, and each worker
owns one anchor's moments, coefficients, and directory. No two
workers touch the same memory, and writes to disjoint memory
commute, so the state after phase~A is the same set of per-anchor
updates under every interleaving. The barrier between the phases
guarantees that no phase-B unit starts before every phase-A write
has completed, so each phase-B unit reads the same post-A state
under every schedule. Phase~B runs one work unit per new window,
and the set of units is fixed by the batch contents, not by the
schedule. Each unit probes the directory, rejects candidates with
the Parseval bound, and evaluates the survivors exactly, all
deterministic functions of the post-A state and the unit's
identity, and it records the resulting cell improvements only in
its own private buffer. The buffer contents after phase~B are
therefore identical for every schedule. Phase~C merges the buffers
into the shared profile sequentially, in an order fixed by anchor
and buffer index. Its inputs are schedule-independent and its order
is fixed, so the resulting profile is as well. The min-updates use
strict inequality, so a cell changes only when the new distance is
strictly smaller. When two buffered improvements tie, the one
merged first is kept, and the fixed merge order determines which
one that is, so ties resolve identically on every run.

The state after a batch is therefore a function only of the state
before it and the batch contents. The initial state is fixed by the
configuration and the seed, and induction over the batches extends
the identity to the state after the whole stream. In the emit
stage, the hot-class sweep parallelizes only across anchors, and
each anchor's sweep writes only that anchor's own profile rows, so
the disjoint-memory argument of phase~A applies unchanged. Local
completion, the zero-snap of \S\ref{sec:emit}, and the final
top-$k$ selection run sequentially in a fixed order on this
schedule-independent state. The emitted motifs are therefore
identical for every thread count and schedule.
\end{proof}

\section{Per-Row Recovery Composition}
\label{app:recovery}

Table~\ref{tab:recovery-stats} decomposes the top-20 recovery of
Table~\ref{tab:master} for the twenty rows scored against the exact
\scamp{}-GPU reference, the $\dagger$-rows of
Table~\ref{tab:master}. For each row, the \emph{reference list} is
the top-20 motifs of that exact reference, and the \emph{emitted
list} is the top-20 that \panache{} reports; recovery scoring
(\S\ref{sec:eval}) pairs them. As defined there, a reference motif
is a \emph{position match} when \panache{} reports a start within
$m/2$ of the reference start at the same length, and a \emph{value
match} when it is not a position match but \panache{} reports that
length's best distance within the tolerance of \S\ref{sec:eval},
$5\%$ relative or $0.05$ absolute. The table gives, per row, the
count of each, the smallest distance in the reference list
(ref.\ best), and the largest gap between the two lists compared
rank by rank: for each rank $r$ in $1,\ldots,20$ the absolute
difference between the $r$-th smallest emitted distance and the
$r$-th smallest reference distance, reported over the twenty ranks
both in absolute form and relative to the reference distance at
that rank.

Two facts explain every value match in the table, that is, every
reference motif that \panache{} recovers by distance but not by
position. First, several reference lists consist partly or wholly
of distance-zero cells (ref.\ best $=0$). A distance-zero cell is a
motif whose two windows have z-normalized distance exactly zero
(\S\ref{sec:background}), so they have identical normalized shape
and are duplicates of each other. When a window has several
duplicates its nearest neighbor is not unique, more than one start
is an equally valid motif, and the reference and \panache{} may
report different ones while their distances agree to below
$10^{-3}$. Such a row also has no meaningful relative error, since
dividing by a zero reference distance is undefined, which is why
\S\ref{sec:eval} admits an absolute tolerance. The tiled Wafer rows
are the extreme case. Past its natural 152k length Wafer is
extended by repeating the series (\S\ref{sec:setup}), so every
window has an exact copy, every top-20 motif sits at distance zero,
and recovery is value-based by construction. Second, on rows whose
reference distances are nonzero, a value match is a near-miss:
\panache{} reports the motif a few positions from the reference
start, at a distance that matches within tolerance. Across the
twenty rows the largest rank-wise relative gap is 10.8\%
(TwoLeadECG) and the largest absolute gap is 0.0102 (ECG200), both
inside the \S\ref{sec:eval} tolerances. Every reference motif in
the table is therefore accounted for as a position match, a
duplicate-window tie counted by value, or a near-miss within
tolerance.

\begin{table}[H]
\caption{Per-row composition of top-20 recovery on the
$\dagger$-rows: position matches (pos), value matches (val), best
reference distance, and the largest rank-wise deviation between the
emitted and reference top-20 value lists, absolute and relative.
Relative deviation is reported over ranks with reference distance at
least $0.01$; ``--'' indicates that no rank meets this threshold.
Rows with zero position matches have reference top-20 lists
consisting entirely of distance-zero cells, where duplicate windows
make position identity non-unique; on all of them the largest value
deviation is below $10^{-3}$. The
six Wafer rows from 128k to 5M have identical statistics and are
shown as one line.}
\label{tab:recovery-stats}
\scriptsize
\setlength{\tabcolsep}{3.6pt}
\renewcommand{\arraystretch}{0.9}
\begin{tabular*}{\columnwidth}{@{\extracolsep{\fill}}lrrrrrr@{}}
\toprule
Dataset & $n$ & pos & val & ref.\ best &
$\max|\Delta\dist|$ & $\max$ rel. \\
\midrule
\multicolumn{7}{@{}l}{\emph{Full natural lengths (capped at 1M)}} \\
ItalyPowerDemand & 1{,}608 & 20 & 0 & 0.0432 & $3.3{\cdot}10^{-6}$ & $<0.1\%$ \\
ECG200 & 9{,}600 & 16 & 4 & 0.1846 & 0.0102 & 5.3\% \\
TwoLeadECG & 93{,}398 & 12 & 8 & 0.0413 & 0.0045 & 10.8\% \\
Wafer & 152{,}000 & 0 & 20 & 0 & 0 & -- \\
Earthquakes & 164{,}864 & 19 & 1 & 0 & 0 & -- \\
ECG5000 & 630{,}000 & 8 & 12 & 0.0889 & 0.0063 & 7.0\% \\
ChlorineConc. & 637{,}440 & 3 & 17 & 0 & $4.0{\cdot}10^{-4}$ & -- \\
FordA & 660{,}000 & 13 & 7 & 0.0508 & 0.0049 & 8.3\% \\
StarLightCurves & 1M & 1 & 19 & 0 & $9.3{\cdot}10^{-4}$ & -- \\
Mallat & 1M & 0 & 20 & 0 & 0 & -- \\
\midrule
\multicolumn{7}{@{}l}{\emph{Wafer scaling (tiled beyond 152k)}} \\
Wafer & 8k & 17 & 3 & 0 & 0 & -- \\
Wafer & 16k & 17 & 3 & 0 & 0 & -- \\
Wafer & 32k & 19 & 1 & 0 & 0 & -- \\
Wafer & 64k & 10 & 10 & 0 & 0 & -- \\
Wafer (6 rows) & 128k--5M & 0 & 20 & 0 & 0 & -- \\
\bottomrule
\end{tabular*}
\end{table}

\section{Parameter Study}
\label{app:params}

This study reruns the complete benchmark of Table~\ref{tab:master}
while changing one design constant at a time. Every configuration
covers all 37 reported rows, is scored against the same ground
truth as \S\ref{sec:eval}, and runs on a
\texttt{c3-standard-44} instance with the fixed seed of the
reference implementation, so every run is deterministic and every
difference between configurations is caused by the changed
constant. The control configuration is unmodified \panache{}; it
reproduces Table~\ref{tab:master} exactly, with 740 of 740 motifs
recovered and 310 of 340 exact positions on the seventeen rows with
exact \stumpy{} ground truth.

Every runtime in this study is a full-materialization time, the
t.\ mat.\ accounting of Table~\ref{tab:master}. Each run executes
the emit stage of \S\ref{sec:emit} and writes every profile row,
so that the complete output can be checked cell by cell against
ground truth, which recovery scoring requires. The
materialization work is identical in every configuration, so every
ratio below divides two such runs, the shared cost cancels, and
the ratio isolates the changed constant.

Tables~\ref{tab:abl-stride}--\ref{tab:abl-alpha} report every
benchmark row under every setting. Cells follow the format of
Table~\ref{tab:master}. Each cell gives the recovered top-20 motifs
with their split into position and value matches, and the runtime
in seconds.

\textbf{The most important takeaway of this section is that
changing a design constant does not affect recovery. Only runtime
responds.} All 22 single-constant settings recover 740 of 740
motifs, and their exact-position count, out of 340, never differs
from the control's by more than 15.
Figure~\ref{fig:param-sweeps} plots the runtime response.
Table~\ref{tab:sweeps} summarizes the study.

\paragraph{Anchor stride (Table~\ref{tab:abl-stride}).}
Anchors index every $s$-th length and leave the remaining lengths
to local completion (\S\ref{sec:anchors}).
Dense indexing ($s=1$) maintains a window and a directory at every
length and recovers the same 740 motifs at $1.92\times$ the total
runtime; at Wafer 5M it takes 714\,s against 363\,s for the
default. Coarser strides cut runtime further. At $s=8$ the
benchmark completes at $0.37\times$ the control (130\,s at 5M) and
still recovers every motif, with exact positions within 4 of the
control.
Recovery does not depend on the stride anywhere in the swept range;
only cost does.

\paragraph{Sieve width (Table~\ref{tab:abl-w}).}
The sieve evaluates each member of an over-budget bucket against
its $w$ nearest neighbors in the projection order
(\S\ref{sec:emit}). It costs at most 2\% of total runtime at any
width, so this sweep isolates quality. Every width recovers 740 of 740, and widths
above the default add up to 13 exact positions (323 of 340 at
$w=32$). The one visible per-row effect is StarLightCurves 64k,
where $w=2$ yields 18 exact positions and $w\ge 8$ yields all 20.

\paragraph{Refinement radii (Table~\ref{tab:abl-radii}).}
Local completion evaluates a box of $\Delta_m$ lengths and
$\Delta_p$ starts around each seed (\S\ref{sec:emit}). Recovery and
runtime are unchanged from $(2,1)$ to $(12,8)$, and positions move
by at most two matches. Any radius that reaches the adjacent
non-anchor lengths suffices. The neighborhood size is not a
sensitive quantity.

\paragraph{Refinement rounds (Table~\ref{tab:abl-passes}).}
Refinement repeats until no cell improves, capped at $P_{\max}$
rounds (\S\ref{sec:emit}).
Exact positions are identical for every cap $P_{\max}\ge 1$.
Refinement converges in one round on every benchmark row, and later
rounds only re-verify. A cap of one round completes the benchmark
at $0.58\times$ the control (200\,s against 362\,s at Wafer 5M).
The default cap of eight is a safety margin for streams where one
round does not suffice, not a tuned value.

\paragraph{Occupancy target (Table~\ref{tab:abl-alpha}).}
The target $\alpha$ sets the bucket size the directory maintains,
and the verification budget is $R=8\alpha$ (\S\ref{sec:hashing}).
$\alpha$ is the one parameter exposed to the analyst, and this
sweep bounds its risk in both directions. A quarter of the default
verification budget ($\alpha=4$, $R=32$) still recovers 740 of 740
at $0.74\times$ runtime, and four times the default ($\alpha=64$)
buys 10 additional exact positions at $1.94\times$. Mis-setting the
only exposed knob by a factor of four in either direction costs no
motif.

\paragraph{Spectral width (Table~\ref{tab:abl-kwidth}).}
The width $K$ is the one constant the calibration of
\S\ref{sec:spectra} sets from data, so this sweep overrides it
with fixed widths across the admissible range, $K\in\{4,8,12,16\}$,
in a dedicated session on the same machine type.
Proposition~\ref{prop:widthfree} states that no width can
invalidate an output, and the sweep is the matching measurement
for coverage. Recovery is 740 of 740 at every sampled width from the
three-bin bound at $K=4$ to the full sixteen, exact positions stay
within 6 of the control, and runtime falls
as the width narrows, to $0.81\times$ at $K=4$, where the looser bound
performs $1.41\times$ the exact evaluations but every slide,
signature, and stored feature is cheaper. The width moves speed,
never discovery. The implementation floors the width at 4, below
which the $B$ signature bits would span too few feature dimensions
to control bucket occupancy (\S\ref{sec:hashing}).

\paragraph{Maximum performance (Table~\ref{tab:abl-maxperf}).}
Each sweep above moves one constant and holds the rest at their
defaults, so the fastest settings never meet in a single run. The
maximum-performance configuration combines them, taking the fastest
settings that jointly hold full recovery: anchor stride $s=6$,
occupancy target $\alpha=4$, and round cap $P_{\max}=1$. The radii
stay at $(\Delta_m,\Delta_p)=(6,4)$ because a stride of six places
non-anchor lengths up to three positions from the nearest anchor,
and any $\Delta_m$ below three would leave those lengths outside
every completion box. The sieve width stays at $w=8$ because the
width sweep shows no runtime response. Table~\ref{tab:abl-maxperf}
reports this configuration against the control, both measured in one
session on a fresh instance of the benchmark machine type. The
session is directly comparable to the sweeps. The control rerun
reproduced every recovery, position, and value count of the sweep
session from a byte-identical binary, and its total runtime moved by
half a percent. The combined configuration recovers all 740 motifs
at $5.9\times$ lower total runtime, 139\,s against 823\,s across the
benchmark. Wafer 5M completes in 52.1\,s against 362\,s, and
StarLightCurves 1M in 10.1\,s against 53.1\,s. Exact positions are
308 of 340 against the control's 310. The cost of the speed is
two positions and no motif.

A final configuration pushes the combination to the floor of the
width range. Taking $\alpha=4$ and $P_{\max}=1$ at the more
aggressive stride $s=8$ and fixing $K=4$ trades the calibration's
collision headroom for speed on top of the quartered budget, the
two thinnest margins in the design. The
configuration recovers 740 of 740 motifs, at 104.3\,s total
against the control's 823\,s, a $7.9\times$ reduction. Wafer 5M
completes in 38.2\,s against 362\,s. Exact positions are 292 of
340. This configuration was measured in its own session on a fresh
instance of the same machine type from a byte-identical binary. Its outputs
are deterministic (Prop.~\ref{prop:det}), and the control totals
reproduced across sessions within half a percent.

Table~\ref{tab:sweeps} condenses the study. The sweeps were run
after every default was fixed and did not inform any default.
Tuning constants from prior analysis of the benchmark datasets is
not part of the study design. \textbf{Staying faithful to that
design, we report the conservative configuration in
Table~\ref{tab:master}, although the maximum-performance
configuration of Table~\ref{tab:abl-maxperf} delivers the same
recovery at a further $7.9\times$ lower runtime.}

\par\vspace{2pt}
\noindent\begin{minipage}{\columnwidth}
\centering
\includegraphics[width=\columnwidth]{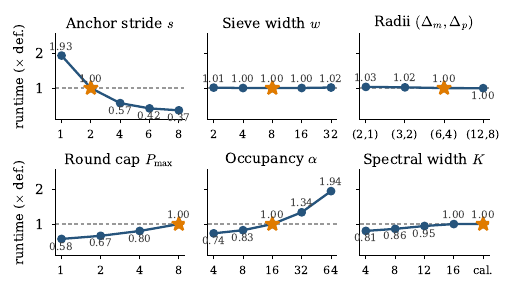}\par
\vspace{1pt}
\captionof{figure}{Runtime response to each design constant,
aggregated over the 37 benchmark rows and normalized to the default
configuration (star). Recovery is 740 of 740 at every plotted
setting; only runtime responds.}
\label{fig:param-sweeps}
\end{minipage}
\par\vspace{2pt}

\begin{table}[t]
\caption{Summary of the parameter study. Each row changes one
constant; all others stay at their defaults ($\star$). Rec.:
recovered motifs over all 37 rows (of 740). Pos.: exact-position
matches on the seventeen exact-ground-truth rows (of 340). Time:
summed runtime relative to the control.}
\label{tab:sweeps}
\footnotesize
\begin{tabular*}{\columnwidth}{@{\extracolsep{\fill}}llrrr@{}}
\toprule
Constant & Setting & Rec. & Pos. & Time \\
\midrule
Anchor stride $s$ & 1 (dense) & 740 & 313 & $1.92\times$ \\
(\S\ref{sec:anchors}) & $2^\star$ & 740 & 310 & $1.00\times$ \\
& 4 & 740 & 305 & $0.58\times$ \\
& 6 & 740 & 306 & $0.43\times$ \\
& 8 & 740 & 306 & $0.37\times$ \\
\midrule
Sieve width $w$ & 2 & 740 & 306 & $1.02\times$ \\
(\S\ref{sec:emit}) & 4 & 740 & 310 & $1.00\times$ \\
& $8^\star$ & 740 & 310 & $1.00\times$ \\
& 16 & 740 & 314 & $1.01\times$ \\
& 32 & 740 & 323 & $1.02\times$ \\
\midrule
Radii $(\Delta_m,\Delta_p)$ & $(2,1)$ & 740 & 308 & $1.05\times$ \\
(\S\ref{sec:emit}) & $(3,2)$ & 740 & 308 & $1.02\times$ \\
& $(6,4)^\star$ & 740 & 310 & $1.00\times$ \\
& $(12,8)$ & 740 & 310 & $1.00\times$ \\
\midrule
Round cap $P_{\max}$ & 1 & 740 & 310 & $0.58\times$ \\
(\S\ref{sec:emit}) & 2 & 740 & 310 & $0.67\times$ \\
& 4 & 740 & 310 & $0.81\times$ \\
& $8^\star$ & 740 & 310 & $1.00\times$ \\
\midrule
Occupancy $\alpha$ & 4 & 740 & 306 & $0.74\times$ \\
(\S\ref{sec:hashing}) & 8 & 740 & 306 & $0.83\times$ \\
& $16^\star$ & 740 & 310 & $1.00\times$ \\
& 32 & 740 & 319 & $1.34\times$ \\
& 64 & 740 & 320 & $1.93\times$ \\
\midrule
Spectral width $K$ & 4 & 740 & 315 & $0.81\times$ \\
(\S\ref{sec:spectra}) & 8 & 740 & 310 & $0.87\times$ \\
& 12 & 740 & 315 & $0.96\times$ \\
& 16 & 740 & 316 & $1.00\times$ \\
& calibrated$^\star$ & 740 & 310 & $1.00\times$ \\
\bottomrule
\end{tabular*}
\end{table}

\begin{table*}[tp]
\caption{Anchor stride sweep, complete per-row results. Each column
group is one setting of the stride $s$ with all other constants at
their defaults ($\star$ marks the default). rec (p+v): recovered
top-20 motifs and their split into position and value matches (of
20). t(s): runtime in seconds.}
\label{tab:abl-stride}
\scriptsize
\renewcommand{\arraystretch}{0.80}
\begin{tabular*}{\textwidth}{@{\extracolsep{\fill}}llrrrrrrrrrr@{}}
\toprule
Dataset & $n$ & \multicolumn{2}{c}{1} & \multicolumn{2}{c}{2$^\star$} & \multicolumn{2}{c}{4} & \multicolumn{2}{c}{6} & \multicolumn{2}{c@{}}{8} \\
 & & rec (p+v) & t(s) & rec (p+v) & t(s) & rec (p+v) & t(s) & rec (p+v) & t(s) & rec (p+v) & t(s) \\
\midrule
\multicolumn{12}{@{}l}{\emph{Recovery suites --- UCR (per-dataset length ranges)}} \\
ECG200 & 8k & 20 (16+4) & 0.25 & 20 (16+4) & 0.14 & 20 (16+4) & 0.10 & 20 (16+4) & 0.08 & 20 (16+4) & 0.07 \\
ItalyPowerDemand & 1,600 & 20 (20+0) & 0.03 & 20 (20+0) & 0.02 & 20 (20+0) & 0.02 & 20 (20+0) & 0.02 & 20 (20+0) & 0.02 \\
Wafer & 8k & 20 (19+1) & 0.36 & 20 (19+1) & 0.25 & 20 (19+1) & 0.20 & 20 (19+1) & 0.14 & 20 (19+1) & 0.14 \\
Earthquakes & 8k & 20 (20+0) & 0.47 & 20 (20+0) & 0.25 & 20 (20+0) & 0.16 & 20 (20+0) & 0.13 & 20 (20+0) & 0.11 \\
StarLightCurves & 8k & 20 (20+0) & 0.41 & 20 (20+0) & 0.23 & 20 (20+0) & 0.15 & 20 (20+0) & 0.12 & 20 (20+0) & 0.12 \\
ChlorineConc. & 8k & 20 (20+0) & 0.29 & 20 (20+0) & 0.22 & 20 (20+0) & 0.14 & 20 (20+0) & 0.11 & 20 (20+0) & 0.12 \\
TwoLeadECG & 8k & 20 (20+0) & 0.26 & 20 (20+0) & 0.15 & 20 (20+0) & 0.11 & 20 (20+0) & 0.09 & 20 (20+0) & 0.08 \\
ECG5000 & 8k & 20 (20+0) & 0.30 & 20 (20+0) & 0.14 & 20 (17+3) & 0.11 & 20 (17+3) & 0.08 & 20 (17+3) & 0.08 \\
FordA & 8k & 20 (20+0) & 0.37 & 20 (20+0) & 0.20 & 20 (20+0) & 0.13 & 20 (20+0) & 0.11 & 20 (20+0) & 0.11 \\
Mallat & 8k & 20 (13+7) & 0.37 & 20 (15+5) & 0.21 & 20 (15+5) & 0.19 & 20 (15+5) & 0.13 & 20 (15+5) & 0.10 \\
\multicolumn{12}{@{}l}{\emph{Recovery suite --- 64k, lengths 30--80}} \\
ChlorineConc. & 64k & 20 (20+0) & 3.54 & 20 (20+0) & 2.27 & 20 (20+0) & 1.33 & 20 (20+0) & 0.99 & 20 (20+0) & 0.85 \\
Earthquakes & 64k & 20 (20+0) & 5.11 & 20 (20+0) & 2.68 & 20 (20+0) & 1.50 & 20 (20+0) & 1.15 & 20 (20+0) & 1.42 \\
ECG5000 & 64k & 20 (12+8) & 4.07 & 20 (12+8) & 2.08 & 20 (6+14) & 1.19 & 20 (6+14) & 0.91 & 20 (6+14) & 1.39 \\
FordA & 64k & 20 (20+0) & 3.85 & 20 (15+5) & 1.97 & 20 (19+1) & 1.16 & 20 (20+0) & 1.03 & 20 (20+0) & 0.90 \\
Mallat & 64k & 20 (17+3) & 8.02 & 20 (17+3) & 4.03 & 20 (17+3) & 2.21 & 20 (17+3) & 1.64 & 20 (17+3) & 1.39 \\
StarLightCurves & 64k & 20 (20+0) & 4.31 & 20 (20+0) & 2.20 & 20 (20+0) & 1.29 & 20 (20+0) & 0.99 & 20 (20+0) & 0.97 \\
Wafer & 64k & 20 (16+4) & 3.89 & 20 (16+4) & 1.99 & 20 (16+4) & 1.16 & 20 (16+4) & 0.89 & 20 (16+4) & 0.76 \\
\multicolumn{12}{@{}l}{\emph{Full natural lengths (capped at 1M)}} \\
ItalyPowerDemand & 1,608 & 20 (20+0) & 0.03 & 20 (20+0) & 0.02 & 20 (20+0) & 0.02 & 20 (20+0) & 0.02 & 20 (20+0) & 0.02 \\
ECG200 & 9,600 & 20 (16+4) & 0.30 & 20 (16+4) & 0.17 & 20 (16+4) & 0.12 & 20 (16+4) & 0.10 & 20 (16+4) & 0.09 \\
TwoLeadECG & 93,398 & 20 (12+8) & 4.42 & 20 (12+8) & 2.38 & 20 (2+18) & 1.47 & 20 (8+12) & 1.09 & 20 (8+12) & 1.05 \\
Wafer & 152k & 20 (0+20) & 11.5 & 20 (0+20) & 5.98 & 20 (0+20) & 3.41 & 20 (0+20) & 2.58 & 20 (0+20) & 2.20 \\
Earthquakes & 164,864 & 20 (19+1) & 13.5 & 20 (19+1) & 7.06 & 20 (19+1) & 4.18 & 20 (19+1) & 3.56 & 20 (19+1) & 2.66 \\
ECG5000 & 630k & 20 (8+12) & 36.8 & 20 (8+12) & 19.6 & 20 (8+12) & 12.5 & 20 (6+14) & 9.02 & 20 (6+14) & 8.81 \\
ChlorineConc. & 637,440 & 20 (3+17) & 52.7 & 20 (3+17) & 31.8 & 20 (3+17) & 18.6 & 20 (3+17) & 13.8 & 20 (3+17) & 11.9 \\
FordA & 660k & 20 (13+7) & 65.1 & 20 (13+7) & 33.6 & 20 (13+7) & 19.5 & 20 (13+7) & 16.3 & 20 (13+7) & 13.9 \\
StarLightCurves & 1M & 20 (1+19) & 88.3 & 20 (1+19) & 53.1 & 20 (1+19) & 30.7 & 20 (1+19) & 22.6 & 20 (1+19) & 19.6 \\
Mallat & 1M & 20 (0+20) & 68.6 & 20 (0+20) & 35.9 & 20 (0+20) & 21.3 & 20 (0+20) & 16.6 & 20 (0+20) & 14.6 \\
\multicolumn{12}{@{}l}{\emph{Length scaling --- Wafer (tiled beyond 152k), lengths 30--80}} \\
Wafer & 8k & 20 (17+3) & 0.36 & 20 (17+3) & 0.24 & 20 (17+3) & 0.20 & 20 (17+3) & 0.14 & 20 (17+3) & 0.14 \\
Wafer & 16k & 20 (17+3) & 0.60 & 20 (17+3) & 0.41 & 20 (17+3) & 0.25 & 20 (17+3) & 0.20 & 20 (17+3) & 0.18 \\
Wafer & 32k & 20 (19+1) & 1.73 & 20 (19+1) & 0.89 & 20 (19+1) & 0.53 & 20 (19+1) & 0.42 & 20 (19+1) & 0.38 \\
Wafer & 64k & 20 (10+10) & 3.85 & 20 (10+10) & 2.03 & 20 (10+10) & 1.15 & 20 (10+10) & 0.89 & 20 (10+10) & 0.75 \\
Wafer & 128k & 20 (0+20) & 7.56 & 20 (0+20) & 4.75 & 20 (0+20) & 2.72 & 20 (0+20) & 2.05 & 20 (0+20) & 1.79 \\
Wafer & 256k & 20 (0+20) & 29.4 & 20 (0+20) & 15.0 & 20 (0+20) & 8.47 & 20 (0+20) & 6.26 & 20 (0+20) & 5.27 \\
Wafer & 512k & 20 (0+20) & 59.2 & 20 (0+20) & 30.4 & 20 (0+20) & 17.4 & 20 (0+20) & 12.7 & 20 (0+20) & 11.2 \\
Wafer & 1M & 20 (0+20) & 123 & 20 (0+20) & 62.7 & 20 (0+20) & 36.3 & 20 (0+20) & 27.0 & 20 (0+20) & 22.4 \\
Wafer & 2M & 20 (0+20) & 264 & 20 (0+20) & 133 & 20 (0+20) & 77.7 & 20 (0+20) & 56.2 & 20 (0+20) & 48.1 \\
Wafer & 5M & 20 (0+20) & 714 & 20 (0+20) & 363 & 20 (0+20) & 207 & 20 (0+20) & 152 & 20 (0+20) & 130 \\
\bottomrule
\end{tabular*}

\end{table*}

\begin{table*}[tp]
\caption{Sieve width sweep, complete per-row results. Columns as in
Table~\ref{tab:abl-stride}.}
\label{tab:abl-w}
\scriptsize
\renewcommand{\arraystretch}{0.80}
\begin{tabular*}{\textwidth}{@{\extracolsep{\fill}}llrrrrrrrrrr@{}}
\toprule
Dataset & $n$ & \multicolumn{2}{c}{2} & \multicolumn{2}{c}{4} & \multicolumn{2}{c}{8$^\star$} & \multicolumn{2}{c}{16} & \multicolumn{2}{c@{}}{32} \\
 & & rec (p+v) & t(s) & rec (p+v) & t(s) & rec (p+v) & t(s) & rec (p+v) & t(s) & rec (p+v) & t(s) \\
\midrule
\multicolumn{12}{@{}l}{\emph{Recovery suites --- UCR (per-dataset length ranges)}} \\
ECG200 & 8k & 20 (16+4) & 0.14 & 20 (16+4) & 0.14 & 20 (16+4) & 0.14 & 20 (16+4) & 0.14 & 20 (20+0) & 0.15 \\
ItalyPowerDemand & 1,600 & 20 (18+2) & 0.02 & 20 (20+0) & 0.02 & 20 (20+0) & 0.02 & 20 (20+0) & 0.02 & 20 (20+0) & 0.02 \\
Wafer & 8k & 20 (19+1) & 0.24 & 20 (19+1) & 0.24 & 20 (19+1) & 0.25 & 20 (19+1) & 0.25 & 20 (19+1) & 0.24 \\
Earthquakes & 8k & 20 (20+0) & 0.25 & 20 (20+0) & 0.25 & 20 (20+0) & 0.25 & 20 (20+0) & 0.25 & 20 (20+0) & 0.25 \\
StarLightCurves & 8k & 20 (20+0) & 0.23 & 20 (20+0) & 0.23 & 20 (20+0) & 0.23 & 20 (20+0) & 0.23 & 20 (20+0) & 0.23 \\
ChlorineConc. & 8k & 20 (20+0) & 0.21 & 20 (20+0) & 0.21 & 20 (20+0) & 0.22 & 20 (20+0) & 0.22 & 20 (20+0) & 0.21 \\
TwoLeadECG & 8k & 20 (20+0) & 0.15 & 20 (20+0) & 0.15 & 20 (20+0) & 0.15 & 20 (20+0) & 0.15 & 20 (20+0) & 0.16 \\
ECG5000 & 8k & 20 (20+0) & 0.14 & 20 (20+0) & 0.14 & 20 (20+0) & 0.14 & 20 (20+0) & 0.14 & 20 (20+0) & 0.14 \\
FordA & 8k & 20 (20+0) & 0.22 & 20 (20+0) & 0.20 & 20 (20+0) & 0.20 & 20 (20+0) & 0.21 & 20 (20+0) & 0.21 \\
Mallat & 8k & 20 (15+5) & 0.21 & 20 (15+5) & 0.21 & 20 (15+5) & 0.21 & 20 (15+5) & 0.22 & 20 (15+5) & 0.21 \\
\multicolumn{12}{@{}l}{\emph{Recovery suite --- 64k, lengths 30--80}} \\
ChlorineConc. & 64k & 20 (20+0) & 2.23 & 20 (20+0) & 2.29 & 20 (20+0) & 2.27 & 20 (20+0) & 2.24 & 20 (20+0) & 2.27 \\
Earthquakes & 64k & 20 (20+0) & 2.67 & 20 (20+0) & 2.65 & 20 (20+0) & 2.68 & 20 (20+0) & 2.63 & 20 (20+0) & 2.63 \\
ECG5000 & 64k & 20 (10+10) & 2.11 & 20 (10+10) & 2.07 & 20 (12+8) & 2.08 & 20 (12+8) & 2.15 & 20 (16+4) & 2.13 \\
FordA & 64k & 20 (17+3) & 1.98 & 20 (17+3) & 2.41 & 20 (15+5) & 1.97 & 20 (19+1) & 2.02 & 20 (20+0) & 2.02 \\
Mallat & 64k & 20 (17+3) & 4.01 & 20 (17+3) & 4.01 & 20 (17+3) & 4.03 & 20 (17+3) & 4.00 & 20 (17+3) & 4.18 \\
StarLightCurves & 64k & 20 (18+2) & 2.59 & 20 (20+0) & 2.61 & 20 (20+0) & 2.20 & 20 (20+0) & 2.26 & 20 (20+0) & 2.26 \\
Wafer & 64k & 20 (16+4) & 1.98 & 20 (16+4) & 1.99 & 20 (16+4) & 1.99 & 20 (16+4) & 2.04 & 20 (16+4) & 2.01 \\
\multicolumn{12}{@{}l}{\emph{Full natural lengths (capped at 1M)}} \\
ItalyPowerDemand & 1,608 & 20 (18+2) & 0.02 & 20 (20+0) & 0.02 & 20 (20+0) & 0.02 & 20 (20+0) & 0.02 & 20 (20+0) & 0.02 \\
ECG200 & 9,600 & 20 (16+4) & 0.17 & 20 (16+4) & 0.17 & 20 (16+4) & 0.17 & 20 (16+4) & 0.18 & 20 (20+0) & 0.18 \\
TwoLeadECG & 93,398 & 20 (12+8) & 2.28 & 20 (12+8) & 2.30 & 20 (12+8) & 2.38 & 20 (12+8) & 2.33 & 20 (13+7) & 2.38 \\
Wafer & 152k & 20 (0+20) & 5.95 & 20 (0+20) & 5.96 & 20 (0+20) & 5.98 & 20 (0+20) & 6.05 & 20 (0+20) & 6.04 \\
Earthquakes & 164,864 & 20 (19+1) & 7.16 & 20 (19+1) & 7.06 & 20 (19+1) & 7.06 & 20 (19+1) & 7.03 & 20 (19+1) & 7.07 \\
ECG5000 & 630k & 20 (8+12) & 19.5 & 20 (8+12) & 19.7 & 20 (8+12) & 19.6 & 20 (12+8) & 20.0 & 20 (12+8) & 20.0 \\
ChlorineConc. & 637,440 & 20 (3+17) & 31.9 & 20 (3+17) & 31.8 & 20 (3+17) & 31.8 & 20 (3+17) & 32.6 & 20 (3+17) & 32.4 \\
FordA & 660k & 20 (17+3) & 35.4 & 20 (17+3) & 33.6 & 20 (13+7) & 33.6 & 20 (18+2) & 33.7 & 20 (18+2) & 34.1 \\
StarLightCurves & 1M & 20 (0+20) & 61.0 & 20 (0+20) & 53.1 & 20 (1+19) & 53.1 & 20 (1+19) & 53.6 & 20 (1+19) & 54.2 \\
Mallat & 1M & 20 (0+20) & 36.1 & 20 (0+20) & 35.9 & 20 (0+20) & 35.9 & 20 (0+20) & 36.5 & 20 (0+20) & 36.9 \\
\multicolumn{12}{@{}l}{\emph{Length scaling --- Wafer (tiled beyond 152k), lengths 30--80}} \\
Wafer & 8k & 20 (17+3) & 0.24 & 20 (17+3) & 0.24 & 20 (17+3) & 0.24 & 20 (17+3) & 0.24 & 20 (17+3) & 0.24 \\
Wafer & 16k & 20 (17+3) & 0.41 & 20 (17+3) & 0.41 & 20 (17+3) & 0.41 & 20 (17+3) & 0.41 & 20 (17+3) & 0.41 \\
Wafer & 32k & 20 (19+1) & 0.89 & 20 (19+1) & 0.89 & 20 (19+1) & 0.89 & 20 (19+1) & 0.75 & 20 (19+1) & 0.76 \\
Wafer & 64k & 20 (10+10) & 2.00 & 20 (10+10) & 1.98 & 20 (10+10) & 2.03 & 20 (10+10) & 2.00 & 20 (10+10) & 2.01 \\
Wafer & 128k & 20 (0+20) & 4.86 & 20 (0+20) & 4.78 & 20 (0+20) & 4.75 & 20 (0+20) & 4.83 & 20 (0+20) & 4.82 \\
Wafer & 256k & 20 (0+20) & 15.1 & 20 (0+20) & 14.9 & 20 (0+20) & 15.0 & 20 (0+20) & 15.1 & 20 (0+20) & 15.2 \\
Wafer & 512k & 20 (0+20) & 30.0 & 20 (0+20) & 30.1 & 20 (0+20) & 30.4 & 20 (0+20) & 32.2 & 20 (0+20) & 30.8 \\
Wafer & 1M & 20 (0+20) & 63.5 & 20 (0+20) & 62.8 & 20 (0+20) & 62.7 & 20 (0+20) & 62.9 & 20 (0+20) & 63.5 \\
Wafer & 2M & 20 (0+20) & 134 & 20 (0+20) & 134 & 20 (0+20) & 133 & 20 (0+20) & 134 & 20 (0+20) & 138 \\
Wafer & 5M & 20 (0+20) & 364 & 20 (0+20) & 363 & 20 (0+20) & 363 & 20 (0+20) & 363 & 20 (0+20) & 367 \\
\bottomrule
\end{tabular*}

\end{table*}

\begin{table*}[tp]
\caption{Refinement radii sweep, complete per-row results. Columns
as in Table~\ref{tab:abl-stride}.}
\label{tab:abl-radii}
\scriptsize
\renewcommand{\arraystretch}{0.80}
\begin{tabular*}{\textwidth}{@{\extracolsep{\fill}}llrrrrrrrr@{}}
\toprule
Dataset & $n$ & \multicolumn{2}{c}{(2,1)} & \multicolumn{2}{c}{(3,2)} & \multicolumn{2}{c}{(6,4)$^\star$} & \multicolumn{2}{c@{}}{(12,8)} \\
 & & rec (p+v) & t(s) & rec (p+v) & t(s) & rec (p+v) & t(s) & rec (p+v) & t(s) \\
\midrule
\multicolumn{10}{@{}l}{\emph{Recovery suites --- UCR (per-dataset length ranges)}} \\
ECG200 & 8k & 20 (16+4) & 0.13 & 20 (16+4) & 0.13 & 20 (16+4) & 0.14 & 20 (16+4) & 0.22 \\
ItalyPowerDemand & 1,600 & 20 (20+0) & 0.02 & 20 (20+0) & 0.02 & 20 (20+0) & 0.02 & 20 (20+0) & 0.06 \\
Wafer & 8k & 20 (19+1) & 0.19 & 20 (19+1) & 0.19 & 20 (19+1) & 0.25 & 20 (19+1) & 0.36 \\
Earthquakes & 8k & 20 (20+0) & 0.24 & 20 (20+0) & 0.20 & 20 (20+0) & 0.25 & 20 (20+0) & 0.36 \\
StarLightCurves & 8k & 20 (20+0) & 0.21 & 20 (20+0) & 0.21 & 20 (20+0) & 0.23 & 20 (20+0) & 0.33 \\
ChlorineConc. & 8k & 20 (20+0) & 0.19 & 20 (20+0) & 0.19 & 20 (20+0) & 0.22 & 20 (20+0) & 0.31 \\
TwoLeadECG & 8k & 20 (20+0) & 0.16 & 20 (20+0) & 0.14 & 20 (20+0) & 0.15 & 20 (20+0) & 0.23 \\
ECG5000 & 8k & 20 (18+2) & 0.13 & 20 (18+2) & 0.13 & 20 (20+0) & 0.14 & 20 (20+0) & 0.21 \\
FordA & 8k & 20 (20+0) & 0.20 & 20 (20+0) & 0.21 & 20 (20+0) & 0.20 & 20 (20+0) & 0.29 \\
Mallat & 8k & 20 (15+5) & 0.33 & 20 (15+5) & 0.27 & 20 (15+5) & 0.21 & 20 (15+5) & 0.32 \\
\multicolumn{10}{@{}l}{\emph{Recovery suite --- 64k, lengths 30--80}} \\
ChlorineConc. & 64k & 20 (20+0) & 2.23 & 20 (20+0) & 2.23 & 20 (20+0) & 2.27 & 20 (20+0) & 2.41 \\
Earthquakes & 64k & 20 (20+0) & 2.21 & 20 (20+0) & 3.07 & 20 (20+0) & 2.68 & 20 (20+0) & 2.29 \\
ECG5000 & 64k & 20 (12+8) & 2.47 & 20 (12+8) & 2.49 & 20 (12+8) & 2.08 & 20 (12+8) & 2.17 \\
FordA & 64k & 20 (15+5) & 2.78 & 20 (15+5) & 2.37 & 20 (15+5) & 1.97 & 20 (15+5) & 2.08 \\
Mallat & 64k & 20 (17+3) & 3.96 & 20 (17+3) & 3.97 & 20 (17+3) & 4.03 & 20 (17+3) & 4.19 \\
StarLightCurves & 64k & 20 (20+0) & 2.62 & 20 (20+0) & 2.59 & 20 (20+0) & 2.20 & 20 (20+0) & 2.29 \\
Wafer & 64k & 20 (16+4) & 2.39 & 20 (16+4) & 1.97 & 20 (16+4) & 1.99 & 20 (16+4) & 2.07 \\
\multicolumn{10}{@{}l}{\emph{Full natural lengths (capped at 1M)}} \\
ItalyPowerDemand & 1,608 & 20 (20+0) & 0.02 & 20 (20+0) & 0.02 & 20 (20+0) & 0.02 & 20 (20+0) & 0.06 \\
ECG200 & 9,600 & 20 (16+4) & 0.16 & 20 (16+4) & 0.19 & 20 (16+4) & 0.17 & 20 (16+4) & 0.25 \\
TwoLeadECG & 93,398 & 20 (12+8) & 2.72 & 20 (12+8) & 2.68 & 20 (12+8) & 2.38 & 20 (12+8) & 2.37 \\
Wafer & 152k & 20 (0+20) & 5.98 & 20 (0+20) & 5.95 & 20 (0+20) & 5.98 & 20 (0+20) & 6.21 \\
Earthquakes & 164,864 & 20 (19+1) & 7.04 & 20 (19+1) & 7.06 & 20 (19+1) & 7.06 & 20 (19+1) & 7.13 \\
ECG5000 & 630k & 20 (8+12) & 25.3 & 20 (8+12) & 22.4 & 20 (8+12) & 19.6 & 20 (8+12) & 19.8 \\
ChlorineConc. & 637,440 & 20 (3+17) & 31.8 & 20 (3+17) & 32.0 & 20 (3+17) & 31.8 & 20 (3+17) & 32.1 \\
FordA & 660k & 20 (13+7) & 38.7 & 20 (13+7) & 38.8 & 20 (13+7) & 33.6 & 20 (13+7) & 33.8 \\
StarLightCurves & 1M & 20 (1+19) & 77.0 & 20 (1+19) & 61.3 & 20 (1+19) & 53.1 & 20 (1+19) & 45.6 \\
Mallat & 1M & 20 (0+20) & 35.9 & 20 (0+20) & 35.9 & 20 (0+20) & 35.9 & 20 (0+20) & 35.9 \\
\multicolumn{10}{@{}l}{\emph{Length scaling --- Wafer (tiled beyond 152k), lengths 30--80}} \\
Wafer & 8k & 20 (17+3) & 0.20 & 20 (17+3) & 0.19 & 20 (17+3) & 0.24 & 20 (17+3) & 0.36 \\
Wafer & 16k & 20 (17+3) & 0.39 & 20 (17+3) & 0.33 & 20 (17+3) & 0.41 & 20 (17+3) & 0.49 \\
Wafer & 32k & 20 (19+1) & 0.91 & 20 (19+1) & 0.92 & 20 (19+1) & 0.89 & 20 (19+1) & 0.97 \\
Wafer & 64k & 20 (10+10) & 2.36 & 20 (10+10) & 2.14 & 20 (10+10) & 2.03 & 20 (10+10) & 2.06 \\
Wafer & 128k & 20 (0+20) & 4.74 & 20 (0+20) & 4.88 & 20 (0+20) & 4.75 & 20 (0+20) & 4.85 \\
Wafer & 256k & 20 (0+20) & 15.0 & 20 (0+20) & 15.0 & 20 (0+20) & 15.0 & 20 (0+20) & 15.2 \\
Wafer & 512k & 20 (0+20) & 30.2 & 20 (0+20) & 30.3 & 20 (0+20) & 30.4 & 20 (0+20) & 30.4 \\
Wafer & 1M & 20 (0+20) & 62.6 & 20 (0+20) & 63.3 & 20 (0+20) & 62.7 & 20 (0+20) & 63.0 \\
Wafer & 2M & 20 (0+20) & 134 & 20 (0+20) & 134 & 20 (0+20) & 133 & 20 (0+20) & 134 \\
Wafer & 5M & 20 (0+20) & 363 & 20 (0+20) & 362 & 20 (0+20) & 363 & 20 (0+20) & 364 \\
\bottomrule
\end{tabular*}

\end{table*}

\begin{table*}[tp]
\caption{Refinement round cap sweep, complete per-row results.
Columns as in Table~\ref{tab:abl-stride}.}
\label{tab:abl-passes}
\scriptsize
\renewcommand{\arraystretch}{0.80}
\begin{tabular*}{\textwidth}{@{\extracolsep{\fill}}llrrrrrrrr@{}}
\toprule
Dataset & $n$ & \multicolumn{2}{c}{1} & \multicolumn{2}{c}{2} & \multicolumn{2}{c}{4} & \multicolumn{2}{c@{}}{8$^\star$} \\
 & & rec (p+v) & t(s) & rec (p+v) & t(s) & rec (p+v) & t(s) & rec (p+v) & t(s) \\
\midrule
\multicolumn{10}{@{}l}{\emph{Recovery suites --- UCR (per-dataset length ranges)}} \\
ECG200 & 8k & 20 (16+4) & 0.09 & 20 (16+4) & 0.12 & 20 (16+4) & 0.15 & 20 (16+4) & 0.14 \\
ItalyPowerDemand & 1,600 & 20 (20+0) & 0.01 & 20 (20+0) & 0.02 & 20 (20+0) & 0.02 & 20 (20+0) & 0.02 \\
Wafer & 8k & 20 (19+1) & 0.12 & 20 (19+1) & 0.16 & 20 (19+1) & 0.24 & 20 (19+1) & 0.25 \\
Earthquakes & 8k & 20 (20+0) & 0.12 & 20 (20+0) & 0.17 & 20 (20+0) & 0.25 & 20 (20+0) & 0.25 \\
StarLightCurves & 8k & 20 (20+0) & 0.14 & 20 (20+0) & 0.18 & 20 (20+0) & 0.23 & 20 (20+0) & 0.23 \\
ChlorineConc. & 8k & 20 (20+0) & 0.12 & 20 (20+0) & 0.16 & 20 (20+0) & 0.21 & 20 (20+0) & 0.22 \\
TwoLeadECG & 8k & 20 (20+0) & 0.10 & 20 (20+0) & 0.13 & 20 (20+0) & 0.15 & 20 (20+0) & 0.15 \\
ECG5000 & 8k & 20 (20+0) & 0.09 & 20 (20+0) & 0.12 & 20 (20+0) & 0.14 & 20 (20+0) & 0.14 \\
FordA & 8k & 20 (20+0) & 0.11 & 20 (20+0) & 0.16 & 20 (20+0) & 0.21 & 20 (20+0) & 0.20 \\
Mallat & 8k & 20 (15+5) & 0.12 & 20 (15+5) & 0.16 & 20 (15+5) & 0.21 & 20 (15+5) & 0.21 \\
\multicolumn{10}{@{}l}{\emph{Recovery suite --- 64k, lengths 30--80}} \\
ChlorineConc. & 64k & 20 (20+0) & 1.43 & 20 (20+0) & 1.84 & 20 (20+0) & 2.27 & 20 (20+0) & 2.27 \\
Earthquakes & 64k & 20 (20+0) & 1.39 & 20 (20+0) & 1.81 & 20 (20+0) & 2.68 & 20 (20+0) & 2.68 \\
ECG5000 & 64k & 20 (12+8) & 1.30 & 20 (12+8) & 1.65 & 20 (12+8) & 2.09 & 20 (12+8) & 2.08 \\
FordA & 64k & 20 (15+5) & 1.20 & 20 (15+5) & 1.57 & 20 (15+5) & 1.99 & 20 (15+5) & 1.97 \\
Mallat & 64k & 20 (17+3) & 1.38 & 20 (17+3) & 1.74 & 20 (17+3) & 2.52 & 20 (17+3) & 4.03 \\
StarLightCurves & 64k & 20 (20+0) & 1.37 & 20 (20+0) & 1.78 & 20 (20+0) & 2.24 & 20 (20+0) & 2.20 \\
Wafer & 64k & 20 (16+4) & 1.18 & 20 (16+4) & 1.58 & 20 (16+4) & 2.04 & 20 (16+4) & 1.99 \\
\multicolumn{10}{@{}l}{\emph{Full natural lengths (capped at 1M)}} \\
ItalyPowerDemand & 1,608 & 20 (20+0) & 0.01 & 20 (20+0) & 0.02 & 20 (20+0) & 0.02 & 20 (20+0) & 0.02 \\
ECG200 & 9,600 & 20 (16+4) & 0.11 & 20 (16+4) & 0.15 & 20 (16+4) & 0.19 & 20 (16+4) & 0.17 \\
TwoLeadECG & 93,398 & 20 (12+8) & 1.55 & 20 (12+8) & 1.93 & 20 (12+8) & 2.30 & 20 (12+8) & 2.38 \\
Wafer & 152k & 20 (0+20) & 3.95 & 20 (0+20) & 4.96 & 20 (0+20) & 5.96 & 20 (0+20) & 5.98 \\
Earthquakes & 164,864 & 20 (19+1) & 4.91 & 20 (19+1) & 5.93 & 20 (19+1) & 7.04 & 20 (19+1) & 7.06 \\
ECG5000 & 630k & 20 (8+12) & 14.2 & 20 (8+12) & 16.8 & 20 (8+12) & 19.6 & 20 (8+12) & 19.6 \\
ChlorineConc. & 637,440 & 20 (3+17) & 22.4 & 20 (3+17) & 27.3 & 20 (3+17) & 32.0 & 20 (3+17) & 31.8 \\
FordA & 660k & 20 (13+7) & 23.4 & 20 (13+7) & 28.5 & 20 (13+7) & 33.6 & 20 (13+7) & 33.6 \\
StarLightCurves & 1M & 20 (1+19) & 29.5 & 20 (1+19) & 37.6 & 20 (1+19) & 53.2 & 20 (1+19) & 53.1 \\
Mallat & 1M & 20 (0+20) & 29.9 & 20 (0+20) & 35.8 & 20 (0+20) & 36.0 & 20 (0+20) & 35.9 \\
\multicolumn{10}{@{}l}{\emph{Length scaling --- Wafer (tiled beyond 152k), lengths 30--80}} \\
Wafer & 8k & 20 (17+3) & 0.12 & 20 (17+3) & 0.16 & 20 (17+3) & 0.25 & 20 (17+3) & 0.24 \\
Wafer & 16k & 20 (17+3) & 0.24 & 20 (17+3) & 0.33 & 20 (17+3) & 0.42 & 20 (17+3) & 0.41 \\
Wafer & 32k & 20 (19+1) & 0.52 & 20 (19+1) & 0.70 & 20 (19+1) & 0.89 & 20 (19+1) & 0.89 \\
Wafer & 64k & 20 (10+10) & 1.20 & 20 (10+10) & 1.58 & 20 (10+10) & 2.00 & 20 (10+10) & 2.03 \\
Wafer & 128k & 20 (0+20) & 3.11 & 20 (0+20) & 3.95 & 20 (0+20) & 4.89 & 20 (0+20) & 4.75 \\
Wafer & 256k & 20 (0+20) & 7.04 & 20 (0+20) & 8.29 & 20 (0+20) & 10.4 & 20 (0+20) & 15.0 \\
Wafer & 512k & 20 (0+20) & 15.4 & 20 (0+20) & 17.7 & 20 (0+20) & 21.7 & 20 (0+20) & 30.4 \\
Wafer & 1M & 20 (0+20) & 33.9 & 20 (0+20) & 37.5 & 20 (0+20) & 47.8 & 20 (0+20) & 62.7 \\
Wafer & 2M & 20 (0+20) & 73.6 & 20 (0+20) & 81.5 & 20 (0+20) & 98.6 & 20 (0+20) & 133 \\
Wafer & 5M & 20 (0+20) & 200 & 20 (0+20) & 223 & 20 (0+20) & 269 & 20 (0+20) & 363 \\
\bottomrule
\end{tabular*}

\end{table*}

\begin{table*}[tp]
\caption{Occupancy target sweep, complete per-row results. Columns
as in Table~\ref{tab:abl-stride}.}
\label{tab:abl-alpha}
\scriptsize
\renewcommand{\arraystretch}{0.80}
\begin{tabular*}{\textwidth}{@{\extracolsep{\fill}}llrrrrrrrrrr@{}}
\toprule
Dataset & $n$ & \multicolumn{2}{c}{4} & \multicolumn{2}{c}{8} & \multicolumn{2}{c}{16$^\star$} & \multicolumn{2}{c}{32} & \multicolumn{2}{c@{}}{64} \\
 & & rec (p+v) & t(s) & rec (p+v) & t(s) & rec (p+v) & t(s) & rec (p+v) & t(s) & rec (p+v) & t(s) \\
\midrule
\multicolumn{12}{@{}l}{\emph{Recovery suites --- UCR (per-dataset length ranges)}} \\
ECG200 & 8k & 20 (16+4) & 0.14 & 20 (16+4) & 0.14 & 20 (16+4) & 0.14 & 20 (20+0) & 0.15 & 20 (20+0) & 0.16 \\
ItalyPowerDemand & 1,600 & 20 (19+1) & 0.02 & 20 (20+0) & 0.02 & 20 (20+0) & 0.02 & 20 (20+0) & 0.02 & 20 (20+0) & 0.02 \\
Wafer & 8k & 20 (19+1) & 0.24 & 20 (19+1) & 0.24 & 20 (19+1) & 0.25 & 20 (19+1) & 0.25 & 20 (18+2) & 0.26 \\
Earthquakes & 8k & 20 (20+0) & 0.20 & 20 (20+0) & 0.20 & 20 (20+0) & 0.25 & 20 (20+0) & 0.35 & 20 (20+0) & 0.38 \\
StarLightCurves & 8k & 20 (20+0) & 0.21 & 20 (20+0) & 0.22 & 20 (20+0) & 0.23 & 20 (20+0) & 0.22 & 20 (20+0) & 0.24 \\
ChlorineConc. & 8k & 20 (20+0) & 0.21 & 20 (20+0) & 0.20 & 20 (20+0) & 0.22 & 20 (20+0) & 0.22 & 20 (20+0) & 0.23 \\
TwoLeadECG & 8k & 20 (20+0) & 0.14 & 20 (20+0) & 0.15 & 20 (20+0) & 0.15 & 20 (20+0) & 0.16 & 20 (20+0) & 0.16 \\
ECG5000 & 8k & 20 (20+0) & 0.14 & 20 (20+0) & 0.14 & 20 (20+0) & 0.14 & 20 (20+0) & 0.15 & 20 (20+0) & 0.16 \\
FordA & 8k & 20 (20+0) & 0.21 & 20 (20+0) & 0.20 & 20 (20+0) & 0.20 & 20 (20+0) & 0.21 & 20 (20+0) & 0.22 \\
Mallat & 8k & 20 (15+5) & 0.22 & 20 (15+5) & 0.21 & 20 (15+5) & 0.21 & 20 (15+5) & 0.22 & 20 (15+5) & 0.24 \\
\multicolumn{12}{@{}l}{\emph{Recovery suite --- 64k, lengths 30--80}} \\
ChlorineConc. & 64k & 20 (20+0) & 1.90 & 20 (20+0) & 2.02 & 20 (20+0) & 2.27 & 20 (20+0) & 2.65 & 20 (20+0) & 3.32 \\
Earthquakes & 64k & 20 (20+0) & 2.36 & 20 (20+0) & 2.47 & 20 (20+0) & 2.68 & 20 (20+0) & 3.04 & 20 (20+0) & 3.74 \\
ECG5000 & 64k & 20 (6+14) & 1.94 & 20 (8+12) & 2.08 & 20 (12+8) & 2.08 & 20 (12+8) & 2.37 & 20 (14+6) & 2.82 \\
FordA & 64k & 20 (18+2) & 1.83 & 20 (14+6) & 1.86 & 20 (15+5) & 1.97 & 20 (20+0) & 2.63 & 20 (20+0) & 2.68 \\
Mallat & 64k & 20 (17+3) & 3.74 & 20 (17+3) & 3.80 & 20 (17+3) & 4.03 & 20 (17+3) & 4.41 & 20 (17+3) & 4.33 \\
StarLightCurves & 64k & 20 (20+0) & 1.88 & 20 (20+0) & 1.98 & 20 (20+0) & 2.20 & 20 (20+0) & 2.62 & 20 (20+0) & 3.42 \\
Wafer & 64k & 20 (16+4) & 1.47 & 20 (17+3) & 2.07 & 20 (16+4) & 1.99 & 20 (16+4) & 1.85 & 20 (16+4) & 2.30 \\
\multicolumn{12}{@{}l}{\emph{Full natural lengths (capped at 1M)}} \\
ItalyPowerDemand & 1,608 & 20 (19+1) & 0.02 & 20 (20+0) & 0.02 & 20 (20+0) & 0.02 & 20 (20+0) & 0.02 & 20 (20+0) & 0.02 \\
ECG200 & 9,600 & 20 (16+4) & 0.17 & 20 (16+4) & 0.17 & 20 (16+4) & 0.17 & 20 (20+0) & 0.19 & 20 (20+0) & 0.21 \\
TwoLeadECG & 93,398 & 20 (12+8) & 1.89 & 20 (6+14) & 2.03 & 20 (12+8) & 2.38 & 20 (12+8) & 2.86 & 20 (20+0) & 2.86 \\
Wafer & 152k & 20 (0+20) & 4.83 & 20 (0+20) & 5.42 & 20 (0+20) & 5.98 & 20 (0+20) & 6.41 & 20 (0+20) & 8.01 \\
Earthquakes & 164,864 & 20 (20+0) & 5.55 & 20 (19+1) & 6.38 & 20 (19+1) & 7.06 & 20 (19+1) & 8.39 & 20 (19+1) & 10.8 \\
ECG5000 & 630k & 20 (2+18) & 17.5 & 20 (6+14) & 16.1 & 20 (8+12) & 19.6 & 20 (14+6) & 26.6 & 20 (12+8) & 39.1 \\
ChlorineConc. & 637,440 & 20 (3+17) & 23.6 & 20 (3+17) & 26.4 & 20 (3+17) & 31.8 & 20 (3+17) & 42.9 & 20 (3+17) & 61.2 \\
FordA & 660k & 20 (18+2) & 24.2 & 20 (15+5) & 27.5 & 20 (13+7) & 33.6 & 20 (15+5) & 37.4 & 20 (17+3) & 49.9 \\
StarLightCurves & 1M & 20 (1+19) & 43.1 & 20 (1+19) & 46.8 & 20 (1+19) & 53.1 & 20 (1+19) & 66.8 & 20 (1+19) & 96.1 \\
Mallat & 1M & 20 (0+20) & 23.5 & 20 (0+20) & 28.2 & 20 (0+20) & 35.9 & 20 (0+20) & 51.3 & 20 (0+20) & 80.4 \\
\multicolumn{12}{@{}l}{\emph{Length scaling --- Wafer (tiled beyond 152k), lengths 30--80}} \\
Wafer & 8k & 20 (17+3) & 0.23 & 20 (17+3) & 0.23 & 20 (17+3) & 0.24 & 20 (17+3) & 0.25 & 20 (17+3) & 0.26 \\
Wafer & 16k & 20 (17+3) & 0.38 & 20 (17+3) & 0.39 & 20 (17+3) & 0.41 & 20 (17+3) & 0.44 & 20 (17+3) & 0.39 \\
Wafer & 32k & 20 (19+1) & 0.85 & 20 (19+1) & 0.86 & 20 (19+1) & 0.89 & 20 (18+2) & 0.81 & 20 (18+2) & 0.95 \\
Wafer & 64k & 20 (10+10) & 1.45 & 20 (10+10) & 1.93 & 20 (10+10) & 2.03 & 20 (10+10) & 1.86 & 20 (10+10) & 2.28 \\
Wafer & 128k & 20 (0+20) & 4.00 & 20 (0+20) & 4.35 & 20 (0+20) & 4.75 & 20 (0+20) & 5.24 & 20 (0+20) & 6.47 \\
Wafer & 256k & 20 (0+20) & 12.1 & 20 (0+20) & 13.1 & 20 (0+20) & 15.0 & 20 (0+20) & 18.0 & 20 (0+20) & 19.1 \\
Wafer & 512k & 20 (0+20) & 23.0 & 20 (0+20) & 25.9 & 20 (0+20) & 30.4 & 20 (0+20) & 38.5 & 20 (0+20) & 48.0 \\
Wafer & 1M & 20 (0+20) & 46.7 & 20 (0+20) & 52.0 & 20 (0+20) & 62.7 & 20 (0+20) & 82.6 & 20 (0+20) & 118 \\
Wafer & 2M & 20 (0+20) & 97.4 & 20 (0+20) & 110 & 20 (0+20) & 133 & 20 (0+20) & 181 & 20 (0+20) & 264 \\
Wafer & 5M & 20 (0+20) & 261 & 20 (0+20) & 294 & 20 (0+20) & 363 & 20 (0+20) & 506 & 20 (0+20) & 754 \\
\bottomrule
\end{tabular*}

\end{table*}

\begin{table*}[tp]
\caption{Spectral width sweep, complete per-row results. Each column
group is one fixed spectral width $K$; the calibrated control
($\star$) sets $K$ per length by the rule of \S\ref{sec:spectra}.
Columns as in Table~\ref{tab:abl-stride}.}
\label{tab:abl-kwidth}
\scriptsize
\renewcommand{\arraystretch}{0.80}
\begin{tabular*}{\textwidth}{@{\extracolsep{\fill}}llrrrrrrrrrr@{}}
\toprule
Dataset & $n$ & \multicolumn{2}{c}{4} & \multicolumn{2}{c}{8} & \multicolumn{2}{c}{12} & \multicolumn{2}{c}{16} & \multicolumn{2}{c@{}}{calibrated$^\star$} \\
 & & rec (p+v) & t(s) & rec (p+v) & t(s) & rec (p+v) & t(s) & rec (p+v) & t(s) & rec (p+v) & t(s) \\
\midrule
\multicolumn{12}{@{}l}{\emph{Recovery suites --- UCR (per-dataset length ranges)}} \\
ECG200 & 8k & 20 (20+0) & 0.14 & 20 (16+4) & 0.15 & 20 (16+4) & 0.15 & 20 (16+4) & 0.15 & 20 (16+4) & 0.14 \\
ItalyPowerDemand & 1,600 & 20 (20+0) & 0.02 & 20 (20+0) & 0.02 & 20 (20+0) & 0.02 & 20 (20+0) & 0.02 & 20 (20+0) & 0.02 \\
Wafer & 8k & 20 (18+2) & 0.24 & 20 (19+1) & 0.24 & 20 (18+2) & 0.24 & 20 (17+3) & 0.25 & 20 (19+1) & 0.24 \\
Earthquakes & 8k & 20 (20+0) & 0.26 & 20 (20+0) & 0.25 & 20 (20+0) & 0.30 & 20 (20+0) & 0.25 & 20 (20+0) & 0.25 \\
StarLightCurves & 8k & 20 (20+0) & 0.20 & 20 (20+0) & 0.21 & 20 (20+0) & 0.22 & 20 (20+0) & 0.23 & 20 (20+0) & 0.22 \\
ChlorineConc. & 8k & 20 (20+0) & 0.28 & 20 (20+0) & 0.35 & 20 (20+0) & 0.21 & 20 (20+0) & 0.21 & 20 (20+0) & 0.21 \\
TwoLeadECG & 8k & 20 (20+0) & 0.14 & 20 (20+0) & 0.15 & 20 (20+0) & 0.15 & 20 (20+0) & 0.15 & 20 (20+0) & 0.15 \\
ECG5000 & 8k & 20 (19+1) & 0.14 & 20 (20+0) & 0.14 & 20 (18+2) & 0.16 & 20 (20+0) & 0.21 & 20 (20+0) & 0.14 \\
FordA & 8k & 20 (20+0) & 0.21 & 20 (20+0) & 0.21 & 20 (20+0) & 0.21 & 20 (20+0) & 0.21 & 20 (20+0) & 0.20 \\
Mallat & 8k & 20 (15+5) & 0.21 & 20 (15+5) & 0.20 & 20 (15+5) & 0.21 & 20 (15+5) & 0.22 & 20 (15+5) & 0.21 \\
\multicolumn{12}{@{}l}{\emph{Recovery suite --- 64k, lengths 30--80}} \\
ChlorineConc. & 64k & 20 (20+0) & 1.93 & 20 (20+0) & 2.01 & 20 (20+0) & 2.16 & 20 (20+0) & 2.23 & 20 (20+0) & 2.22 \\
Earthquakes & 64k & 20 (20+0) & 2.43 & 20 (20+0) & 2.16 & 20 (20+0) & 2.59 & 20 (20+0) & 2.63 & 20 (20+0) & 2.65 \\
ECG5000 & 64k & 20 (14+6) & 1.98 & 20 (12+8) & 2.06 & 20 (16+4) & 2.20 & 20 (16+4) & 2.48 & 20 (12+8) & 2.10 \\
FordA & 64k & 20 (20+0) & 1.97 & 20 (15+5) & 1.98 & 20 (18+2) & 2.13 & 20 (19+1) & 2.67 & 20 (15+5) & 1.96 \\
Mallat & 64k & 20 (17+3) & 3.67 & 20 (17+3) & 3.74 & 20 (17+3) & 3.94 & 20 (17+3) & 4.31 & 20 (17+3) & 3.99 \\
StarLightCurves & 64k & 20 (18+2) & 1.83 & 20 (20+0) & 1.95 & 20 (20+0) & 2.09 & 20 (20+0) & 2.20 & 20 (20+0) & 2.20 \\
Wafer & 64k & 20 (14+6) & 1.47 & 20 (16+4) & 1.97 & 20 (17+3) & 1.71 & 20 (16+4) & 1.82 & 20 (16+4) & 2.01 \\
\multicolumn{12}{@{}l}{\emph{Full natural lengths (capped at 1M)}} \\
ItalyPowerDemand & 1,608 & 20 (20+0) & 0.02 & 20 (20+0) & 0.02 & 20 (20+0) & 0.02 & 20 (20+0) & 0.02 & 20 (20+0) & 0.02 \\
ECG200 & 9,600 & 20 (20+0) & 0.17 & 20 (16+4) & 0.17 & 20 (16+4) & 0.18 & 20 (16+4) & 0.18 & 20 (16+4) & 0.17 \\
TwoLeadECG & 93,398 & 20 (12+8) & 1.86 & 20 (18+2) & 2.00 & 20 (9+11) & 2.18 & 20 (12+8) & 2.35 & 20 (12+8) & 2.27 \\
Wafer & 152k & 20 (0+20) & 5.10 & 20 (0+20) & 5.36 & 20 (0+20) & 5.82 & 20 (0+20) & 5.52 & 20 (0+20) & 6.01 \\
Earthquakes & 164,864 & 20 (19+1) & 6.11 & 20 (19+1) & 6.62 & 20 (19+1) & 6.97 & 20 (19+1) & 7.04 & 20 (19+1) & 7.10 \\
ECG5000 & 630k & 20 (6+14) & 16.3 & 20 (8+12) & 17.1 & 20 (10+10) & 21.4 & 20 (8+12) & 19.4 & 20 (8+12) & 19.9 \\
ChlorineConc. & 637,440 & 20 (3+17) & 26.0 & 20 (3+17) & 28.2 & 20 (3+17) & 30.7 & 20 (3+17) & 33.2 & 20 (3+17) & 32.0 \\
FordA & 660k & 20 (19+1) & 27.8 & 20 (17+3) & 33.4 & 20 (15+5) & 33.5 & 20 (13+7) & 33.4 & 20 (13+7) & 33.6 \\
StarLightCurves & 1M & 20 (0+20) & 44.4 & 20 (1+19) & 40.0 & 20 (1+19) & 50.9 & 20 (1+19) & 53.2 & 20 (1+19) & 53.2 \\
Mallat & 1M & 20 (0+20) & 26.0 & 20 (0+20) & 29.4 & 20 (0+20) & 35.1 & 20 (0+20) & 35.7 & 20 (0+20) & 35.8 \\
\multicolumn{12}{@{}l}{\emph{Length scaling --- Wafer (tiled beyond 152k), lengths 30--80}} \\
Wafer & 8k & 20 (17+3) & 0.26 & 20 (17+3) & 0.24 & 20 (17+3) & 0.24 & 20 (17+3) & 0.25 & 20 (17+3) & 0.24 \\
Wafer & 16k & 20 (17+3) & 0.33 & 20 (17+3) & 0.41 & 20 (17+3) & 0.43 & 20 (17+3) & 0.43 & 20 (17+3) & 0.41 \\
Wafer & 32k & 20 (18+2) & 0.67 & 20 (19+1) & 0.88 & 20 (19+1) & 0.75 & 20 (18+2) & 0.79 & 20 (19+1) & 0.90 \\
Wafer & 64k & 20 (9+11) & 1.48 & 20 (10+10) & 2.03 & 20 (10+10) & 1.71 & 20 (10+10) & 1.83 & 20 (10+10) & 2.00 \\
Wafer & 128k & 20 (0+20) & 4.12 & 20 (0+20) & 4.40 & 20 (0+20) & 4.96 & 20 (0+20) & 5.09 & 20 (0+20) & 4.79 \\
Wafer & 256k & 20 (0+20) & 12.7 & 20 (0+20) & 13.9 & 20 (0+20) & 14.3 & 20 (0+20) & 15.0 & 20 (0+20) & 15.0 \\
Wafer & 512k & 20 (0+20) & 24.7 & 20 (0+20) & 26.6 & 20 (0+20) & 29.0 & 20 (0+20) & 30.1 & 20 (0+20) & 30.3 \\
Wafer & 1M & 20 (0+20) & 50.5 & 20 (0+20) & 54.7 & 20 (0+20) & 59.1 & 20 (0+20) & 62.6 & 20 (0+20) & 62.7 \\
Wafer & 2M & 20 (0+20) & 107 & 20 (0+20) & 116 & 20 (0+20) & 126 & 20 (0+20) & 133 & 20 (0+20) & 133 \\
Wafer & 5M & 20 (0+20) & 288 & 20 (0+20) & 313 & 20 (0+20) & 341 & 20 (0+20) & 363 & 20 (0+20) & 361 \\
\bottomrule
\end{tabular*}

\end{table*}

\begin{table*}[tp]
\caption{Maximum-performance configurations, complete per-row
results against the control. \emph{Maximum performance}: the fastest
full-recovery combination ($s=6$, $\alpha=4$, $P_{\max}=1$).
\emph{$K=4$}: the same quartered budget at the more aggressive
$s=8$ with the spectral width fixed to $K=4$.
$\times$: per-row runtime ratio against the control.}
\label{tab:abl-maxperf}
{\scriptsize\renewcommand{\arraystretch}{0.80}%
\begin{tabular*}{\textwidth}{@{\extracolsep{\fill}}llrrrrrrrr@{}}
\toprule
Dataset & $n$ & \multicolumn{2}{c}{Control} & \multicolumn{3}{c}{Maximum performance} & \multicolumn{3}{c}{Maximum performance, $K=4$} \\
 & & rec (p+v) & t(s) & rec (p+v) & t(s) & $\times$ & rec (p+v) & t(s) & $\times$ \\
\midrule
\multicolumn{10}{@{}l}{\emph{Recovery suites --- UCR (per-dataset length ranges)}} \\
ECG200 & 8k & 20 (16+4) & 0.14 & 20 (16+4) & 0.05 & 2.8 & 20 (20+0) & 0.05 & 3.0 \\
ItalyPowerDemand & 1,600 & 20 (20+0) & 0.02 & 20 (19+1) & 0.01 & 2.0 & 20 (19+1) & 0.06 & 0.4 \\
Wafer & 8k & 20 (19+1) & 0.24 & 20 (19+1) & 0.07 & 3.6 & 20 (18+2) & 0.06 & 4.1 \\
Earthquakes & 8k & 20 (20+0) & 0.25 & 20 (20+0) & 0.06 & 3.9 & 20 (20+0) & 0.06 & 4.3 \\
StarLightCurves & 8k & 20 (20+0) & 0.23 & 20 (20+0) & 0.07 & 3.4 & 20 (20+0) & 0.06 & 3.9 \\
ChlorineConc. & 8k & 20 (20+0) & 0.21 & 20 (20+0) & 0.07 & 3.1 & 20 (20+0) & 0.06 & 3.5 \\
TwoLeadECG & 8k & 20 (20+0) & 0.15 & 20 (20+0) & 0.05 & 2.9 & 20 (20+0) & 0.05 & 3.3 \\
ECG5000 & 8k & 20 (20+0) & 0.14 & 20 (17+3) & 0.05 & 2.8 & 20 (7+13) & 0.05 & 3.0 \\
FordA & 8k & 20 (20+0) & 0.22 & 20 (20+0) & 0.06 & 3.4 & 20 (20+0) & 0.06 & 3.6 \\
Mallat & 8k & 20 (15+5) & 0.21 & 20 (15+5) & 0.07 & 2.9 & 20 (15+5) & 0.06 & 3.3 \\
\multicolumn{10}{@{}l}{\emph{Recovery suite --- 64k, lengths 30--80}} \\
ChlorineConc. & 64k & 20 (20+0) & 2.26 & 20 (20+0) & 0.58 & 3.9 & 20 (20+0) & 0.48 & 4.7 \\
Earthquakes & 64k & 20 (20+0) & 2.63 & 20 (20+0) & 0.58 & 4.6 & 20 (20+0) & 0.48 & 5.4 \\
ECG5000 & 64k & 20 (12+8) & 2.08 & 20 (10+10) & 0.59 & 3.5 & 20 (4+16) & 0.50 & 4.1 \\
FordA & 64k & 20 (15+5) & 2.00 & 20 (19+1) & 0.53 & 3.8 & 20 (19+1) & 0.46 & 4.4 \\
Mallat & 64k & 20 (17+3) & 4.05 & 20 (17+3) & 0.56 & 7.3 & 20 (17+3) & 0.44 & 9.2 \\
StarLightCurves & 64k & 20 (20+0) & 2.20 & 20 (20+0) & 0.58 & 3.8 & 20 (18+2) & 0.45 & 4.9 \\
Wafer & 64k & 20 (16+4) & 1.99 & 20 (16+4) & 0.71 & 2.8 & 20 (15+5) & 0.45 & 4.4 \\
\multicolumn{10}{@{}l}{\emph{Full natural lengths (capped at 1M)}} \\
ItalyPowerDemand & 1,608 & 20 (20+0) & 0.02 & 20 (19+1) & 0.01 & 2.0 & 20 (19+1) & 0.01 & 2.1 \\
ECG200 & 9,600 & 20 (16+4) & 0.17 & 20 (16+4) & 0.06 & 2.8 & 20 (20+0) & 0.06 & 2.9 \\
TwoLeadECG & 93,398 & 20 (12+8) & 2.31 & 20 (6+14) & 0.67 & 3.5 & 20 (2+18) & 0.59 & 3.9 \\
Wafer & 152k & 20 (0+20) & 5.96 & 20 (0+20) & 1.51 & 4.0 & 20 (0+20) & 1.16 & 5.1 \\
Earthquakes & 164,864 & 20 (19+1) & 7.32 & 20 (20+0) & 1.75 & 4.2 & 20 (20+0) & 1.30 & 5.6 \\
ECG5000 & 630k & 20 (8+12) & 19.9 & 20 (2+18) & 4.88 & 4.1 & 20 (2+18) & 4.04 & 4.9 \\
ChlorineConc. & 637,440 & 20 (3+17) & 31.9 & 20 (3+17) & 7.26 & 4.4 & 20 (3+17) & 5.29 & 6.0 \\
FordA & 660k & 20 (13+7) & 33.8 & 20 (18+2) & 7.14 & 4.7 & 20 (18+2) & 5.70 & 5.9 \\
StarLightCurves & 1M & 20 (1+19) & 53.1 & 20 (1+19) & 10.1 & 5.2 & 20 (0+20) & 7.65 & 6.9 \\
Mallat & 1M & 20 (0+20) & 36.1 & 20 (0+20) & 9.72 & 3.7 & 20 (0+20) & 7.25 & 5.0 \\
\multicolumn{10}{@{}l}{\emph{Length scaling --- Wafer (tiled beyond 152k), lengths 30--80}} \\
Wafer & 8k & 20 (17+3) & 0.24 & 20 (17+3) & 0.06 & 3.7 & 20 (17+3) & 0.06 & 4.0 \\
Wafer & 16k & 20 (17+3) & 0.41 & 20 (17+3) & 0.12 & 3.3 & 20 (17+3) & 0.11 & 3.7 \\
Wafer & 32k & 20 (19+1) & 0.92 & 20 (19+1) & 0.27 & 3.4 & 20 (19+1) & 0.24 & 3.9 \\
Wafer & 64k & 20 (10+10) & 1.98 & 20 (10+10) & 0.60 & 3.3 & 20 (10+10) & 0.46 & 4.3 \\
Wafer & 128k & 20 (0+20) & 4.77 & 20 (0+20) & 1.23 & 3.9 & 20 (0+20) & 0.95 & 5.0 \\
Wafer & 256k & 20 (0+20) & 15.0 & 20 (0+20) & 2.41 & 6.2 & 20 (0+20) & 1.76 & 8.5 \\
Wafer & 512k & 20 (0+20) & 30.3 & 20 (0+20) & 4.70 & 6.4 & 20 (0+20) & 3.71 & 8.2 \\
Wafer & 1M & 20 (0+20) & 62.6 & 20 (0+20) & 9.51 & 6.6 & 20 (0+20) & 7.07 & 8.8 \\
Wafer & 2M & 20 (0+20) & 135 & 20 (0+20) & 20.2 & 6.7 & 20 (0+20) & 14.9 & 9.0 \\
Wafer & 5M & 20 (0+20) & 362 & 20 (0+20) & 52.1 & 6.9 & 20 (0+20) & 38.2 & 9.5 \\
\bottomrule
\end{tabular*}
}
\end{table*}

\FloatBarrier
\section{Breaking the Algorithm}
\label{app:factorial}

Appendix~\ref{app:params} moved every design constant and failed
to break recovery. This section removes whole mechanisms instead,
and the question it answers is attribution: if a mechanism earns
its place, disabling it must lose motifs, the losses must be
predictable from what the mechanism does, and restoring it must
recover them. One guarantee frames the whole section. Every
configuration below is sound by Proposition~\ref{prop:sound}: each
reports only exact distances to valid neighbors, so disabling a
mechanism can lose a motif but can never fabricate one. The
factorial therefore isolates recall, and every loss it records is a
motif missed, never a motif corrupted.

\panache{} discovers motifs through three cooperating mechanisms,
anchor indexing (\S\ref{sec:anchors}), the hot-class sieve
(\S\ref{sec:emit}), and local completion (\S\ref{sec:emit}). The
factorial disables each independently. Dense indexing ($s=1$)
replaces anchoring and maintains online state at every length. A
sieve width of zero turns the sweep off. Completion, normally
governed by a run-time rule, is forced fully on or fully off so
that every configuration is unambiguous, and the control keeps the
published rule. The result is a full $2\times2\times2$ factorial
over the three mechanisms, run on all 37 benchmark rows with the
setup of Appendix~\ref{app:params}.

A run can miss a reference motif in exactly two ways, and
separating them is what lets us attribute each loss to a
mechanism. A \emph{coverage loss} is a reference motif whose length
the run never output at all. Anchoring is the only source of
these: it maintains online state at every second length and leaves
the lengths in between to completion, so an anchored run with
completion off emits nothing at those in-between lengths, and every
reference motif that lives there is lost before any distance is
computed. An \emph{evaluation loss} is different. The motif's
length was indexed and both its windows entered the directory, but
their pair was never compared, because it landed in a hash bucket
holding more entries than the per-query budget $R$ could scan. The
budget is the only source of these. This pairing is the claim the
section tests: completion exists to close coverage losses, the
sieve exists to close evaluation losses, so if the attribution is
right, removing completion should produce only coverage losses and
removing the sieve should expose only evaluation losses.
Table~\ref{tab:factorial} reports all eight configurations with
every loss split by the length at which it occurred, and
Tables~\ref{tab:fact-anchored} and~\ref{tab:fact-dense} give the
per-row detail.

\begin{table}[t]
\caption{The mechanism factorial. Each configuration disables a subset of
\{anchor indexing, hot-class sieve, local completion\}; $^\star$
marks the published configuration. Losses are counted against the
740 reference motifs and split by whether the lost motif's length
was indexed (anchor) or not (non-anchor). Dense indexing has no
non-anchor lengths. Time is total runtime relative to the control.}
\label{tab:factorial}
\scriptsize
\renewcommand{\arraystretch}{0.85}
\begin{tabular*}{\columnwidth}{@{\extracolsep{\fill}}lllrrrrrr@{}}
\toprule
Indexing & Sieve & Compl. & rec & pos & val & \multicolumn{2}{c}{lost at lengths} & time \\
 & & & of 740 & of 340 & & non-anc. & anchor & $\times$ \\
\midrule
dense & off & off & 714 & 236 & 349 & -- & 26 & 0.90 \\
dense & on & off & 714 & 268 & 313 & -- & 26 & 0.91 \\
dense & off & on & 736 & 292 & 288 & -- & 4 & 2.32 \\
dense & on & on & 740 & 313 & 272 & -- & 0 & 1.92 \\
anchored & off & off & 609 & 186 & 301 & 127 & 4 & 0.48 \\
anchored & on & off & 611 & 209 & 275 & 127 & 2 & 0.49 \\
anchored & off & on & 736 & 292 & 289 & 2 & 2 & 1.22 \\
anchored$^\star$ & on & on & 740 & 310 & 275 & 0 & 0 & 1.00 \\
\bottomrule
\end{tabular*}

\end{table}

The factorial has one headline. Both configurations that keep the full emit
stage recover all 740 motifs, on either indexing scheme, and every
configuration that loses a motif is missing an emit mechanism. The dense
full configuration matches 313 of 340 exact positions against the control's
310, so anchoring at stride two costs no motif and just three
exact positions of 340, inside the tie noise of this study
(Appendix~\ref{app:params}), while running at half the dense
runtime. The losses in the
remaining six configurations total 300 motif instances, and each one is
attributed below to the mechanism whose absence caused it, with the
evidence in Tables~\ref{tab:modea} and~\ref{tab:modeb}.

\paragraph{Dense pass, no emit stage (no anchoring, no sieve, no
completion).}
This configuration turns off all three mechanisms. Because it
indexes every length, it has no in-between lengths, so no coverage
loss is even possible, and any motif it misses must be an
evaluation loss. It recovers 714 of 740, and all 26 misses lie at
indexed lengths, as expected. The emitted output shows directly
that these are evaluation losses, because it contains pairs worse
than the ones that were missed: a motif can only be an evaluation
loss if a better pair than anything reported was available but
never compared. On TwoLeadECG 93k the eight lost reference motifs
at length $m=21$ have z-normalized distances between 0.0448 and
0.0483, yet the smallest distance this configuration reported at
any length is 0.0526, so eight pairs closer than everything in its
output existed and went uncompared. ECG5000 8k shows the same
signature: its two lost motifs at $m=23$ have distance 0.192,
inside the 0.182--0.206 range of distances the configuration did
report, so their quality was well within what it emitted
elsewhere, yet their pairs were never evaluated. What stopped the
comparisons is the verification budget. Table~\ref{tab:modeb}
counts, per losing row, the directory buckets whose population
exceeded the budget $R=128$: TwoLeadECG 93k has a single bucket of
9{,}994 windows, and across the six rows between 56\% and 93\% of
all directory entries sit in such over-budget buckets. These are
not hashing failures. Near-duplicate windows collide with
near-certainty (Prop.~\ref{prop:recall}), so the two windows of
each lost motif did reach the same bucket; the budget simply could
not scan far enough into it to reach them. This is precisely the
failure the hot-class sieve is built to repair.

\paragraph{Anchored pass, no emit stage (anchoring on, no sieve,
no completion).}
Anchoring is on and both emit mechanisms are off, so this
configuration should exhibit coverage losses, and it does. It
recovers 609 of 740, and the 131 misses split into 127 at
non-anchor lengths and 4 at anchor lengths. The 127 are coverage
losses, and they are not explained after the fact but predicted
before the run. With completion off, the configuration emits only
the anchor (even) lengths, so every reference motif at a
non-anchor (odd) length is lost by construction. Counting the
reference motifs at odd lengths from the ground-truth lists alone,
before running anything, gives 127, spread over the 17 of the 37
benchmark rows that carry any reference motif at a non-anchor
length (the remaining 20 place their entire top-20 at anchor
lengths, so anchoring alone loses nothing on them), and
Table~\ref{tab:modea} shows this count matching the observed losses
row by row and length by length. \textbf{The entire non-anchor loss
ledger is predicted by the reference lists alone.}
The remaining 4 losses fall at anchor lengths, so they are
evaluation losses, not coverage losses. All four are tie entries at
$m=34$ on StarLightCurves 8k, the row whose largest hash class
alone holds 46\% of its windows (Table~\ref{tab:modeb}), and the
sieve and completion recover them below.

\paragraph{Anchored pass with the sieve (anchoring and sieve on,
no completion).}
Turning the sieve on should touch only evaluation losses and never
coverage losses, because the sieve re-scans crowded buckets at
lengths that are already indexed and adds no new length to the
output. That is what happens. The non-anchor column of
Table~\ref{tab:modea} does not move: the 127 coverage losses
remain exactly as the census predicts. The change is at the anchor
lengths, where recovery rises from 609 to 611 and exact positions
from 186 to 209 of 340 as the sieve reaches pairs the budgeted
scan had skipped. The dense side isolates this effect further,
with no coverage losses to confound it: there the same sieve lifts
exact positions from 236 to 268 (Table~\ref{tab:factorial}). The
sieve recovers the four StarLightCurves 8k ties, and the top-20 of
StarLightCurves 64k shifts by two tie entries at $m=32$, the tie
behavior explained below.

\begin{table*}[tp]
\caption{Mechanism ablation under anchor indexing ($s=2$), complete
per-row results. bare = budgeted scans only, no emit stage;
+sieve and +completion enable one emit mechanism each; full is the
control. Columns as in Table~\ref{tab:abl-stride}.}
\label{tab:fact-anchored}
\scriptsize
\renewcommand{\arraystretch}{0.80}
\begin{tabular*}{\textwidth}{@{\extracolsep{\fill}}llrrrrrrrr@{}}
\toprule
Dataset & $n$ & \multicolumn{2}{c}{bare} & \multicolumn{2}{c}{+sieve} & \multicolumn{2}{c}{+completion} & \multicolumn{2}{c@{}}{full$^\star$} \\
 & & rec (p+v) & t(s) & rec (p+v) & t(s) & rec (p+v) & t(s) & rec (p+v) & t(s) \\
\midrule
\multicolumn{10}{@{}l}{\emph{Recovery suites --- UCR (per-dataset length ranges)}} \\
ECG200 & 8k & 12 (11+1) & 0.06 & 12 (11+1) & 0.06 & 20 (16+4) & 0.15 & 20 (16+4) & 0.14 \\
ItalyPowerDemand & 1,600 & 13 (12+1) & 0.01 & 13 (13+0) & 0.01 & 20 (18+2) & 0.02 & 20 (20+0) & 0.02 \\
Wafer & 8k & 20 (19+1) & 0.07 & 20 (19+1) & 0.07 & 20 (19+1) & 0.36 & 20 (19+1) & 0.25 \\
Earthquakes & 8k & 20 (19+1) & 0.07 & 20 (19+1) & 0.08 & 20 (20+0) & 0.26 & 20 (20+0) & 0.25 \\
StarLightCurves & 8k & 10 (8+2) & 0.09 & 14 (12+2) & 0.09 & 20 (20+0) & 0.27 & 20 (20+0) & 0.23 \\
ChlorineConc. & 8k & 10 (10+0) & 0.07 & 10 (10+0) & 0.07 & 20 (20+0) & 0.21 & 20 (20+0) & 0.22 \\
TwoLeadECG & 8k & 12 (10+2) & 0.06 & 12 (10+2) & 0.07 & 20 (20+0) & 0.15 & 20 (20+0) & 0.15 \\
ECG5000 & 8k & 13 (8+5) & 0.06 & 13 (11+2) & 0.06 & 20 (20+0) & 0.15 & 20 (20+0) & 0.14 \\
FordA & 8k & 14 (14+0) & 0.07 & 14 (14+0) & 0.07 & 20 (20+0) & 0.21 & 20 (20+0) & 0.20 \\
Mallat & 8k & 10 (5+5) & 0.07 & 10 (7+3) & 0.07 & 20 (15+5) & 0.30 & 20 (15+5) & 0.21 \\
\multicolumn{10}{@{}l}{\emph{Recovery suite --- 64k, lengths 30--80}} \\
ChlorineConc. & 64k & 10 (10+0) & 1.03 & 10 (10+0) & 1.04 & 20 (20+0) & 2.23 & 20 (20+0) & 2.27 \\
Earthquakes & 64k & 20 (20+0) & 0.96 & 20 (20+0) & 1.00 & 20 (20+0) & 2.61 & 20 (20+0) & 2.68 \\
ECG5000 & 64k & 12 (2+10) & 0.80 & 12 (4+8) & 0.86 & 20 (12+8) & 2.09 & 20 (12+8) & 2.08 \\
FordA & 64k & 14 (6+8) & 0.70 & 14 (6+8) & 0.77 & 20 (17+3) & 2.01 & 20 (15+5) & 1.97 \\
Mallat & 64k & 20 (16+4) & 1.09 & 20 (17+3) & 1.02 & 20 (17+3) & 3.01 & 20 (17+3) & 4.03 \\
StarLightCurves & 64k & 12 (0+12) & 0.97 & 10 (10+0) & 0.95 & 16 (2+14) & 2.58 & 20 (20+0) & 2.20 \\
Wafer & 64k & 20 (16+4) & 0.83 & 20 (16+4) & 0.80 & 20 (16+4) & 1.97 & 20 (16+4) & 1.99 \\
\multicolumn{10}{@{}l}{\emph{Full natural lengths (capped at 1M)}} \\
ItalyPowerDemand & 1,608 & 13 (12+1) & 0.01 & 13 (13+0) & 0.01 & 20 (18+2) & 0.02 & 20 (20+0) & 0.02 \\
ECG200 & 9,600 & 12 (11+1) & 0.07 & 12 (11+1) & 0.08 & 20 (16+4) & 0.17 & 20 (16+4) & 0.17 \\
TwoLeadECG & 93,398 & 12 (2+10) & 1.14 & 12 (4+8) & 1.24 & 20 (10+10) & 2.28 & 20 (12+8) & 2.38 \\
Wafer & 152k & 20 (0+20) & 3.00 & 20 (0+20) & 2.95 & 20 (0+20) & 6.93 & 20 (0+20) & 5.98 \\
Earthquakes & 164,864 & 20 (19+1) & 3.62 & 20 (19+1) & 3.65 & 20 (19+1) & 7.03 & 20 (19+1) & 7.06 \\
ECG5000 & 630k & 18 (4+14) & 10.7 & 18 (4+14) & 11.3 & 20 (8+12) & 19.6 & 20 (8+12) & 19.6 \\
ChlorineConc. & 637,440 & 20 (5+15) & 17.2 & 20 (5+15) & 17.4 & 20 (3+17) & 31.6 & 20 (3+17) & 31.8 \\
FordA & 660k & 12 (6+6) & 17.9 & 12 (6+6) & 18.4 & 20 (17+3) & 38.3 & 20 (13+7) & 33.6 \\
StarLightCurves & 1M & 20 (0+20) & 20.8 & 20 (2+18) & 21.7 & 20 (1+19) & 68.2 & 20 (1+19) & 53.1 \\
Mallat & 1M & 20 (0+20) & 23.5 & 20 (0+20) & 24.6 & 20 (0+20) & 83.2 & 20 (0+20) & 35.9 \\
\multicolumn{10}{@{}l}{\emph{Length scaling --- Wafer (tiled beyond 152k), lengths 30--80}} \\
Wafer & 8k & 20 (17+3) & 0.08 & 20 (17+3) & 0.07 & 20 (17+3) & 0.25 & 20 (17+3) & 0.24 \\
Wafer & 16k & 20 (17+3) & 0.14 & 20 (17+3) & 0.15 & 20 (17+3) & 0.41 & 20 (17+3) & 0.41 \\
Wafer & 32k & 20 (19+1) & 0.33 & 20 (19+1) & 0.34 & 20 (19+1) & 0.89 & 20 (19+1) & 0.89 \\
Wafer & 64k & 20 (10+10) & 0.81 & 20 (10+10) & 0.79 & 20 (10+10) & 1.97 & 20 (10+10) & 2.03 \\
Wafer & 128k & 20 (0+20) & 2.19 & 20 (0+20) & 2.24 & 20 (0+20) & 4.75 & 20 (0+20) & 4.75 \\
Wafer & 256k & 20 (0+20) & 5.78 & 20 (0+20) & 5.92 & 20 (0+20) & 17.1 & 20 (0+20) & 15.0 \\
Wafer & 512k & 20 (0+20) & 13.1 & 20 (0+20) & 13.5 & 20 (0+20) & 36.0 & 20 (0+20) & 30.4 \\
Wafer & 1M & 20 (0+20) & 30.6 & 20 (0+20) & 29.2 & 20 (0+20) & 73.2 & 20 (0+20) & 62.7 \\
Wafer & 2M & 20 (0+20) & 62.3 & 20 (0+20) & 63.7 & 20 (0+20) & 153 & 20 (0+20) & 133 \\
Wafer & 5M & 20 (0+20) & 174 & 20 (0+20) & 176 & 20 (0+20) & 437 & 20 (0+20) & 363 \\
\bottomrule
\end{tabular*}

\end{table*}

\begin{table*}[tp]
\caption{Mechanism ablation under dense indexing ($s=1$), complete
per-row results. Cells as in Table~\ref{tab:fact-anchored}.}
\label{tab:fact-dense}
\scriptsize
\renewcommand{\arraystretch}{0.80}
\begin{tabular*}{\textwidth}{@{\extracolsep{\fill}}llrrrrrrrr@{}}
\toprule
Dataset & $n$ & \multicolumn{2}{c}{bare} & \multicolumn{2}{c}{+sieve} & \multicolumn{2}{c}{+completion} & \multicolumn{2}{c@{}}{full} \\
 & & rec (p+v) & t(s) & rec (p+v) & t(s) & rec (p+v) & t(s) & rec (p+v) & t(s) \\
\midrule
\multicolumn{10}{@{}l}{\emph{Recovery suites --- UCR (per-dataset length ranges)}} \\
ECG200 & 8k & 20 (14+6) & 0.08 & 20 (14+6) & 0.08 & 20 (16+4) & 0.24 & 20 (16+4) & 0.25 \\
ItalyPowerDemand & 1,600 & 20 (14+6) & 0.01 & 20 (17+3) & 0.01 & 20 (20+0) & 0.03 & 20 (20+0) & 0.03 \\
Wafer & 8k & 20 (19+1) & 0.10 & 20 (19+1) & 0.10 & 20 (19+1) & 0.36 & 20 (19+1) & 0.36 \\
Earthquakes & 8k & 20 (19+1) & 0.11 & 20 (19+1) & 0.13 & 20 (20+0) & 0.47 & 20 (20+0) & 0.47 \\
StarLightCurves & 8k & 14 (8+6) & 0.13 & 20 (14+6) & 0.14 & 20 (20+0) & 0.41 & 20 (20+0) & 0.41 \\
ChlorineConc. & 8k & 20 (20+0) & 0.10 & 20 (20+0) & 0.10 & 20 (20+0) & 0.30 & 20 (20+0) & 0.29 \\
TwoLeadECG & 8k & 14 (12+2) & 0.09 & 18 (14+4) & 0.10 & 20 (20+0) & 0.26 & 20 (20+0) & 0.26 \\
ECG5000 & 8k & 18 (8+10) & 0.09 & 20 (15+5) & 0.10 & 20 (20+0) & 0.30 & 20 (20+0) & 0.30 \\
FordA & 8k & 20 (20+0) & 0.09 & 20 (20+0) & 0.10 & 20 (20+0) & 0.37 & 20 (20+0) & 0.37 \\
Mallat & 8k & 20 (14+6) & 0.10 & 20 (13+7) & 0.10 & 20 (13+7) & 0.37 & 20 (13+7) & 0.37 \\
\multicolumn{10}{@{}l}{\emph{Recovery suite --- 64k, lengths 30--80}} \\
ChlorineConc. & 64k & 20 (20+0) & 1.88 & 20 (20+0) & 1.92 & 20 (20+0) & 4.36 & 20 (20+0) & 3.54 \\
Earthquakes & 64k & 20 (20+0) & 1.70 & 20 (20+0) & 1.70 & 20 (20+0) & 5.11 & 20 (20+0) & 5.11 \\
ECG5000 & 64k & 20 (6+14) & 1.48 & 20 (8+12) & 1.50 & 20 (10+10) & 4.02 & 20 (12+8) & 4.07 \\
FordA & 64k & 20 (10+10) & 1.26 & 20 (10+10) & 1.31 & 20 (19+1) & 3.83 & 20 (20+0) & 3.85 \\
Mallat & 64k & 20 (16+4) & 1.72 & 20 (17+3) & 1.91 & 20 (17+3) & 5.86 & 20 (17+3) & 8.02 \\
StarLightCurves & 64k & 18 (0+18) & 1.65 & 12 (12+0) & 1.76 & 16 (2+14) & 4.23 & 20 (20+0) & 4.31 \\
Wafer & 64k & 20 (16+4) & 1.44 & 20 (16+4) & 1.42 & 20 (16+4) & 3.89 & 20 (16+4) & 3.89 \\
\multicolumn{10}{@{}l}{\emph{Full natural lengths (capped at 1M)}} \\
ItalyPowerDemand & 1,608 & 20 (14+6) & 0.01 & 20 (17+3) & 0.01 & 20 (20+0) & 0.03 & 20 (20+0) & 0.03 \\
ECG200 & 9,600 & 20 (14+6) & 0.10 & 20 (14+6) & 0.12 & 20 (16+4) & 0.29 & 20 (16+4) & 0.30 \\
TwoLeadECG & 93,398 & 12 (2+10) & 2.02 & 12 (4+8) & 2.12 & 20 (10+10) & 4.30 & 20 (12+8) & 4.42 \\
Wafer & 152k & 20 (0+20) & 5.47 & 20 (0+20) & 5.34 & 20 (0+20) & 11.5 & 20 (0+20) & 11.5 \\
Earthquakes & 164,864 & 20 (19+1) & 6.72 & 20 (19+1) & 6.56 & 20 (19+1) & 13.4 & 20 (19+1) & 13.5 \\
ECG5000 & 630k & 18 (4+14) & 19.3 & 18 (4+14) & 19.4 & 20 (8+12) & 36.5 & 20 (8+12) & 36.8 \\
ChlorineConc. & 637,440 & 20 (3+17) & 31.8 & 20 (3+17) & 32.4 & 20 (3+17) & 51.7 & 20 (3+17) & 52.7 \\
FordA & 660k & 20 (10+10) & 33.7 & 14 (8+6) & 34.0 & 20 (17+3) & 64.6 & 20 (13+7) & 65.1 \\
StarLightCurves & 1M & 20 (0+20) & 40.9 & 20 (1+19) & 39.8 & 20 (0+20) & 86.5 & 20 (1+19) & 88.3 \\
Mallat & 1M & 20 (0+20) & 43.4 & 20 (0+20) & 44.6 & 20 (0+20) & 166 & 20 (0+20) & 68.6 \\
\multicolumn{10}{@{}l}{\emph{Length scaling --- Wafer (tiled beyond 152k), lengths 30--80}} \\
Wafer & 8k & 20 (17+3) & 0.10 & 20 (17+3) & 0.10 & 20 (17+3) & 0.36 & 20 (17+3) & 0.36 \\
Wafer & 16k & 20 (17+3) & 0.24 & 20 (17+3) & 0.25 & 20 (17+3) & 0.79 & 20 (17+3) & 0.60 \\
Wafer & 32k & 20 (19+1) & 0.56 & 20 (19+1) & 0.59 & 20 (19+1) & 1.73 & 20 (19+1) & 1.73 \\
Wafer & 64k & 20 (10+10) & 1.43 & 20 (10+10) & 1.44 & 20 (10+10) & 3.89 & 20 (10+10) & 3.85 \\
Wafer & 128k & 20 (0+20) & 4.08 & 20 (0+20) & 4.10 & 20 (0+20) & 9.21 & 20 (0+20) & 7.56 \\
Wafer & 256k & 20 (0+20) & 10.8 & 20 (0+20) & 11.1 & 20 (0+20) & 33.9 & 20 (0+20) & 29.4 \\
Wafer & 512k & 20 (0+20) & 24.5 & 20 (0+20) & 24.9 & 20 (0+20) & 70.5 & 20 (0+20) & 59.2 \\
Wafer & 1M & 20 (0+20) & 53.4 & 20 (0+20) & 54.4 & 20 (0+20) & 144 & 20 (0+20) & 123 \\
Wafer & 2M & 20 (0+20) & 118 & 20 (0+20) & 121 & 20 (0+20) & 304 & 20 (0+20) & 264 \\
Wafer & 5M & 20 (0+20) & 330 & 20 (0+20) & 335 & 20 (0+20) & 866 & 20 (0+20) & 714 \\
\bottomrule
\end{tabular*}

\end{table*}

\begin{table}[H]
\caption{Census of reference motifs at non-anchor lengths against
the observed losses of the two anchored configurations without completion.
refs counts the reference motifs at odd lengths per row; the two
loss columns count the reference motifs lost at those lengths with
the sieve off and on.}
\label{tab:modea}
\scriptsize
\renewcommand{\arraystretch}{0.85}
\begin{tabular*}{\columnwidth}{@{\extracolsep{\fill}}llrlrr@{}}
\toprule
Dataset & $n$ & refs & at lengths $m$ & \multicolumn{2}{c@{}}{lost, completion off} \\
 & & & & sieve off & sieve on \\
\midrule
ChlorineConc. & 64k & 10 & 31, 33 & 10 & 10 \\
ChlorineConc. & 8k & 10 & 31, 33 & 10 & 10 \\
ECG200 & 8k & 8 & 21, 23 & 8 & 8 \\
ECG200 & 9,600 & 8 & 21, 23 & 8 & 8 \\
ECG5000 & 630k & 2 & 21 & 2 & 2 \\
ECG5000 & 64k & 8 & 31 & 8 & 8 \\
ECG5000 & 8k & 7 & 21, 23 & 7 & 7 \\
FordA & 64k & 6 & 31 & 6 & 6 \\
FordA & 660k & 8 & 31, 33 & 8 & 8 \\
FordA & 8k & 6 & 31, 33 & 6 & 6 \\
ItalyPowerDemand & 1,600 & 7 & 9 & 7 & 7 \\
ItalyPowerDemand & 1,608 & 7 & 9 & 7 & 7 \\
Mallat & 8k & 10 & 31 & 10 & 10 \\
StarLightCurves & 64k & 8 & 31, 33, 35 & 8 & 8 \\
StarLightCurves & 8k & 6 & 31, 33, 35 & 6 & 6 \\
TwoLeadECG & 8k & 8 & 21, 23, 25 & 8 & 8 \\
TwoLeadECG & 93,398 & 8 & 21 & 8 & 8 \\
\midrule
Total & & 127 & & 127 & 127 \\
\bottomrule
\end{tabular*}

\end{table}

\begin{table}[H]
\caption{Hot-class census of the six rows that lose motifs at
indexed lengths, measured by the control run's own counters. hot
buckets counts directory buckets whose population exceeded the
verification budget $R=128$; dir.\ share is the fraction of all
directory entries inside such buckets; sweep updates counts the
profile cells the hot-class sieve improved on that row.}
\label{tab:modeb}
\scriptsize
\renewcommand{\arraystretch}{0.85}
\begin{tabular*}{\columnwidth}{@{\extracolsep{\fill}}llrrrr@{}}
\toprule
Dataset & $n$ & hot & largest class & dir.\ share & sweep \\
 & & buckets & (\% windows) & (\%) & updates \\
\midrule
ECG5000 & 8k & 258 & 1,490 (19) & 56 & 9,479 \\
StarLightCurves & 8k & 290 & 3,695 (46) & 82 & 82,329 \\
TwoLeadECG & 93,398 & 1517 & 9,994 (11) & 89 & 358,587 \\
ECG5000 & 630k & 9878 & 67,806 (11) & 81 & 1,170,444 \\
FordA & 660k & 21502 & 31,489 (5) & 68 & 1,911,441 \\
StarLightCurves & 64k & 1043 & 21,531 (34) & 93 & 933,207 \\
\bottomrule
\end{tabular*}

\end{table}

\paragraph{Anchored pass with completion (anchoring and completion
on, no sieve).}
Completion should close the coverage gap, and it nearly does.
Recovery jumps from 609 to 736 of 740: completion evaluates the
neighborhoods of the motifs the stream pass found, at the skipped
lengths around them, which restores 125 of the 127 non-anchor
coverage losses and the four StarLightCurves 8k anchor ties. Four
losses remain, all on StarLightCurves 64k, two at $m=34$ and two
at $m=35$. This is the one row where the two emit mechanisms
depend on each other. Completion can only search around a motif
the stream pass already found, and with the sieve off the budgeted
pass never reaches into this row's enormous tie class, where 93\%
of the directory sits in over-budget buckets
(Table~\ref{tab:modeb}), so completion is handed no seed there to
expand from. With the sieve on as well, the full configuration
finds those seeds and recovers the row completely.

\paragraph{The dense half.}
Dense indexing has no uncovered lengths, so every loss in the
dense configurations is an evaluation loss. This lets each
mechanism be judged on evaluation losses alone. The bare
configuration loses 26, the sieve alone still loses 26, completion
alone loses 4, and the two together lose none. The steady count of
26 under the sieve looks like no effect, but it hides a change of
composition. An incomplete top-20 profile is full of borderline
ties. Any mechanism that lowers one more cell changes which of
them survive the cut, so the twenty-six misses become a different
twenty-six even though the total holds. The clear signal is the
exact-position column. It rises with every mechanism added, from
236 to 268 to 292 to 313 of 340, while the value-tie count falls
from 349 to 272 (Table~\ref{tab:factorial}). Each mechanism trades
a coincidental value match for an exact reproduction, an
improvement only the position column can show.

\paragraph{StarLightCurves anatomy.}
One row sits behind every residual in this section, and it also
explains why recovery and exact positions can disagree.
StarLightCurves 64k is dominated by near-identical repeats. Its
largest hash class alone covers 21{,}531 of its 64{,}000 windows,
and 93\% of its directory entries sit in over-budget buckets
(Table~\ref{tab:modeb}). The eight entries the partial
configurations lose here are all the same near-identical repeat,
recurring near positions $(4804,33469)$ at the non-anchor lengths
$m=31$, $33$, and $35$ (Table~\ref{tab:modea}), at distances of
0.0023 to 0.0026. Each sits inside a tie set far larger than
twenty. When so many pairs
tie at the same tiny distance, the top-20 list is not unique, and
two exact algorithms would themselves disagree on which members to
report. This makes the row a test of what the two recovery columns
mean. The bare dense configuration scores 18 of 20 here
yet matches zero exact positions: the score comes entirely from
tie values landing in range. Both full configurations score
20 of 20, with all 20 positions exact. Recovery alone cannot
separate a lucky tie-value match from an exact reproduction. The
exact-position column can, which is why every table in this
appendix reports both.

The factorial closes with a complete attribution. Every loss in
the six partial configurations belongs to a mechanism. The 127
non-anchor losses equal the reference census exactly and vanish
the moment completion is enabled. The evaluation losses at indexed
lengths lie on the six rows of Table~\ref{tab:modeb} and vanish
once the sieve and completion work together. The two emit
mechanisms close disjoint failures: completion closes the coverage
gap that anchoring opens (\S\ref{sec:anchors}), and the sieve
closes the evaluation gap that the budget opens
(\S\ref{sec:emit}). With both enabled, the full configuration
loses nothing under either indexing scheme. Hence, no mechanism is
redundant.

\FloatBarrier
\section{Recovery of Determinate Motifs}
\label{app:unique}

Recovery (\S\ref{sec:eval}) credits a reference motif to \panache{}
when it is reported by position or by value. This section examines a
stronger criterion, \emph{identity}: whether \panache{} reports the
same window as the exact baseline rather than a distinct window at an
equal distance. The criterion is well defined only for a motif with a
\emph{determinate} nearest neighbor, that is, a single closest window.
When a window has one or more exact duplicates, its nearest-neighbor
distance is zero and several identical windows are equidistant; the
nearest neighbor is then not unique, and two exact algorithms may
themselves report different members of the tie (\S\ref{app:recovery}).

We therefore measure, for each configuration's exact reference top-20
(computed by \stumpy{} or \scamp{}-GPU, \S\ref{sec:setup}), which of
its twenty cells have a determinate nearest neighbor. We read each
cell's exact nearest-neighbor distance: a distance above the zero-snap
threshold $\varepsilon_0=10^{-6}$ marks a single closest window and one
correct motif to recover, whereas a distance of zero does not. We then
compare \panache{}'s top-20 with the exact top-20 on the determinate
cells, window for window. Of the $740$ reference motifs, $320$ have a
determinate nearest neighbor; the other $420$ are distance-zero
duplicates, almost all in the tiled Wafer scaling rows.

\panache{} recovers every one of the $320$ determinate motifs, and its
top-20 agrees with the exact top-20 on all of them: $272$ at the
identical start position and $48$ as equal-distance value matches
(Table~\ref{tab:unique}). On the shorter series the two lists are
identical, window for window, as on FordA, ChlorineConcentration,
ECG5000, and TwoLeadECG at 8k, on ItalyPowerDemand, and on
StarLightCurves. Counting the distance-zero duplicates, which are all
recovered as well, benchmark recovery is $740/740$.

\begin{table}[t]
\caption{Same-motif recovery, default configuration. A reference motif
has a \emph{determinate} nearest neighbor when its exact
nearest-neighbor distance is nonzero, so its closest window is a
single, well-defined window; a distance-zero cell has one or more
duplicate windows and no single closest neighbor. det.\ (pos$+$val):
motifs with a determinate neighbor, of 20, split into position and
value matches (all recovered). dist-0: distance-zero cells, of 20.
rec.: motifs recovered by position or value, of 20. ``--'' marks
configurations whose reference top-20 is entirely distance-zero
duplicates.}
\label{tab:unique}
\scriptsize
\setlength{\tabcolsep}{3.6pt}
\renewcommand{\arraystretch}{0.88}
\begin{tabular*}{\columnwidth}{@{\extracolsep{\fill}}lrrrr@{}}
\toprule
Dataset & $n$ & det.\ (pos$+$val) & dist-0 & rec. \\
\midrule
\multicolumn{5}{@{}l}{\emph{Recovery suites --- UCR (per-dataset length ranges)}} \\
ECG200 & 8k & 20 (16+4) & 0 & 20 \\
ItalyPowerDemand & 1,600 & 20 (20+0) & 0 & 20 \\
Wafer & 8k & -- & 20 & 20 \\
Earthquakes & 8k & -- & 20 & 20 \\
StarLightCurves & 8k & 20 (20+0) & 0 & 20 \\
ChlorineConc. & 8k & 20 (20+0) & 0 & 20 \\
TwoLeadECG & 8k & 20 (20+0) & 0 & 20 \\
ECG5000 & 8k & 20 (20+0) & 0 & 20 \\
FordA & 8k & 20 (20+0) & 0 & 20 \\
Mallat & 8k & -- & 20 & 20 \\
\multicolumn{5}{@{}l}{\emph{Recovery suite --- 64k, lengths 30--80}} \\
ChlorineConc. & 64k & 20 (20+0) & 0 & 20 \\
Earthquakes & 64k & -- & 20 & 20 \\
ECG5000 & 64k & 20 (12+8) & 0 & 20 \\
FordA & 64k & 20 (15+5) & 0 & 20 \\
Mallat & 64k & -- & 20 & 20 \\
StarLightCurves & 64k & 20 (20+0) & 0 & 20 \\
Wafer & 64k & -- & 20 & 20 \\
\multicolumn{5}{@{}l}{\emph{Full natural lengths (capped at 1M)}} \\
ItalyPowerDemand & 1,608 & 20 (20+0) & 0 & 20 \\
ECG200 & 9,600 & 20 (16+4) & 0 & 20 \\
TwoLeadECG & 93,398 & 20 (12+8) & 0 & 20 \\
Wafer & 152k & -- & 20 & 20 \\
Earthquakes & 164,864 & -- & 20 & 20 \\
ECG5000 & 630k & 20 (8+12) & 0 & 20 \\
ChlorineConc. & 637,440 & -- & 20 & 20 \\
FordA & 660k & 20 (13+7) & 0 & 20 \\
StarLightCurves & 1M & -- & 20 & 20 \\
Mallat & 1M & -- & 20 & 20 \\
\multicolumn{5}{@{}l}{\emph{Length scaling --- Wafer (tiled beyond 152k), lengths 30--80}} \\
Wafer (10 rows, 8k--5M) & & -- & 200 & 200 \\
\midrule
\textbf{Total} & & \textbf{320 (272+48)} & \textbf{420} & \textbf{740} \\
\bottomrule
\end{tabular*}

\end{table}

\section{Recovery of the Rarest Patterns}
\label{app:rare}

Motif discovery must surface rare patterns, those that occur only once
in a series, as well as common ones that repeat many times. This
section is a dedicated study of rare-pattern discovery: we rank every
reference motif in the benchmark by how rarely its shape recurs,
isolate the rarest, and measure whether \panache{} discovers them and
reports the same patterns as the exact baseline.

We measure rarity directly from the raw series, independently of
\panache{}. Take a determinate reference motif: a window starting at
$q$ at length $m$ whose exact nearest-neighbor distance is $d_1$
(\S\ref{app:unique}). We z-normalize every length-$m$ window of the
series and compute the z-normalized Euclidean distance from each to the
window at $q$, its distance profile. Discarding the trivial matches
inside the exclusion zone, we count the windows whose distance is at
most $2d_1$. A window that close is another near-occurrence of the same
shape, so this count is the number of times the pattern occurs, the
size of its near-duplicate class. A recurring pattern, a repeated
heartbeat say, has a large class; a one-off pattern has an empty one
but for its single partner. We call a motif \emph{isolated} when the
count is exactly one: the pattern occurs once, as a pair that matches
each other and no third window. To order the isolated motifs by how
distinct they are, we use the gap $d_2/d_1$, the ratio of the
second-smallest to the smallest distance in the profile; $d_1$ is the
matched pair and $d_2$ the nearest unrelated window, so a large ratio
means the match stands alone and the pattern is sharply distinct. The
twenty rarest patterns are the isolated motifs with the largest gap.

Of the $320$ determinate reference motifs (\S\ref{app:unique}), $123$
are isolated, occurring across six datasets and lengths from $8$ to
$35$. \panache{} recovers \emph{all $123$}: every isolated motif the
exact baseline reports is found, with no miss on any dataset, at any
length, or at any stream size (Table~\ref{tab:rare-perds}). Rarity does
not cost recovery. Whatever a series' mix of common and one-off
patterns, and however few or many of its motifs are isolated, each
dataset's recovered count equals its isolated count exactly. Of the
$123$, $110$ are matched at the identical start position and the other
$13$ as equal-distance motifs, the latter only on the two largest
streams (FordA 660k, ECG5000 630k), where the budgeted scan returns an
equally close neighbor. The twenty rarest patterns of all, those with
the largest gap to their second-nearest window, are recovered $20/20$
at the identical position (Table~\ref{tab:rare-top}): \panache{}
returns the same one-off motif the exact baseline reports, on series
from $1{,}600$ to $660{,}000$ samples. Across the study, no rare
pattern the exact baseline reports goes undiscovered.

\begin{table}[H]
\small
\setlength{\tabcolsep}{5pt}
\renewcommand{\arraystretch}{0.95}
\caption{Isolated-pattern recovery by dataset. det.: determinate
motifs (nonzero-distance nearest neighbor, \S\ref{app:unique}).
isolated: those whose near-duplicate class has size one, so the pattern
occurs exactly once in the series. rec.: isolated motifs recovered by
position or value. exact pos.: recovered at the identical start.
Datasets whose determinate motifs are all repeated (ECG200) contribute
no isolated patterns.}
\label{tab:rare-perds}
\begin{tabular*}{\columnwidth}{@{\extracolsep{\fill}}lrrrr@{}}
\toprule
Dataset & det.\ & isolated & rec.\ & exact pos. \\
\midrule
ItalyPowerDemand & 40 & 10 & 10 & 10 \\
ECG200 & 40 & 0 & 0 & 0 \\
TwoLeadECG & 40 & 17 & 17 & 17 \\
ECG5000 & 60 & 5 & 5 & 3 \\
FordA & 60 & 48 & 48 & 37 \\
ChlorineConc. & 40 & 22 & 22 & 22 \\
StarLightCurves & 40 & 21 & 21 & 21 \\
\midrule
\textbf{Total} & \textbf{320} & \textbf{123} & \textbf{123} & \textbf{110} \\
\bottomrule
\end{tabular*}

\vspace{12pt}
\caption{The twenty rarest patterns in the benchmark, ranked by the gap
$d_2/d_1$ between the motif distance $d_1$ and the distance to the
second-nearest window (larger is more isolated). match: \panache{}
recovers the pattern at the identical start position (pos) or as an
equal-distance motif (val). All twenty are recovered at the identical
position, across series from $1{,}600$ to $660{,}000$ samples.}
\label{tab:rare-top}
\begin{tabular*}{\columnwidth}{@{\extracolsep{\fill}}llrrrc@{}}
\toprule
Dataset & $n$ & $m$ & $d_1$ & $d_2/d_1$ & match \\
\midrule
StarLightCurves & 8k & 34 & 0.007 & 8.9 & pos \\
StarLightCurves & 8k & 34 & 0.007 & 8.7 & pos \\
ChlorineConc. & 8k & 30 & 0.136 & 5.0 & pos \\
FordA & 660k & 31 & 0.053 & 5.0 & pos \\
FordA & 660k & 30 & 0.051 & 4.9 & pos \\
FordA & 660k & 31 & 0.053 & 4.8 & pos \\
FordA & 660k & 30 & 0.051 & 4.8 & pos \\
ChlorineConc. & 8k & 30 & 0.136 & 4.8 & pos \\
FordA & 660k & 32 & 0.056 & 4.8 & pos \\
FordA & 660k & 33 & 0.059 & 4.7 & pos \\
FordA & 660k & 32 & 0.056 & 4.6 & pos \\
StarLightCurves & 64k & 30 & 0.002 & 4.6 & pos \\
StarLightCurves & 64k & 30 & 0.002 & 4.5 & pos \\
FordA & 660k & 33 & 0.059 & 4.5 & pos \\
FordA & 8k & 31 & 0.195 & 4.3 & pos \\
ChlorineConc. & 8k & 31 & 0.162 & 4.3 & pos \\
FordA & 8k & 32 & 0.201 & 4.3 & pos \\
FordA & 8k & 32 & 0.201 & 4.2 & pos \\
FordA & 8k & 31 & 0.195 & 4.2 & pos \\
StarLightCurves & 8k & 30 & 0.006 & 4.2 & pos \\
\bottomrule
\end{tabular*}

\end{table}

\end{document}